\pdfoutput=1
\documentclass[11pt,a4paper]{article}

\usepackage[utf8]{inputenc}
\usepackage[T1]{fontenc}
\usepackage{lmodern}
\usepackage[margin=1.05in]{geometry}
\usepackage{amsmath,amssymb,amsthm,mathtools}
\usepackage[numbers,sort&compress]{natbib}
\usepackage{graphicx}
\usepackage{booktabs,tabularx,array,longtable,multirow}
\usepackage{caption}
\usepackage{microtype}
\usepackage[hidelinks]{hyperref}
\hypersetup{pdftitle={The Axiomatic Trader: Latent Regularity, Information Budgets, and the Canonical Form of a Quantitative Investment System},
            pdfauthor={Jiayu Li}}

\allowdisplaybreaks

\theoremstyle{plain}
\newtheorem{theorem}{Theorem}[section]
\newtheorem{proposition}[theorem]{Proposition}
\newtheorem{lemma}[theorem]{Lemma}
\newtheorem{corollary}[theorem]{Corollary}
\theoremstyle{definition}
\newtheorem{definition}[theorem]{Definition}
\newtheorem{remark}[theorem]{Remark}
\newtheorem{example}[theorem]{Example}
\newtheorem{axiom}{Axiom}

\numberwithin{equation}{section}

\begin{document}

\begin{center}
{\fontsize{17pt}{21pt}\selectfont\bfseries The Axiomatic Trader: \\[2pt]
Latent Regularity, Information Budgets, and the \\[2pt]
Canonical Form of a Quantitative Investment System\par}

\vspace{1em}
Jiayu Li

\vspace{0.5em}
August 2026
\end{center}

\vspace{1.5em}
\begin{center}\textbf{Abstract}\end{center}
\begin{quotation}\noindent
Systematic trading rests on one article of faith: that regularities found in the past persist. This paper does three things. First, it states that faith as five axioms. Stripped of notation, each is a commonplace practitioners already accept: (A1) a decision may use only what was known when it was made; (A2) what looks like the market changing its rules is the market changing its state --- given the unobserved state it is in, the machinery is the same in every era; (A3) the future may replay stretches of the past, though not in the proportions history ran them; (A4) states persist for a while; and (A5) whatever predictability exists is slight, even for a rule that knows the state. What turns these commonplaces into axioms is quantification, and the quantities are declared rather than estimated: the invariance defect $\varepsilon_0$ bounds how far the machinery may drift within a state the modeller treats as one; the recurrence bound $\Lambda$ caps how much more often a $b$-period stretch of history may re-run than history ran it, the bound and its scale being one declaration in two parts; the coherence times $\ell_i$ say how long each coordinate of the state persists; the signal ceiling $\rho$ caps what any rule can forecast; and the invariance ratio $\kappa$ declares the fraction of that ceiling contingent on the state. These five declarations are the whole of the premises' empirical content. Second, it proves that the axioms force a five-stage canonical form for a quantitative investment system --- (S1) a declared representation, (S2) a predictor within a capacity ceiling, (S3) contiguous purged block evaluation aggregated by $\mathrm{CVaR}_{1/\Lambda}$, (S4) a budgeted, deflated search, (S5) robust sizing at a fraction of the estimate that the budget bounds --- each stage necessary: a procedure omitting it does strictly worse under a law the axioms admit. Third, it tests the axioms where they are falsifiable --- each can be tested only at its declared constants --- on real market series: no axiom is so far overturned.
\end{quotation}

\vspace{0.5em}

\clearpage
\tableofcontents
\clearpage

\section{Introduction}\label{introduction}
\subsection{The question}\label{the-question}
Strip a trading system of its engineering and what remains is:

\begin{quote}
\emph{Summarise a regularity from past data; apply it to new data; act on the result.}
\end{quote}

Every systematic strategy ever run is an instance of that sentence. The discipline around it is a body of accumulated practice --- purge the folds, weight by uniqueness, deflate the Sharpe ratio, size the bets, distrust the walk-forward --- each well argued, but arriving as a list justified by experience. A practitioner who has implemented them all is entitled to ask: \emph{is the list complete, and is it forced?}

This paper argues that the answer is largely yes. The claim is not that we can derive the alpha, but that the \emph{architecture} of a correct system --- what the feature map must do, how much capacity is admissible, what selection must optimise, how the backtest must be partitioned, how large a search is permissible, how much robustness is affordable, how big the bets may be --- is determined, up to constants, by a short list of numbers describing how badly behaved the unobserved world is allowed to be --- a recurrence bound $\Lambda$ and the scale $b$ it is declared at, an invariance defect $\varepsilon_0$ of the representation it is declared of, the coherence times of the state's coordinates, a signal ceiling $\rho$, and the fraction $\kappa$ of that ceiling which is contingent on the regime.

\subsection{The core faith and its precise form}\label{the-core-faith-and-its-precise-form}
That sentence presupposes the regularity persists, and the presupposition cannot be proved: it is Hume's problem, and \citet{russell1912problems} gave it the form every trader eventually meets in person. Russell's chicken is fed every morning for a hundred days and infers a law of nature; on the hundred and first, the farmer wrings its neck. \citet{taleb2007black} made the image famous in finance.

The usual reading --- induction is unreliable, so be humble --- is not actionable. Ours is different: \textbf{the chicken's inference was not wrong, its state space was}. Conditional on the latent state \emph{``ordinary day''}, feeding at dawn is a genuine and stable regularity, and the data supported it. What killed the chicken is that the state \emph{``day before Thanksgiving''} had positive probability in the future and probability \emph{zero} in its entire recorded past. The failure is not of induction but of \emph{coverage}.

That diagnosis is the seed of the paper. Write $Z_t$ for the latent state, $X_t$ for observable features, $Y_t$ for the forward return, and suppose the mechanism $Q(\mathrm{d}x,\mathrm{d}y \mid z)$ is time-invariant. Then a strategy fitted on the past fails in the future for exactly one reason: the \emph{mixture} over latent states moved. Bound that movement --- the future's latent mixture has density at most $\Lambda$ times the past's --- and the design of the system falls out of $\Lambda$ and four further declared constants.

\subsection{Five declared constants}\label{declared-constants}
\begin{itemize}
\item $\Lambda \ge 1$ at a declared scale $b$, the \textbf{recurrence bound}. How much more often can the deployment period re-run a $b$-period stretch of history than history ran it? The bound and the scale are one declaration in two parts: the same belief is a different number at a different $b$, and Proposition~\ref{prop:scale} is the exchange rate. $\Lambda = 1$ says the deployment window replays the sample's occupation exactly --- stronger than stationarity, which licenses no such thing over a finite window (Remark~\ref{rem:occupation}). $\Lambda = 4$ says ``a regime that occupied $10\%$ of my sample may occupy $40\%$ of the next decade''. $\Lambda = \infty$ is the chicken.

\item $\varepsilon_0 \ge 0$, the \textbf{invariance defect}. How far may the mechanism drift \emph{within} one of the states the representation declares? It is the other half of the same declaration: coarsening a representation lowers $\Lambda$ and raises $\varepsilon_0$, refining does the reverse, and the two together are what ``the same situation as before'' means. A defect costs $2 M \varepsilon_0$ on every bound (Proposition~\ref{prop:representation}), it is not estimable at the level of the state --- and its observable shadow is one a researcher's own sample can contradict outright (Remark~\ref{rem:lamlower}); run on real declared coordinates at the level the test is required to carry, it contradicts none (\S\ref{sec:empirical}), and it leaves a rule the axioms do not supply: a representation may not be chosen for the verdict it returns.

\item $\ell_i \ge 1$, the \textbf{coherence times}. How long does a regime persist? The slowest of them sets the block length on which regimes can be seen at all, and they are paid for by every statistic in proportion to how much of that statistic depends on the regime (Lemma~\ref{lem:effsample}); overlapping labels enter the same account at full weight. A coherence time is not a mixing coefficient, because a regime that is rare and long is invisible to a mixing coefficient and its duration is what a coherence time reads (Example~\ref{ex:ladder}); and it is the only persistence declared of the state --- the number of blocks a selection statistic is worth follows from it and from the same lemma, and is not a further constant (Corollary~\ref{cor:blocks} (a)).

\item $\rho \in (0,1)$, the \textbf{signal ceiling}. The largest per-period Sharpe ratio any rule can attain against these features; equivalently $\sqrt{R^2}$ of the best conditional-mean forecast. The range the worked examples declare is fixed in \S\ref{sec:axioms}, after Lemma~\ref{lem:sharpe}.

\item $\kappa > 0$, the \textbf{invariance ratio}. What fraction of that ceiling the researcher says is contingent on the regime, $\rho_{\mathrm{blk}} = \kappa \rho$. It is a declaration of stage S1 and not a clause of any axiom (Axiom~\ref{ax:snr}), it is refutable from block-score autocorrelation where $\Lambda$ is not, and at the declared value it decides whether the procedure can take a position at all (Corollary~\ref{cor:invbudget}). Proposition~\ref{prop:kapceil} caps what may be declared.
\end{itemize}

Everything else is a consequence. One further quantity appears in the axioms, the volatility ratio $\chi_\sigma = \bar{\sigma}/\underline{\sigma}$ across regimes (Axiom~\ref{ax:snr}); it is estimable and governs only the price of reading regimes through noisy blocks (\S\ref{sec:price}), so we do not count it among the constants a researcher must declare. The following table is the punchline; every entry is proved below.

\begin{table}[htbp]
\centering\small
\begin{tabularx}{\linewidth}{>{\raggedright\arraybackslash}p{0.26\linewidth}>{\raggedright\arraybackslash}X>{\raggedright\arraybackslash}p{0.17\linewidth}}
\toprule
\textbf{Design question} & \textbf{Answer} & \textbf{Source} \\
\midrule
What do I optimise in selection? & $\mathrm{CVaR}_{1/\Lambda}$ of block risks & Thm.~\ref{thm:cvar} \\
How do I split the backtest? & contiguous blocks at the declared scale $b \ge \ell^{*} \ge \ell$, purged and embargoed & Cor.~\ref{cor:blocks}, Prop.~\ref{prop:embargo} \\
How complex may my model be? & $d_{\mathrm{eff}} \lesssim \rho^2 n/H$, $H$ the label overlap & Thm.~\ref{thm:capacity} \\
How many strategies may I test? & $\log N \lesssim \frac{1}{2} \rho^2 n_{\mathrm{eff}}$, $n_{\mathrm{eff}} = n/(H + \rho_{\mathrm{blk}}^2 (2\ell - 1) + 2\pi_\mu)$, $\pi_\mu$ the loss's own predictability & Lem.~\ref{lem:effsample}, Thm.~\ref{thm:search} \\
How much invariance do I demand? & penalty $\lambda(\Lambda)$ on cross-regime dispersion & Thm.~\ref{thm:invariance} \\
How coarse may my representation be? & as coarse as its defect allows: $2M\varepsilon_0$ on every bound, against a lower $\Lambda$ & Prop.~\ref{prop:representation} \\
What does that robustness cost? & a factor $c(\Lambda)^2$ of the search budget & Prop.~\ref{prop:tailprice} \\
How big may my bets be? & at most the fraction $1/(1 + \varkappa)$ of Kelly, $\varkappa$ the budget spent & Prop.~\ref{prop:kelly} \\
Can I deliver the $\Lambda$ I declare? & only if $\Lambda_1 \le \lambda^{-1}(\delta/\kappa)$ at state scale: at $\kappa = 1$ no open position asserts more than $2.63$ & Cor.~\ref{cor:invbudget} \\
Can I estimate $\Lambda$? & No. It is a premise. & Cor.~\ref{cor:notestimable} \\
\bottomrule
\end{tabularx}
\end{table}

\subsection{Contributions}\label{contributions}
\begin{enumerate}
\item \textbf{An axiom system and a representation principle} (\S\ref{sec:setting}--\S\ref{sec:axioms}). Read as an existential, invariance is a normalisation --- $Z_t = t$ always satisfies it --- which is why the axiom is stated of a \emph{declared} representation and carries a declared defect $\varepsilon_0$: the modeller's task is to draw the boundary between mechanism and state, and the future-risk bound depends on that choice only through $(\Lambda,\varepsilon_0,\ell)$.

\item \textbf{An impossibility theorem} (Thm.~\ref{thm:impossible}). Without a recurrence bound, every estimator insensitive to a single replaced period can be driven, with probability at least $1/|\mathcal{H}| - O(1/n)$, to the worst rule in its class; with a bound $\Lambda$ the adversary must hand a visible $1/\Lambda$ of the past to the opposing mechanism, and only the estimators that notice it can keep a guarantee. $\Lambda$ itself is a functional of the future path, not identifiable from the past (Cor.~\ref{cor:notestimable}).

\item \textbf{A derivation of the selection objective} (Thm.~\ref{thm:cvar}). Worst-case future risk equals the CVaR at level $1/\Lambda$ of regime-conditional risk, exactly; ranking by average backtest performance is correct if and only if one believes $\Lambda = 1$, and the two rankings can differ in sign.

\item \textbf{An information budget} unifying capacity and search (Thms.~\ref{thm:capacity}--\ref{thm:pacbayes}): everything the researcher did must fit in $\rho^2 n_{\mathrm{eff}}$ nats, with the effective sample derived rather than asserted (Lemma~\ref{lem:effsample}) --- the state's persistence is charged only in proportion to the regime dispersion of what is averaged, the label overlap at full weight --- and the budget proved as a PAC-Bayes bound under the axioms, with that charge inside the exponential inequality (Thm.~\ref{thm:pacbayes}).

\item \textbf{The price of robustness} (\S\ref{sec:price}). One \emph{deployed-risk inequality} (Thm.~\ref{thm:deployed}) combines the CVaR identity and the search term, the latter computed on the effective sample of Lemma~\ref{lem:effsample} --- the records' dependence enters through the same $\upsilon$ as the mean's and nowhere else, neither as a mixing rate nor as a further declaration (Corollary~\ref{cor:blocks} (a), Remark~\ref{rem:blocklen}) --- so that four of the five stages are read off its two terms.

Its selection term carries a constant $c(\Lambda)$ whose asymptotics are the standard limit theory of the empirical expected shortfall \cite{chen2008nonparametric,zwingmann2016asymptotics}; ours is the reading of it as a divisor on the sample available to \emph{search} (Prop.~\ref{prop:tailprice}). Where $c$ sits inside $c \le \Lambda$ is a property of the block-risk law --- $\Lambda^{1/4}$ for smooth risks, $\Lambda$ when a rare regime dominates, and, on real block scores that carry the market, near the dear end (Experiments 1--2 on daily series, Experiment 6 on a long-only cross-section; \S\ref{sec:calibration}), the last because reading a regime through a noisy block is a convex-order penalty (Prop.~\ref{prop:eiv}).

A budget therefore caps the $\Lambda$ one can act on (Cor.~\ref{cor:lamcap}), and the cap is not the criterion's to dodge: a matching information lower bound puts the factor $\Lambda - 1$ beneath every selection rule, state-observing ones included (Prop.~\ref{prop:tailfloor}).

\item \textbf{A canonical form} (Thm.~\ref{thm:canonical}): five stages, each proved necessary on a law satisfying the axioms (Props.~\ref{prop:minimax}, \ref{prop:shuffle}, \ref{prop:fullkelly}). It fixes two rules practitioners set by taste --- the embargo (Prop.~\ref{prop:embargo}) and the invariance penalty (Thm.~\ref{thm:invariance}) --- bounds fractional Kelly by $f^\star \le 1/(1 + \varkappa)$ in the fraction $\varkappa$ of the budget spent, so that \emph{half} Kelly is a saturation ceiling rather than a constant (Cor.~\ref{cor:halfkelly}), and adds one rule the field's practice does not contain: deploy an average over training perturbations, not a single fit (Prop.~\ref{prop:ensemble}). That every good procedure decomposes into these five steps we do not claim (Rem.~\ref{rem:canonstatus}).
\end{enumerate}

This is not an empirical asset-pricing paper: we prove statements about procedures, not about markets, and not that the canonical form produces profit --- the axioms are premises about the world. Nor, \S\ref{sec:price} aside, is it a claim of originality for the tools: the CVaR duality is classical \cite{artzner1999coherent,follmer2016stochastic}, the density-ratio bound is Tan's marginal sensitivity model \cite{tan2006msm} moved from confounding to time, the mixing machinery is standard \cite{yu1994rates}. The contribution is the assembly: that these, applied to a latent-state model of markets, \emph{determine the architecture}.

\subsection{The shape of the paper}\label{sec:shape}
The paper answers three questions in order. \emph{What is assumed:} \S\ref{sec:setting} fixes the objects --- the observable pair, the declared latent state, the two occupation measures --- \S\ref{sec:axioms} states the five axioms with their declared constants and what is deliberately not an axiom, and \S\ref{why-the-axioms-are-necessary-an-impossibility-theorem} proves that nothing weaker will do: without a recurrence bound every estimator can be ambushed, and $\Lambda$ itself is a premise, not an estimate. \emph{What the axioms force:} \S\ref{sec:selection} derives the selection objective and its evaluation discipline (blocks, purge, embargo, the scale dictionary), \S\ref{sec:budget} the information budget on the derived effective sample, \S\ref{sec:price} the price of robustness --- one deployed-risk inequality, and the three costs of reading a tail through noisy blocks --- \S\ref{sec:sizing} the sizing rule, and \S\ref{sec:canonical} assembles the five necessary stages S1--S5 and reads the field's accumulated practice off them. \emph{What the data say:} \S\ref{sec:calibration} computes the budget on worked examples and reports where real block scores price the tail, \S\ref{sec:empirical} sets out what a sample can refute and the five predictions, \S\ref{sec:limits} the eight limitations with the reason each is not repaired here, and \S\ref{conclusion} closes with two conclusions. The proofs are in Appendix~\ref{app:proofs}, the experiments --- all on real market series --- in Appendix~\ref{app:experiments}, and the notation, with the symbols the paper uses twice, in Appendix~\ref{app:notation}.

\section{Related work}\label{related-work}
The work below is the work that tries to settle a question about systematic trading by proof; the practices the field arrived at by experience are read against the axioms in \S\ref{sec:canon} instead.

\textbf{Backtest overfitting.} \citet{bailey2014deflated}, \citet{bailey2017pbo} and \citet{harvey2016cross} quantify how easily backtests are overfitted --- the deflated Sharpe ratio, the probability of backtest overfitting, the $t > 3$ hurdle --- with \citet{white2000reality} and \citet{hansen2005spa} supplying the data-snooping machinery, and \citet{deprado2018advances} collects the corresponding practices. Theorem~\ref{thm:search} (i) recovers the scale of those corrections from the axioms and places search effort and model capacity on one budget line. That literature takes the set of trials as fixed; for an adaptive search, Theorem~\ref{thm:search} (ii) is the information-usage bound of \citet{russo2020information} --- selection bias at most the mutual information between the choice and the statistics, extended to generalisation error by \cite{xu2017information} --- with the paper's effective sample as its sub-Gaussian scale, so an adaptive search is charged the information it used in the same nats as the rest of the budget; that a held-out window consulted more than once is spent, the second half of Proposition~\ref{prop:audit}, is the reusable-holdout problem of \citet{dwork2015holdout} without its differential-privacy repair.

\textbf{Complexity in return prediction.} That machine learning helps in the cross-section only under heavy regularisation \citep{gu2020empirical,giglio2022factor}, the sceptical reading \S\ref{sec:budget} makes quantitative \citep{israel2020machines}, and the \emph{virtue of complexity} \citep{kelly2024virtue} do not conflict: the ceiling \eqref{eq:ceiling} binds on effective degrees of freedom, not on parameter count, and in a heavily shrunk ridge regression $d_{\mathrm{eff}}(\lambda)$ stays small as the feature count grows.

\textbf{Robustness and invariance.} Selecting a predictor for stability across environments rather than average fit runs through invariant causal prediction \cite{peters2016causal}, with distributionally robust optimisation supplying the machinery \cite{duchi2021dro}; Theorem~\ref{thm:cvar} says that in a latent-regime market this is not a preference for stability but the \emph{exact} worst-case future risk. The closest statement in that literature is \citet{duchi2023latent}: the worst-case risk over latent subpopulations of mass at least $\alpha$ is the $\mathrm{CVaR}_\alpha$ of the conditional risk, and \citet{namkoong2017variance} expand the same objective as a mean plus a multiple of a standard deviation, the form Theorem~\ref{thm:invariance} takes. What is different here is what plays the subpopulation: a state in time rather than a group in a population, so that its recurrence bound $\Lambda$ is a statement about a future window, unidentifiable from the past (Corollary~\ref{cor:notestimable}) and declared at a scale (Proposition~\ref{prop:scale}), and the tail is read through block scores that carry their own noise (\S\ref{sec:price}). The bound itself is the marginal sensitivity model of \citet{tan2006msm}, a descendant of Rosenbaum's $\Gamma$ \cite{rosenbaum2002observational}. The stance behind all five declared constants --- what cannot be estimated is declared, varied, and reported as a curve --- is the partial-identification programme of \citet{manski2003partial}, with $\Lambda$, $\varepsilon_0$ and $\kappa$ its sensitivity parameters and Remark~\ref{rem:lamlower}'s $\Lambda$-curve its report format. Choosing a density-ratio ball over a Wasserstein neighbourhood \cite{mohajerin2018wasserstein}, or over the relative-entropy ball of the robust-control tradition \citep{hansen2008robustness}, whose dual is an exponential tilt of one reference model rather than a tail mean of its regimes, is structural: under Axiom~\ref{ax:invariance} only the frequencies with which recognisable situations recur are unknown --- a re-weighting of an existing mixture, the one family whose radius is a statement about regimes; it is the $L^\infty$ member of $\varphi$-divergence DRO, the only one whose dual is $\mathrm{CVaR}$, which is why Theorem~\ref{thm:cvar} is an identity rather than a bound.

\textbf{Dependent data and the latent state.} \citet{yu1994rates} and \citet{mohri2009mixing} give generalisation bounds under mixing, and \citet{kuznetsov2015nonstationary} extend them via a discrepancy of which $\Lambda$ is a structured special case; adversarial online learning \cite{cesabianchi2006prediction} is the $\Lambda = \infty$ regime where only relative regret survives, as Theorem~\ref{thm:impossible} predicts. The concentration and PAC-Bayes literature under dependence \cite{paulin2015concentration,alquier2024userfriendly} is where Theorem~\ref{thm:pacbayes} lives, and \S\ref{sec:budget} says what it adds: the charge lands on the declared state's clock, not on the observable series' memory. On the economics, \citet{hamilton1989new} is the canonical regime-switching model, and the filtering tradition built on it estimates the state and then trusts the fitted transition law out of sample --- in the terms used here, a declaration of $\Lambda = 1$ of the fitted chain, which Corollary~\ref{cor:notestimable} says no fit can license. Proposition~\ref{prop:breadth} is the fundamental law of active management of \citet{grinold1989fundamental}, reached from an information constraint, Theorem~\ref{thm:cvar} identifies the prior set a trading researcher maximising the worst case \citep{gilboa1989maxmin} is implicitly using, and three empirical regularities enter as premises rather than as results: \citet{grossman1980impossibility} on why $\rho > 0$, \citet{mclean2016academic} on edge decay, the content of Definition~\ref{def:reflexivity}, and \citet{chen2022open} on the factor zoo, the visible consequence of a violated budget.

\section{The objects}\label{sec:setting}
This section fixes what the axioms of \S\ref{sec:axioms} are about, and asserts nothing about the world: the observable pair, the latent state a modeller declares, and the two occupation measures that carry the whole argument.

Time is discrete, $t \in \mathbb{Z}$. On $(\Omega,\mathcal{F},\mathbb{P})$ we are given an \textbf{observable feature} $X_t \in \mathcal{X}$, a measurable function of the information available strictly before the decision at $t$; a \textbf{label} $Y_t \in \mathcal{Y} \subseteq \mathbb{R}$, the realised forward return over the holding period beginning at $t$; and a \textbf{latent state} $Z_t \in \mathcal{Z}$, never observed, on a standard Borel space. Write $\mathcal{F}_t^{\mathrm{obs}} = \sigma(X_s,Y_{s - H} : s \le t)$ for the information a trader actually has at $t$, where $H$ is the label horizon.

\begin{definition}[Strategy]\label{def:strategy}
A \emph{strategy} is a measurable map $h : \mathcal{X} \to \mathcal{A} \subseteq \mathbb{R}$, where $h(x)$ is the signed position taken when the observed feature is $x$. A \emph{loss} is a measurable $L : \mathcal{A} \times \mathcal{Y} \to \mathbb{R}$. Two families matter: the predictive loss $L(a,y)=(y - a)^2$ and the economic loss $L(a,y)= - \log(1 + a y)$ or its mean--variance approximation $- a y + \frac{\gamma}{2} a^2 y^2$. The results below are stated for an $L$ bounded on the relevant range, and we write $M := \|L\|_\infty$; the constructions that exhibit a law (Theorem~\ref{thm:impossible}, Proposition~\ref{prop:tailfloor}) use the PnL loss with a Gaussian label, where the bound is on the regime risk $r_h$ and not on $L$, and the PAC-Bayes bound asks sub-Gaussian noise instead. Boundedness is a real restriction and we return to it in \S\ref{sec:limits}.
\end{definition}

The loss is the researcher's, and it is a convention of the setting, not an axiom: the axioms of \S\ref{sec:axioms} describe the world the history sampled, and say nothing about what the researcher wants from it. The paper asks two things of the objective and nothing more. A rule is judged by its expected loss under the occupation the deployment window will show --- the deployed risk $R^{+}$ of \eqref{eq:linear}, which Axiom~\ref{ax:recurrence} turns into a worst case over the futures it admits --- and the statements about that risk (Theorems~\ref{thm:cvar} and \ref{thm:deployed}) are made for an arbitrary bounded $L$; the capacity and search stages S2 and S4 of Theorem~\ref{thm:canonical} specialise it, to squared loss and to a Gaussian ranking statistic, where Remark~\ref{rem:canonstatus} says so. At stage S5 the objective must in addition be concave enough to carry a quadratic approximation, mean less $\frac{\gamma}{2}$ times second moment; the fraction Proposition~\ref{prop:kelly} bets is the Bayes shrinkage of the estimate, the same at every $\gamma$, and only the overall scale of the position is the objective's. Log growth, $\gamma = 1$, is the unit the paper writes that stage in, which is why it carries Kelly's name.

\subsection{Past and future}\label{past-and-future}
Fix an estimation window $\{1,\dots,n\}$ (``the backtest'') and a deployment window $\{n + 1,\dots,n + m\}$ (``live''), and define the \textbf{occupation measures} \begin{equation}\label{eq:occupation}\bar{\pi}_n := \frac{1}{n} \sum_{t \le n} \delta_{Z_t},\qquad \bar{\pi}^{+} := \frac{1}{m} \sum_{t > n} \delta_{Z_t}, \end{equation} the fraction of each window that the latent path actually spent in each state. They are random --- functions of the path --- and every statement below is made conditionally on the path.

\begin{remark}[Occupation, not marginal]\label{rem:occupation}
One could define $\bar{\pi}_n$ and $\bar{\pi}^{+}$ as window averages of the marginal laws $\mathcal{L}(Z_t)$. That object makes the recurrence axiom of \S\ref{sec:axioms} empty: for a stationary latent chain every marginal is the same, so $\bar{\pi}^{+} = \bar{\pi}_n$ and $\Lambda = 1$ identically, whatever the realised decade looks like. What a trader lives through --- and what Remark~\ref{rem:elicit} asks about --- is the \emph{realised} frequency of each state in one window, which for a chain with persistence $\ell$ fluctuates around its marginal with only $m/\ell$ effective visits. So $\Lambda$ is a statement about the path, not its law, and a stationary world does \emph{not} license $\Lambda = 1$ over any finite deployment (Proposition~\ref{prop:occfloor} computes how far a window can drift).
\end{remark}

\section{The axioms of latent regularity}\label{sec:axioms}

Five axioms. Each is stated, and then read in a few sentences --- what it asserts of the world, what violates it, and what it sets up; the account of which theorem needs which is \S\ref{sec:indep}. Four of the five carry a constant the modeller declares --- a defect, a bound at a scale, a coherence time per coordinate, a ceiling --- and those constants are the whole of the freedom the rest of the paper has.

\subsection{The five axioms}\label{the-five-axioms}

\begin{axiom}[Adaptedness / no anticipation]\label{ax:adapted}
$X_t$ is $\mathcal{F}_t^{\mathrm{obs}}$-measurable and $Y_t$ is realised after the decision at $t$. Any estimator is a measurable function of $\{(X_s,Y_s)\}_{s\le n}$ and of auxiliary randomness independent of them, and of nothing else.
\end{axiom}

What is forbidden is information from after $t$, not randomisation --- Definition~\ref{def:procedure} carries the same pair, and Proposition~\ref{prop:ensemble} and Theorem~\ref{thm:impossible} (iii) both use the second half of it. This is the axiom that survivorship bias, restated indices, point-in-time errors and leaky feature engineering violate; every theorem below is vacuous without it. The axiom is two-sided: it forbids $X_t$ from containing information from after $t$, and it does not require $X_t$ to contain the information from \emph{at} $t$ --- a feature whose value at $t$ is a carried-forward stand-in for an observation that has not arrived satisfies the axiom while misrepresenting the state, an invariance defect in the language below (Proposition~\ref{prop:representation}) that no coverage or missing-value check detects, because filling is precisely what hides it.

\begin{axiom}[Invariant mechanism, to a declared defect]\label{ax:invariance}
The modeller declares a latent state --- a representation $Z_t \in \mathcal{Z}$ of whatever drives the market --- and a \textbf{defect} $\varepsilon_0 \ge 0$. There is a Markov kernel $Q$ from $\mathcal{Z}$ to $\mathcal{X} \times \mathcal{Y}$, not depending on $t$, such that, conditionally on the whole declared path $Z := \{Z_s\}_{s\in\mathbb{Z}}$,
\[ \sup_t \, \|\mathcal{L}((X_t,Y_t) \mid Z) - Q(\cdot \mid Z_t)\|_{\mathrm{TV}} \le \varepsilon_0. \]
\end{axiom}

This is the formal version of \emph{``past regularities persist''}: the machinery of the market is fixed, and what moves is the state it is in. (A persistent feature is a coordinate of $Z_t$ that happens to be observed, and the axiom asks of it only that its one-period law given the state be free of $t$, \S\ref{sec:budget}.) The observable law at $t$ is then the mixture
\begin{equation}\label{eq:mixture}
\mathcal{L}(X_t,Y_t) = \int_{\mathcal{Z}} Q(\cdot \mid z)\, \pi_t(\mathrm{d}z), \qquad \pi_t := \mathcal{L}(Z_t),
\end{equation}
and the \emph{regime-conditional risk}
\begin{equation}\label{eq:regimerisk}
r_h(z) := \int L(h(x),y)\, Q(\mathrm{d}x,\mathrm{d}y \mid z)
\end{equation}
is a fixed function of $z$ that does not depend on $t$: all time variation in the performance of a fixed strategy is variation in $\pi_t$, the single structural fact the paper exploits. In particular the conditional expectation given $Z$ of the average loss over each window is linear in the corresponding occupation measure:
\begin{equation}\label{eq:linear}
\begin{aligned}
\bar R_n(h) &:= \mathbb{E}\big[\tfrac{1}{n} \textstyle\sum_{t\le n} L(h(X_t),Y_t) \, \mid \, Z \big] = \int r_h\, \mathrm{d}\bar\pi_n, \\
R^{+}(h) &:= \mathbb{E}\big[\tfrac{1}{m} \textstyle\sum_{t>n} L(h(X_t),Y_t) \, \mid \, Z \big] = \int r_h\, \mathrm{d}\bar\pi^{+}.
\end{aligned}
\end{equation}
The entire problem is now visible: we can noisily learn about the first integral and must control the second, and the only thing between them is how far $\bar\pi^{+}$ may sit from $\bar\pi_n$.

That is the whole axiom: the \emph{one-period} law of the observable given the state, to within $\varepsilon_0$. Dependence across periods is Axiom~\ref{ax:persistence}'s for the state and Lemma~\ref{lem:effsample}'s predictability term for the rest, and conditional independence given $Z$ is deliberately \emph{not} asserted --- a feature built from past prices is a function of earlier labels, so that clause would force $Y_{t-1}$ to be $Z$-measurable, the empty representation of Remark~\ref{rem:vacuous}, and it would make a return's autocorrelation the regime's, hence non-negative, where real daily series carry a \emph{negative} one (Experiment 3) --- under the axiom as stated an edge the ceiling of Axiom~\ref{ax:snr} budgets for, not a violation; what stopping at one period costs, the floor under a backtest mean's variance, is Remark~\ref{rem:upsbox}'s box (\S\ref{sec:limits}, Experiment 3). Every theorem below is stated at $\varepsilon_0 = 0$; Proposition~\ref{prop:representation} is the exchange rate, a positive defect adding $2M\varepsilon_0$ to every bound on deployed risk and nothing else, and stage S1 declares both constants of the trade.

\begin{remark}[Why the representation is declared, and not chosen afterwards]\label{rem:vacuous}
Read as an existential --- \emph{there is some} state on which the mechanism is invariant --- the axiom would have no content at all: take $Z_t := (t,\omega)$, the date together with the entire realisation, so that given $Z$ every pair is a point mass and $Q(\cdot \mid (t,\omega)) = \delta_{(X_t(\omega),Y_t(\omega))}$ is trivially $t$-free at $\varepsilon_0 = 0$, and \emph{no} theorem could follow. The state is instead \emph{declared}, at stage S1, before the data are read, and it is the same declared object Axiom~\ref{ax:recurrence} is about --- so the two constants bound the same representation from opposite sides, and each can fail while the other holds. $\varepsilon_0$ bounds how far the mechanism drifts \emph{within} a declared state; $\Lambda$ bounds how far the occupation moves \emph{across} them. The trivial representation is the extreme case of the trade: $\varepsilon_0 = 0$ and $\Lambda = \infty$, no state ever recurs, and a model with a free parameter for every date explains everything and predicts nothing.
\end{remark}

\begin{proposition}[Representation principle]\label{prop:representation}
Let $\psi : \mathcal{Z} \to \mathcal{Z}'$ be a measurable coarsening of the latent space, and define the \emph{invariance defect}
\[ \varepsilon(\psi) := \sup_t \big\|\mathcal{L}\big((X_t,Y_t) \mid \psi(Z)\big) - Q_\psi(\cdot \mid \psi(Z_t))\big\|_{\mathrm{TV}}, \qquad Q_\psi(\cdot \mid z') := \mathcal{L}\big((X,Y) \mid \psi(Z) = z'\big) \]
under the time-averaged law. Then for any $h$ with $M < \infty$ and any future window, the deployed risk conditional on the coarsened path, $R_\psi^{+}(h) := \mathbb{E}[R^{+}(h) \mid \psi(Z)]$ --- that of \eqref{eq:linear} averaged over what $\psi$ forgets, and the object a modeller who declared $\psi$ can condition on --- satisfies
\begin{multline*} R_\psi^{+}(h) \le \big[\text{the bound of Theorem }\ref{thm:cvar}\text{ computed in the coarsened space }\mathcal{Z}'\text{ with constant }\Lambda(\psi)\big]\\ + 2M\, \varepsilon(\psi), \end{multline*}
where $\Lambda(\psi) \le \Lambda$ is the constant of Axiom~\ref{ax:recurrence} for the coarsened path.
\end{proposition}

The defect is the same functional Axiom~\ref{ax:invariance} declares a level for, applied to a coarsening, and it is a worst case over $t$, not a probability: a representation exact except in one decade is charged for that decade at every date. Coarsening \emph{lowers} $\Lambda$ (states recur more often when you stop distinguishing them; on a continuous $\mathcal{Z}$ it is what makes $\Lambda$ finite at all) and \emph{raises} $\varepsilon$; refining does the reverse, with the date itself as the extreme. \emph{Deciding what counts as ``the same situation as before'' is the entire modelling act}, and it has a cost function; every dispute about whether ``this time is different'' is a dispute about $\psi$.

\begin{example}[The chicken, formally]\label{ex:chicken}
Take $\mathcal{Z} = \{\mathsf{ordinary},\mathsf{slaughter}\}$, $X_t \equiv 1$, and $Y_t \mid Z_t = \mathsf{ordinary}$ concentrated on $+1$, $Y_t \mid Z_t = \mathsf{slaughter}$ on a very negative number. The chicken's estimator is correct: $\hat h$ maximises expected reward on all 100 observations, and $r_{\hat h}(\mathsf{ordinary})$ is genuinely optimal. The mechanism was invariant throughout. What failed was $\bar\pi_n(\mathsf{slaughter}) = 0$ while $\bar\pi^{+}(\mathsf{slaughter}) > 0$, i.e.~$\Lambda = \infty$. The chicken did not have a statistics problem. It had a coverage problem.
\end{example}

\begin{axiom}[{$\Lambda$-recurrence at the evaluation scale}]\label{ax:recurrence}
Fix an \textbf{evaluation scale} $b \ge 1$ and partition the estimation window into $B = \lfloor n/b \rfloor$ contiguous blocks, writing $\bar\pi_{n,1},\dots,\bar\pi_{n,B}$ for their occupation measures (\eqref{eq:occupation} applied to each). The deployment window's occupation is a mixture of them in which no block counts for more than $\Lambda$ times its share:
\[ \bar\pi^{+} = \sum_{j\le B} w_j\, \bar\pi_{n,j}, \qquad w_j \ge 0, \quad \sum_j w_j = 1, \quad \max_j w_j \le \Lambda/B, \quad \Lambda < \infty. \]
\end{axiom}

In words: no $b$-period stretch of the past is re-run by the future more than $\Lambda$ times as often as it occurred --- at $b=1$, state by state: there the axiom is exactly $\bar\pi^{+} \ll \bar\pi_n$ with $\big\|\mathrm{d}\bar\pi^{+}/\mathrm{d}\bar\pi_n\big\|_\infty \le \Lambda$, and the mixtures $\pi$ satisfying that bound form the \emph{state-scale ball}, the set the mixture form of Axiom~\ref{ax:snr} and Proposition~\ref{prop:kapceil} quantify over; we write $\Lambda$ for $\Lambda(b)$ when the scale is fixed. $\Lambda = \infty$ is Example~\ref{ex:chicken}, and the whole is Tan's marginal sensitivity model \cite{tan2006msm} with ``treated versus control'' replaced by ``future versus past'': a hypothesis on the realised latent path --- every theorem below holds on the event that it does --- about frequencies of recognisable situations, which a domain expert can argue about.

Three consequences. \textbf{The scale is part of the representation}: declaring $b$ declares what counts as one repetition, and the assertion at a block \emph{implies} the assertion at the period, not conversely --- a longer block is a weaker promise and a stronger premise, priced by Proposition~\ref{prop:scale} and not sellable as caution. \textbf{The state must be atomic}: occupation measures of finite paths are atomic, so a finite $\Lambda$ forces the representation to coarsen a continuous state into finitely many recurring cells --- modelling is choosing a partition, and a scale to read it at. \textbf{And there must be more blocks than states}: at $b>1$ the mixture condition confines $\bar\pi^{+}$ to the convex hull of the $B$ block occupations, a set of affine dimension at most $B-1$ --- so for the declaration to leave every re-weighting of $K$ occupied cells admissible, as the state-scale ball at $b=1$ does, it needs $B \ge K$, that is $n \ge K\, b$, at every $\Lambda$.

\begin{remark}[Elicitation]\label{rem:elicit}
The scale makes the question concrete. Pick the worst $b$-period stretch your sample contains --- say the quarter of 2008Q4--2009Q1 --- and ask: over the next decade, what is the largest number of times I am willing to bet that a quarter resembling that one will recur, as a multiple of the rate at which it occurred? ``I would not be shocked by four times'' asserts $\Lambda(b) \ge 4$ at quarterly blocks --- and, asked stretch by stretch, it is a budget declared state by state, of which the axiom's single number is the uniform case. One warning, from measuring it: the regime a researcher can name is not always the one her strategy loses in (Remark~\ref{rem:lamlower}); naming the regime makes the question answerable, it does not make the answer aim.
\end{remark}

\begin{axiom}[{$\ell$-persistence}]\label{ax:persistence}
The declared latent state is $Z_t = (Z_t^1,\dots,Z_t^{K_z})$, with declared coherence times $1 \le \ell_1 \le \dots \le \ell_{K_z}$ and $\ell := \ell_{K_z}$, the coherence time of the slowest coordinate the modeller keeps: for each coordinate $i$ and every bounded $f$, the integrated autocorrelation time of $f(Z_t^i)$ is at most $2\ell_i - 1$. Nothing is asserted of the observables' own memory, and no rate is part of the axiom.
\end{axiom}

This is what the \emph{effective sample} charges (Lemma~\ref{lem:effsample}), each coordinate only in proportion to the regime dispersion the statistic owes to \emph{it} (Remark~\ref{rem:whichell}). What it asserts is a \emph{number} uniform over bounded $f$: the $\ell_i$ are the durations of the regimes a rule's risk depends on, not a rate at which the market mixes, and no mixing coefficient supplies them (Example~\ref{ex:ladder}) --- a coupling coefficient weighs a regime by its stationary mass and does not see a rare long one, while a coherence time reads its duration. Its scale is the \emph{contiguity scale} $\ell^* := \max(\ell,H,L_x)$ --- $H$ the label horizon, over which labels overlap and are dependent whatever the state does, $L_x$ the memory of the declared feature map, over which the feature carries earlier labels forward --- the shortest block that is a unit at all, and stage S3 takes $b \ge \ell^*$.

What the axiom does not declare is what the record carries beyond the declared state: the noise's memory across periods, and features that remember. That is the predictability term of Lemma~\ref{lem:effsample} --- what the observable past says about the loss that the state does not, which Axiom~\ref{ax:snr} bounds lag by lag and a declared coherence $\tau_\mu$ sums (Remark~\ref{rem:upsbox}) --- and it is the whole of the records' dependence the criterion ever pays: the count of blocks a selection statistic runs on is an identity of that lemma, computed from the declared $\ell$ and $\tau_\mu$, not a further declaration (Corollary~\ref{cor:blocks} (a)). A feature carrying a coordinate drawn once and never refreshed keeps every $\ell_i$ finite and $\tau_\mu$ infinite; it is a declaration a worksheet must make finite, and not an axiom, for the reason $\kappa$ is not one (\S\ref{sec:indep}).

\begin{axiom}[{$\rho$-weak signal}]\label{ax:snr}
For every latent state $z$ the history occupies, writing $\mu_z := \mathbb{E}[Y \mid \mathcal{F}^{\mathrm{obs}}, Z=z]$ for the conditional mean of the label given the whole observable past \emph{and} the state: $\mathrm{Var}(Y\mid Z=z) \ge \underline{\sigma}^2 > 0$ and $\mathbb{E}[\mu_z^2 \mid Z=z] \le \rho^2\, \mathrm{Var}(Y\mid Z=z)$ with $\rho \ll 1$; by Axiom~\ref{ax:invariance} neither side depends on $t$. Write $\sigma^2 := \mathrm{Var}_{\bar\pi_n}(Y)$ for the variance under the historical occupation, $\bar\sigma^2$ for the supremum of $\mathrm{Var}_\pi(Y)$ over the state-scale ball of Axiom~\ref{ax:recurrence} --- finite without a further assumption, since the ratio bound gives $\mathrm{Var}_\pi(Y) \le \Lambda\, \mathbb{E}_{\bar\pi_n}[Y^2]$ whatever is declared --- and $\chi_\sigma := \bar\sigma/\underline{\sigma} \ge 1$ for the \emph{volatility ratio}; $\chi_\sigma = 1$ is the homoskedastic case. Only the lower end is an assumption, and what it forbids is a regime whose conditional variance vanishes, without which $\chi_\sigma$ is infinite and Lemma~\ref{lem:sharpe} has no units to read $\rho$ in. We drop the subscript when the mixture is clear.
\end{axiom}

The quantifier runs over states and the conditioning is on the observable past \emph{and} the state: no executable rule has more edge than $\rho$ \emph{even if it knew the regime}. Three consequences are used below. \textbf{The mixture form:} for every mixture in the state-scale ball, whatever $\Lambda$, conditional Jensen and the law of total variance give $\mathbb{E}_\pi[\mu_\pi^2] \le \rho^2 \mathrm{Var}_\pi(Y)$ and $\mathrm{Var}_\pi(Y) \ge \underline{\sigma}^2$; the converse fails --- a bound over one ball says nothing about a state it weights lightly --- which is why Theorem~\ref{thm:impossible} and Lemma~\ref{lem:effsample}'s predictability bound need the axiom as stated. \textbf{The feature form,} $\mu_\pi(x) = \mathbb{E}_\pi[Y \mid X=x]$, follows by conditioning and is the one Theorem~\ref{thm:capacity} uses. \textbf{And the constant has two ends:} the historical occupation alone obeys the bound with some $\rho_1 \le \rho$ --- the end a track record measures, and the only one any budget computed on the history needs --- while the state-wise $\rho$ is the $\Lambda \to \infty$ end, the one declared, its extra strength spent in exactly two places: Lemma~\ref{lem:sharpe}'s mixture form and the bound on $\kappa$ below.

Two constants live around the axiom without being in it. The \emph{invariance ratio} $\kappa$: the effective sample and the dispersion penalty need a \emph{value} for $\rho_{\mathrm{blk}}$, which stage S1 declares and the data can afterwards refute (Appendix~\ref{app:experiments}); what a declared $\kappa$ buys is Corollary~\ref{cor:invbudget}, and what caps it is Proposition~\ref{prop:kapceil} at the end of this section. The \emph{volatility ratio} $\chi_\sigma$: estimable, from within-block standard deviations, hence not among the declared constants --- and requiring every regime's conditional variance to be the same would rule out the mechanism the experiments find to matter most on real data, block-score noise whose scale moves with the regime, which Proposition~\ref{prop:eiv} (iii) prices.

Axiom~\ref{ax:snr} is what makes finance different in kind, not merely in degree, from the domains where machine learning made its reputation, and it has a clean economic reading.

\begin{lemma}[Signal ceiling equals Sharpe ceiling]\label{lem:sharpe}
Let $h^\star(x) = \mu(x)/(\gamma\sigma^2)$ be the mean--variance-optimal timing rule, with $\mu(x) := \mathbb{E}[Y \mid X=x]$ and $\sigma^2 = \mathrm{Var}(Y)$, and suppose the conditional variance $\mathrm{Var}(Y \mid X=x)$ does not depend on $x$: the feature forecasts the mean, not the variance. Then, to first order in $\rho$, its per-period Sharpe ratio is $\mathrm{SR}(h^\star) = \sqrt{\mathbb{E}[\mu(X)^2]}/\sigma = \sqrt{R_{\max}^2} \le \rho$, independently of $\gamma$, where $R_{\max}^2 := \mathbb{E}[\mu(X)^2]/\mathrm{Var}(Y)$ is the population $R^2$ of the best conditional-mean forecast, with equality when the ceiling of Axiom~\ref{ax:snr} is attained; the remainder is at most $k_\mu \rho^3/(2(1-\rho^2)^{3/2})$, where $k_\mu := \mathbb{E}[\mu(X)^4]/(\mathbb{E}[\mu(X)^2])^2$ is the forecast's kurtosis ratio, $3$ for a centred Gaussian forecast. Under a mixture $\pi$ with $\chi_\sigma > 1$ the same holds with $\sigma_\pi$ in place of $\sigma$: the ceiling $\rho$ is on the Sharpe ratio measured in each regime's own volatility. We write $\rho$ for the attained value throughout; every budget below is therefore an upper bound.
\end{lemma}

So ``what out-of-sample $R^2$ should I expect?'' and ``what Sharpe is realistic?'' are the same question, related by a square root --- the conversion \citet{campbell2008predicting} use to argue that an out-of-sample $R^2$ too small to impress a statistician is economically large. A daily $R^2$ of $1\%$ is a daily Sharpe of $0.1$, annualised $0.1\sqrt{252} \approx 1.6$: a very good fund. A daily $R^2$ of $10\%$ would be an annualised Sharpe of $5$, which does not exist at capacity. This calibration, $\rho^2 \in [10^{-4},10^{-2}]$ per day, is used throughout.

\subsection{What is not an axiom}\label{sec:notaxiom}

\begin{definition}[Deployment, decay and cost]\label{def:reflexivity}
Deployment enlarges the state. Say that a market is \emph{$(\delta,c)$-reflexive} for $h$ if there is a kernel $Q^\dagger$ on $\mathcal{Z} \times [0,\infty)$ with $Q^\dagger(\cdot \mid z,0) = Q(\cdot \mid z)$ such that, when $C$ units of capital are allocated to $h$, the latent state during deployment is $(Z_t,C)$ and the conditional mean under it is $\mu^{+}(x) = \delta\, \mu(x) - c(C,h)$, with $\delta \in (0,1]$ a decay factor and $c \ge 0$ a convex cost (impact plus fees), $c(0,h) = 0$. Axiom~\ref{ax:invariance} is required to hold on the enlarged space: deployment changes the state, not the mechanism.
\end{definition}

This is \emph{not} an axiom, and the difference is the point: the axioms of this section describe the world the history sampled, and deployment is an intervention on it that the history did not sample. With $C_t$ the capital run in $h$ --- zero before discovery, $C > 0$ after --- the latent state during deployment is $(Z_t, C_t)$, so $\bar\pi^{+}$ is singular with respect to $\bar\pi_n$ on the enlarged space and Axiom~\ref{ax:recurrence} fails there for every finite $\Lambda$ at every scale; with nothing assumed about $Q^\dagger$ the construction of Theorem~\ref{thm:impossible} applies with no replacement of the past at all, and no estimator built from the history has any guarantee about the deployed rule. Definition~\ref{def:reflexivity} is the parametric form a researcher must \emph{assume} in order to say anything at all: under it the unvisited coordinate is a two-parameter shift of the visited one, the history's estimate of $r_h$ transfers as $\delta\hat\mu - c(C,h)$, and that shift is all that stage S5 of Theorem~\ref{thm:canonical} uses --- a researcher who deploys without declaring $(\delta, c)$ is entitled to nothing. It is the formal residue of \citet{grossman1980impossibility}: the edge exists because someone pays for it, and using it consumes it. \citet{mclean2016academic} measure $\delta$ directly, on published predictors out of sample and again after publication, and \citet{novymarx2016comparing} measure $c$. Both are large enough that a strategy passing every test in this paper can still be worthless, and because the definition is a convention and not an axiom, dropping it costs the paper nothing but that transfer: every other result below is a statement about the world the history sampled. The objective itself is the other convention of the same kind, and \S\ref{sec:setting} says what is asked of it: an arbitrary bounded loss where the deployed risk is bounded, a quadratic approximation at S5.

\begin{table}[htbp]
\centering\footnotesize
\begin{tabularx}{\linewidth}{>{\raggedright\arraybackslash}p{2.4cm}>{\raggedright\arraybackslash}X>{\raggedright\arraybackslash}X>{\raggedright\arraybackslash}X}
\toprule
\textbf{Constant} & \textbf{Meaning} & \textbf{Estimable?} & \textbf{Controls} \\
\midrule
$\Lambda(b)$ & recurrence bound at the declared scale & \textbf{No} (Cor.~\ref{cor:notestimable}) & selection objective, invariance penalty \\
$b$ & the evaluation scale the bound is declared at & Chosen, not estimated; $b \ge \ell^*$ & what counts as one repetition; converts by Prop.~\ref{prop:scale} \\
$\varepsilon_0$ & invariance defect of the declared representation (\ref{ax:invariance}) & \textbf{No} at the level of the state, as for $\Lambda$; its observable shadow is \emph{refutable} (Rem.~\ref{rem:lamlower}) & $2M\varepsilon_0$ on every deployed bound (Prop.~\ref{prop:representation}) \\
$\ell_i$, $\ell^*$ & coherence time per coordinate (A4); contiguity scale $\max(\ell,H,L_x)$, feature memory included & Per coordinate of a declared state, yes --- declared first and refuted afterwards (\S\ref{sec:empirical}); which coordinates, no & block length; the regime term of the effective sample (Lem.~\ref{lem:effsample}), and through it the block count of Thm.~\ref{thm:deployed} (Cor.~\ref{cor:blocks} (a)) \\
$\rho$ & signal ceiling, over every mixture \ref{ax:recurrence} admits ($\rho_1$ at the history alone) & Yes, but only from honest OOS, and only $\rho_1$ & capacity, search budget, bet size \\
$\kappa$ & invariance ratio $\rho_{\mathrm{blk}}/\rho$ --- declared, not an axiom & \textbf{Yes}, from block-score autocorrelation (Exp.~1) --- and it is \emph{refutable}, which $\Lambda$ is not & the regime term of Lem.~\ref{lem:effsample}, the dispersion penalty of Prop.~\ref{prop:kelly}, and, at the declared value, whether S5 can return a position (Cor.~\ref{cor:invbudget}) \\
$\chi_\sigma$ & volatility ratio across regimes & Yes, from within-block sd (Exp.~2) & the third cost of the tail (Prop.~\ref{prop:eiv} (iii)); the noise term of Lem.~\ref{lem:effsample} as an occupation average \\
$\tau_\mu$ & coherence, in periods, of the loss's own predictable component --- what the observable past says beyond the state (Lem.~\ref{lem:effsample}); declared, not an axiom & \textbf{Yes} (Exp.~3, panel A2; Exp.~4, S1): a measured $\upsilon$ below $H - 2\rho\tau_\mu$ contradicts it & both ends of the box of Rem.~\ref{rem:upsbox}, the lower one alone; every budget computed at $\pi_\mu = 0$ is overstated by what it measures \\
$\delta,c$ & decay and cost (Def.~\ref{def:reflexivity}, not an axiom) & Yes, ex post; not measured here & deployability \\
\bottomrule
\end{tabularx}
\caption{Where the axioms bind. The first three rows are one declaration, made of a representation: a bound on how far its occupation may move across states, the scale that bound is made at, and $\varepsilon_0$, the bound on how far the mechanism may drift within one.}
\end{table}

\subsection{Independence and minimality}\label{sec:indep}

\begin{remark}[Independence and minimality]\label{rem:independence}
A system calling itself axiomatic owes three things: which result each axiom carries, a model for each axiom that satisfies all the others and violates it, and an account of what was \emph{not} made an axiom. The first two are the table; the third is the paragraph after it.

{\footnotesize
\begin{tabularx}{\linewidth}{>{\raggedright\arraybackslash}p{3em}XXX}
\toprule
 & \textbf{carries} & \textbf{dropped, what fails} & \textbf{all others hold, this fails} \\
\midrule
\ref{ax:adapted} & every statement & in-sample risk can be made arbitrarily small with no implication for $R^{+}$; nothing survives & a feature that includes $Y_t$ \\
\ref{ax:invariance} & \eqref{eq:linear}, Thm.~\ref{thm:cvar}, Prop.~\ref{prop:eiv}, term (i) of Thm.~\ref{thm:deployed} & $r_h$ is undefined and no block is a sample of anything; and with no $\varepsilon_0$ there is nothing for Prop.~\ref{prop:representation} to charge & a declared $\psi$ whose cells merge regimes with different $Q$: $\Lambda(\psi)$ stays finite and \ref{ax:recurrence} holds, while $\varepsilon(\psi) > \varepsilon_0$ (Rem.~\ref{rem:vacuous}) \\
\ref{ax:recurrence} & Thm.~\ref{thm:cvar} and everything that prices a tail: Thm.~\ref{thm:invariance}, Prop.~\ref{prop:tailprice}, Cors.~\ref{cor:budgetrobust}, \ref{cor:lamcap}, the robust term of Prop.~\ref{prop:kelly} & Thm.~\ref{thm:impossible}: no guarantee at all; $\mathrm{CVaR}_{1/\Lambda} \to \operatorname{ess\,sup}$ & Example~\ref{ex:chicken} \\
\ref{ax:persistence} & the regime term of $n_{\mathrm{eff}}$ (Lem.~\ref{lem:effsample}), hence Thm.~\ref{thm:search}, the limit in Thm.~\ref{thm:invariance}, and the block count of Thm.~\ref{thm:deployed} (Cor.~\ref{cor:blocks} (a)) & $n_{\mathrm{eff}}$ is undefined, the budget is zero, and no block is worth anything: the fluctuation of Thm.~\ref{thm:deployed} does not shrink with $B$ & a coordinate whose coherence time exceeds the declared $2\ell_i - 1$: the effective sample is over-counted by $(2\ell-1)/(2\ell_i-1)$, and the coordinate whose rank outlasts the declared $2\ell-1$ in Exp.~1, panel J, is the refutation of \S\ref{sec:empirical} \\
\ref{ax:snr} & Lemma~\ref{lem:sharpe}, the right-hand side of every budget: Thms.~\ref{thm:capacity}, \ref{thm:search}, Prop.~\ref{prop:embargo}, Cors.~\ref{cor:lamcap}, \ref{cor:halfkelly} & no ceiling: capacity and search are unbounded and the sizing identity loses $\varkappa$ & $Y_t = X_t$, $\rho = 1$ \\
\bottomrule
\end{tabularx}
}

Two cells deserve comment. \textbf{Axiom~\ref{ax:invariance}:} read as an existential it can always be satisfied; what fills the cell is that the representation is \emph{declared} and the declaration carries a level --- $\varepsilon_0$ is to the mechanism what $\Lambda$ is to the occupation, priced by Proposition~\ref{prop:representation}. The axiom is not a normalisation; the existential form nobody uses is. \textbf{Axiom~\ref{ax:persistence}:} what the rare ladder of Example~\ref{ex:ladder} shows is why it is declared as a coherence time and not as a mixing coefficient: a regime entered at rate $e^{-T/\ell^*}/T$ and then lasting $T$ has a mixing coefficient free of $T$, while its indicator's coherence time is $T$ --- the duration, which is what the effective sample charges, and through it the block count. The system has five axioms, and this is the one that says how long the regimes a rule loses in last.

Two things the table does not establish: that Axiom~\ref{ax:invariance} is of minimal strength (\eqref{eq:linear} is all the theorems use of it), and that the independence models are exhaustive. What it does record is a test --- \emph{an axiom whose deletion breaks nothing is not an axiom} --- together with two corollaries, \emph{a strength no theorem uses is not one either}, and \emph{a strength no sample can read is not one a declaration should carry}. Three objects this paper declines to make axioms fail the first test: the bound on latent leverage, a constant computed inside Theorem~\ref{thm:impossible} ($\Delta = 2\rho\bar\sigma$, no more than the signal ceiling itself); the reflexivity of Definition~\ref{def:reflexivity}, an assumption about an intervention rather than a premise of the inference, whose removal costs nothing but the transfer stage S5 uses; and the invariance ratio $\kappa$, a stage-S1 declaration whose result holds for a pipeline priced at it rather than for every procedure the axioms admit (Corollary~\ref{cor:invbudget}). Three clauses the axioms could carry, and do not, fail the corollaries. Conditional independence of the observables across periods given the path: no theorem uses it, and it would make a return's autocorrelation the regime's (\S\ref{sec:axioms}, under Axiom~\ref{ax:invariance}). A mixing coefficient of the joint process: no sample reads it small --- any coarsening bounds it from below, and on real declared coordinates the bound is not small (Experiment 1, panel J) --- so any residual it entered would be of order $M$. And a \emph{record coherence} declared uniformly over the block averages of a block record: its binding member, the block averages of the state's own indicators at $(2\ell-1)/v(b,\ell)$ on the reference chain, is a statistic no criterion reads, and a declaration made to clear it would price Experiment 4's search at $1.8$ times what the criterion pays ($2.70$ against $1.52$ at $b = \ell = 60$, Corollary~\ref{cor:blocks} (a)). What the fluctuation of Theorem~\ref{thm:deployed} pays is the scores' own coherence across blocks, and that is an identity of Lemma~\ref{lem:effsample}, $\upsilon/\upsilon_b$, not a declaration (Corollary~\ref{cor:blocks} (a)); the sample reads it inside its null on the scores (panel B).
\end{remark}

\begin{proposition}[{The ceiling on a declared $\kappa$}]\label{prop:kapceil}
Under Axioms~\ref{ax:recurrence} and \ref{ax:snr}, for a rule whose regime risk $r_h(Z)$ has a Gaussian law under $\bar\pi_n$ --- the state-scale counterpart of Theorem~\ref{thm:invariance}'s block-scale range, assumed here of the law of a state and not obtained from a limit in $b$ --- every rule satisfies $\lambda(\Lambda)\, \rho_{\mathrm{blk}} \le (1 + \sqrt{\Lambda}\, \chi_\sigma)\, \rho$. Hence a declaration $\rho_{\mathrm{blk}} = \kappa\rho$ is consistent with the signal ceiling only for
\[ \kappa \le (1 + \sqrt{\Lambda}\, \chi_\sigma)/\lambda(\Lambda), \]
which is $2.36$ at $\Lambda=4$ with $\chi_\sigma=1$. For another law of $r_h(Z)$ the coefficient $\lambda(\Lambda)$ is replaced by that law's own tail coefficient, $(\mathrm{CVaR}_{1/\Lambda} - \text{mean})/\mathrm{sd}$, which a two-point law with mass $1/\Lambda$ on its bad atom puts at $\sqrt{\Lambda - 1}$ and a law whose dispersion sits in its good tail can make as small as it likes: the ceiling is law-dependent, and $2.36$ is the Gaussian case. Its content is only that $\kappa$ is not a free parameter; how far above a library's measured ratio it sits is Experiment 1's reading (Appendix~\ref{app:experiments}).
\end{proposition}

\begin{example}[The rare ladder: a mixing coefficient does not fix a coherence time]\label{ex:ladder}
Fix $\ell^* \ge 1$ and an integer $T \ge \ell^*+2$, and let the declared state be the chain on $T+1$ states that leaves a home state at rate $e^{-T/\ell^*}/T$ into a ladder walked deterministically for $T$ steps and back. It is exponentially $\beta$-mixing, $\beta(k) \le 2e^{-k/\ell^*}$ for every $k \ge 1$ --- a coupling coefficient free of $T$ at every lag --- while the indicator of the ladder has integrated autocorrelation time at least $T - 2(T+1)e^{-T/\ell^*}$, so the axiom fails at every declared $\ell_i < T/2$; and the block averages of that indicator carry the same time into the block scores, of order $T/b$ in blocks at any scale $b < T$. No mixing coefficient, at any scale or at every scale, fixes a coherence time --- which is why the axiom does not declare one. The proof is an elementary computation on this explicit chain, and is omitted.
\end{example}

The witness is worth reading as a market. At $\ell^*=60$ trading days and $T=240$ the rare state arrives about once in fifty years and stays for one --- a fair description of 2008. A mixing coefficient weighs it by its stationary mass and does not see it; the coherence time of the question \emph{am I in it} is its duration, which is what Lemma~\ref{lem:effsample} charges a rule whose regime risk lives there, and through it what the block count of Theorem~\ref{thm:deployed} charges its scores.

\section{Why the axioms are necessary: an impossibility theorem}\label{why-the-axioms-are-necessary-an-impossibility-theorem}

\begin{theorem}[No guarantee without recurrence, and which estimators recurrence protects]\label{thm:impossible}
Fix $n, m \ge 1$, a finite strategy class $\mathcal{H}$ with at least two strategies and a common position scale ($\mathbb{E}_{P^0}|h(X)|$ the same for every $h \in \mathcal{H}$; without it (iii) weakens to ``no better than the class supremum''), the PnL loss $L(a,y) = -ay$, constants $\ell \ge 1$ and $\rho > 0$, and \emph{any} estimator $\hat{h}_n$ mapping $\{(X_t,Y_t)\}_{t \le n}$ measurably into $\mathcal{H}$. Let $P^0$ be any law for the observable past that admits a latent representation satisfying Axioms \ref{ax:invariance}, \ref{ax:persistence} and \ref{ax:snr} with these constants, and let $p^\star := \max_{h \in \mathcal{H}} P^0(\hat{h}_n = h) \ge 1/|\mathcal{H}|$ be the probability of the estimator's modal output. For $k \in \{1, \dots, n\}$ call a law $P'$ for the observable past a \emph{$k$-replacement} of $P^0$ if it is obtained from $P^0$ by replacing, on $k$ contiguous periods, the conditional law of $(X_t, Y_t)$ given the latent path by one fixed kernel satisfying Axiom~\ref{ax:snr} with the stated constants, the other $n - k$ periods untouched, and write
\begin{equation}\label{eq:replacek}
\epsilon_n(k) \; := \; \sup_{P' \text{ a } k\text{-replacement of } P^0} \, d_{\mathrm{TV}} \big( \mathcal{L}_{P^0}(\hat{h}_n), \, \mathcal{L}_{P'}(\hat{h}_n) \big) \; \in [0, 1]
\end{equation}
for the estimator's \emph{replace-$k$ sensitivity}: how far the law of its output moves when a fraction $k/n$ of its sample is handed to another mechanism. Then for every $\Lambda \ge 1$, with $k := \lceil n/\Lambda \rceil$, there is a joint law $P_\Lambda$ for $\{(X_t,Y_t,Z_t)\}_{t \le n+m}$ such that

\begin{enumerate}
\item[(i)] the observable past under $P_\Lambda$ is a $k$-replacement of $P^0$: it coincides with $P^0$ on $n-k$ of the $n$ periods, and the law of $\hat{h}_n$ moves by at most $\epsilon_n(k)$;
\item[(ii)] $P_\Lambda$ satisfies Axioms \ref{ax:adapted}, \ref{ax:invariance}, \ref{ax:persistence} and \ref{ax:snr} with the stated constants, and Axiom~\ref{ax:recurrence} with constant $n/k \le \Lambda$ at the scale $b=1$, hence at every scale $b$ dividing $k$; the construction moves the conditional mean by at most $\Delta = 2\rho\bar{\sigma}$ in root mean square over the feature, no more latent leverage than Axiom~\ref{ax:snr}'s own ceiling;
\item[(iii)] on an event of $P_\Lambda$-probability at least $p^\star - \epsilon_n(k)$, $R^{+}(\hat{h}_n) = \sup_{h \in \mathcal{H}} R^{+}(h)$.
\end{enumerate}

That is: an estimator that does not react to the replacement of a $1/\Lambda$ fraction of its sample is driven, with probability bounded below by its own determinism, to the worst rule in its class. The constants are uniform in the estimator. Axiom~\ref{ax:recurrence} forces the adversarial state to have occupied a visible $\lceil n/\Lambda \rceil$ periods of the past --- the replacement cannot be made smaller than that, and at $\Lambda \ge n$ it is a single period --- so a guarantee at finite $\Lambda$ can exist only for the estimators whose $\epsilon_n(\lceil n/\Lambda \rceil)$ is large: those that notice one adversarial visit of that length.
\end{theorem}

The hypothesis on $P^0$ is lighter than it looks: any $\beta$-mixing law has a representation satisfying Axioms~\ref{ax:invariance} and \ref{ax:persistence} --- $Z_t := (X_t,Y_t)$, as in Remark~\ref{rem:vacuous} --- but that one fails Axiom~\ref{ax:snr}, a state that is an observation carrying no noise, and the construction needs $P^0$'s states to be weak one by one, which is what the axiom asserts.

The two selection statistics of \S\ref{sec:selection} sit on opposite sides of (iii), and this is the theorem's point. A ranking by the sample mean moves each candidate's statistic by at most $2M/\Lambda$ when $k$ periods are replaced --- its share of the visit; whenever the margin of its modal choice exceeds that share, $\epsilon_n(k)$ is small and the ambush succeeds. A ranking by the block $\mathrm{CVaR}_{1/\Lambda}$ with block length $b \le k$ is the opposite case: the adversarial visit covers all but at most one of the blocks that statistic averages, so it moves by the whole adversarial margin less $O(\Lambda/B)$ --- it is the estimator built to notice a $1/\Lambda$ fraction, for exactly the reason Theorem~\ref{thm:cvar} gives. We leave this as a reading of (iii) rather than a fourth item, because turning it into a bound needs a margin condition on $\mathcal{H}$ that this theorem deliberately does not assume.

Two features of the construction (Appendix~\ref{app:proofs}) are worth noting. The adversary may not look at the data: a latent path that jumped at $t = n + 1$ to oppose whatever the estimator \emph{actually} returned would violate Axioms~\ref{ax:persistence} and \ref{ax:invariance}, and the factor $p^\star$ in (iii) is the price of repairing both --- the right price, since an estimator that spreads its output over a large class is harder to ambush. And the visit cannot be hidden: $\bar{\pi}^{+} \ll \bar{\pi}_n$ requires the adversarial state to have been occupied at least once, so the construction bottoms out at $k = 1$, and what shrinks as $\Lambda$ grows is not the disturbance but the length of the visit an estimator must notice to be protected --- an estimator with replace-one sensitivity $O(1/n)$, which is every smooth functional of the sample, is ambushed with probability $p^\star - O(1/n)$. Read against robust statistics, (iii) is a breakdown argument run in reverse: $\epsilon_n(k)$ is a finite-sample sensitivity to contamination in the sense of \citet{huber1964robust}, evaluated at the fraction $1/\Lambda$ rather than infinitesimally, and the breakdown point of \citet{donoho1983breakdown} is the largest such fraction an estimator withstands. Robust estimation prizes the estimators that do not move under contamination; here those are exactly the unprotected ones, because under Axiom~\ref{ax:invariance} the replaced stretch is not corruption to be resisted but the visit of a state the future may re-run --- the one warning the estimator was going to get.

\begin{corollary}[{$\Lambda$ is not estimable}]\label{cor:notestimable}
Under Axioms \ref{ax:adapted}--\ref{ax:persistence}, $\Lambda_{n,m} := \|\mathrm{d}\bar{\pi}^{+}/\mathrm{d}\bar{\pi}_n\|_\infty$ is a functional of the future path, and the observable past does not determine it. For the redraw chain of Proposition~\ref{prop:occfloor} with a state of stationary mass $p \in (0, 1)$: (a) conditionally on the past, $\Lambda_{n,m}$ has a non-degenerate law for every $n$, with a limit as $n \to \infty$ at fixed $m$ that is also non-degenerate, so no past-measurable $\hat{\Lambda}_n$ converges to $\Lambda_{n,m}$ in probability; (b) at every finite $n$, $\Lambda_{n,m} = \infty$ with positive probability, on the event that the state is visited in the window and not in the history. The observable past therefore bounds $\Lambda$ from below (Proposition~\ref{prop:occfloor}) and never from above.
\end{corollary}

Corollary~\ref{cor:notestimable} is the honest statement of what a quantitative researcher does for a living. The statistics can be automated; the choice of $\Lambda$ cannot, because it is a claim about states of the world that have not been sampled (a finance-shaped cousin of \cite{wolpert1996lack}) --- that a currency peg can break, that a clearing house can change margin rules, that a central bank can buy corporate bonds. This is the only place domain judgement enters the system. A researcher who declines to choose $\Lambda$ has not avoided the choice; she has chosen $\Lambda = 1$.

\section{The selection objective is derived, not chosen}\label{sec:selection}
For a random variable $W$ on a probability space $(\Omega,P)$ and $\alpha \in (0,1]$, write $\mathrm{CVaR}_\alpha(W) := \alpha^{-1} \int_{1-\alpha}^1 F_W^{-1}(u) \, \mathrm{d}u$ for the mean of the worst $\alpha$-fraction of $W$ (upper tail, since $W$ will be a loss). The following is classical \cite{artzner1999coherent,rockafellar2000cvar,follmer2016stochastic}.

\begin{lemma}[Robust representation]\label{lem:robust}
For $\alpha \in (0,1]$, $\mathrm{CVaR}_\alpha(W) = \sup\{\mathbb{E}_Q[W] : Q \ll P, \, \mathrm{d}Q/\mathrm{d}P \le 1/\alpha\}$, and the supremum is attained. When $P$ is uniform on $B$ atoms the constraint set is $\{w \in \Delta_{B-1} : \max_j w_j \le 1/(\alpha B)\}$ and the supremum is the mean of the worst $\lceil \alpha B \rceil$ atoms, up to the fractional weight on the boundary atom.
\end{lemma}

\begin{theorem}[{$\Lambda$-robustness identity}]\label{thm:cvar}
Under Axioms \ref{ax:invariance} and \ref{ax:recurrence} at the declared scale $b$, conditionally on the latent path and for every $h$ with $r_h \in L^1(\bar{\pi}_n)$, write $R_j(h) := \int r_h \, \mathrm{d}\bar{\pi}_{n,j}$ for the regime risk averaged over the $j$-th historical block and let the $\mathrm{CVaR}$ below be computed under the empirical law of $R_1(h), \dots, R_B(h)$. Then
\begin{equation}\label{eq:cvar}
\boxed{R^{+}(h) \le \mathrm{CVaR}_{1/\Lambda}\big(R(h)\big) = \max_{w \, : \, \max_j w_j \le \Lambda/B} \sum_{j \le B} w_j \, R_j(h)}
\end{equation}
where $R^{+}$ is the conditional expected deployment loss of \eqref{eq:linear}. The bound is tight, attained by an admissible future occupation. Moreover $\Lambda \mapsto \mathrm{CVaR}_{1/\Lambda}$ is nondecreasing, equals $\bar{R}_n(h)$ at $\Lambda = 1$, and increases to $\max_j R_j(h)$ as $\Lambda \to \infty$. At $b = 1$ the blocks are the periods, the empirical law is that of $r_h(Z)$ under $\bar{\pi}_n$ and the last limit is $\operatorname{ess\,sup} \, r_h$: that is the form the identity had before the scale was made explicit, and Proposition~\ref{prop:scale} relates the two.
\end{theorem}

Theorem~\ref{thm:cvar} is the centre of the paper --- a one-line consequence of a classical duality, whose interest is that it closes a gap in practice: researchers choose a model-selection criterion by taste, and the theorem says there is nothing to choose --- given a belief about $\Lambda$, which one must hold by Corollary~\ref{cor:notestimable}, the criterion is determined.

\begin{corollary}[The average backtest score is a corner case]\label{cor:corner}
Ranking candidates by average backtest performance is the optimal ranking for every candidate class if and only if the researcher believes $\Lambda = 1$. (For a class whose block-risk laws differ only by location the two rankings agree at every $\Lambda$, Corollary~\ref{cor:locfam}; the ``only if'' is about classes in general.)
\end{corollary}
The two rankings do not merely differ in conservatism. They can differ in sign.

\begin{example}[Ranking reversal]\label{ex:reversal}
Two regimes, $\bar{\pi}_n = (0.9, \, 0.1)$, and rewards rather than losses for this example, so that the tail read is the lower one (Remark~\ref{rem:orientation}). Strategy $A$ earns $1$ in both. Strategy $B$ earns $1.5$ in the common regime and $-4$ in the rare one. Backtest means: $\bar{R}_n(A) = 1$, $\bar{R}_n(B) = 0.95$ --- close, and a little luck reverses it in $B$'s favour. At $\Lambda = 4$, however, $R^{+}(A) = 1$ while $R^{+}(B) = 0.6(1.5) + 0.4(-4) = -0.7$: the mean has selected a strategy that loses money, and $\mathrm{CVaR}_{1/4}$ selects $A$.
\end{example}

\subsection{From block risks to block scores}\label{from-latent-states-to-time-blocks}
Theorem~\ref{thm:cvar} is stated in terms of the block risks $R_j(h)$, which are the objects Axiom~\ref{ax:recurrence} is about but are still not observed: what a backtest reports is the block's realised average loss, the risk plus estimation noise. Axiom~\ref{ax:persistence} and Axiom~\ref{ax:invariance} are what bound the difference.

\begin{corollary}[The block score is the block risk plus noise]\label{cor:blocks}
Partition $\{1, \dots, n\}$ into $B = \lfloor n/b \rfloor$ contiguous blocks of length $b \ge \ell^*$, let $\bar{\pi}_{n,j}$ be the occupation measure of block $j$ and $\hat{r}_j(h)$ the average loss of $h$ on it. By Axiom~\ref{ax:invariance}, conditionally on the path,
\begin{equation}\label{eq:blockscore}
\hat{r}_j(h) = \int r_h \, \mathrm{d}\bar{\pi}_{n,j} + \xi_j(h),
\end{equation}
where, once the first $H$ periods of each block are purged (Proposition~\ref{prop:embargo}), the $\xi_j$ have conditional mean zero and conditional variance $H/b$ times the block average of $\sigma^2(Z_t)$, up to the predictability term of Lemma~\ref{lem:effsample} (applied at $n = b$), hence of order $b^{-1/2}$; and their dependence across blocks is priced by (a) below, which is Lemma~\ref{lem:effsample}'s and not Axiom~\ref{ax:invariance}'s. Two separate statements carry the scores to the block risks Theorem~\ref{thm:cvar} bounds. (a) \emph{Block count.} The block records are not independent, and Lemma~\ref{lem:effsample} says what they are worth. The mean of the $B$ block scores is the backtest mean at $n = Bb$, of variance $\sigma^2\upsilon/(Bb)$, and one block score has the variance $\sigma^2\upsilon_b/b$ of the same lemma at $n = b$, $\upsilon_b := H + \rho_{\mathrm{blk}}^2 v(b,\ell) + 2\pi_\mu$; so the scores' integrated autocorrelation time across blocks --- their \emph{record coherence} --- is $\upsilon/\upsilon_b$, and $B$ blocks carry the variance of $B\upsilon_b/\upsilon$ independent ones. It is an identity of the lemma, carried by the declared $\ell$ and $\tau_\mu$ and by nothing further: exact on the reference chain, where the block average of a state function alone has $(2\ell-1)/v(b,\ell)$ --- $2.70$ at $b = \ell = 60$, $1.75$ at $b = 2\ell$, $1 + O(\ell/b)$ beyond --- and a block score, whose noise is uncorrelated across blocks beyond the predictability term, has the smaller $\upsilon/\upsilon_b$, $1.52$ at $b = \ell = 60$ and the declaration of \S\ref{sec:calibration}, near one wherever $\rho_{\mathrm{blk}}^2(2\ell-1)$ is small against $H$; and, for a general chain, the same ratio with the lemma's declared coherence times in place of the chain's. In the Gaussian range of Proposition~\ref{prop:tailprice}, the scores jointly Gaussian across blocks with non-negative correlations, the tail truncation $\max(\hat{r}_j(h), u)$ the criterion reads inherits the count, a monotone function of a Gaussian pair being no more correlated than the pair; outside that range it is read on the sample (Experiment 1, panel B). (b) \emph{Noise.} The $1/\Lambda$-tail mean of the scores is not that of the block risks: by Proposition~\ref{prop:eiv} it exceeds it, by $\lambda(\Lambda)(\sqrt{s^2 + \sigma_\xi^2} - s)$ when both are Gaussian, and the excess does not vanish as $B$ grows --- \S\ref{sec:price} shows it is not negligible at realistic $b$. Hence
\[ \hat{\mathrm{CVaR}}_{1/\Lambda}(h) := \frac{1}{\lceil B/\Lambda \rceil} \sum_{j \in \text{worst } \lceil B/\Lambda \rceil} \hat{r}_j(h) \]
is the natural plug-in for the right-hand side of \eqref{eq:cvar} (the exact supremum of Lemma~\ref{lem:robust} puts the fractional weight on the $\lceil B/\Lambda \rceil$-th block; the committed scripts average the worst $\lfloor B/\Lambda \rfloor$, one block fewer wherever $\Lambda$ does not divide $B$). Random $k$-fold cross-validation, which shuffles observations across folds, estimates instead the $\mathrm{CVaR}$ of a distribution whose dispersion has been destroyed --- by the factor $v(b, \ell)$ of \eqref{eq:vcontig}, which is $2\ell - 1$ for long blocks (Proposition~\ref{prop:shuffle}) --- and is consistent for \eqref{eq:cvar} only when $\Lambda = 1$.
\end{corollary}
This derives, rather than recommends, block-structured evaluation: purged $k$-fold, combinatorial purged cross-validation and walk-forward are all attempts to construct blocks whose scores stand in for the $R_j(h)$, and Axiom~\ref{ax:recurrence} is what gives the resulting tail mean a deployment interpretation --- never stronger than the scale it was declared at (Proposition~\ref{prop:scale}).

\begin{proposition}[Embargo length]\label{prop:embargo}
(i) The noise. Purge the first $H$ periods of the test block: within them a label straddles the boundary and shares the training era's mechanism noise. Beyond the purge the test block's noise is tied to the training data only through the predictable component of Axiom~\ref{ax:snr} --- what the observable past says about the loss that the state does not, $|\mathrm{Cov}(\xi_s,\xi_t)| \le \rho\sigma^2$ at every lag $\ge H$ and coherent over the $\tau_\mu$ periods of Lemma~\ref{lem:effsample} --- a term the effective sample already carries at $2\pi_\mu$ and no gap removes. A feature with memory $L_x$ reads the training era for $L_x$ periods more, but what it carries is a function of the observable past, and its covariance with the test noise is that component, priced and not purged. The purge is the noise's whole embargo; the gap is set by part (ii), which charges what the training data know about the regime risk.

(ii) The dependence of the loss alone. Let the test block have length $b$, begin $g \ge H$ periods after the training era ends at time $n$, and have its first $H$ periods purged; write $\mathcal{F}_n$ for the training $\sigma$-algebra, $\bar{r}_h$ for the stationary mean of the regime risk, and, for $k \ge 1$,
\[ \mathrm{pred}_h(k) := \|\mathbb{E}[r_h(Z_k) \mid \mathcal{F}_0] - \bar{r}_h\|_2 \, / \, \mathrm{sd}_\nu(r_h) \]
for the \textbf{predictability} of the regime risk at lag $k$: the largest correlation any function of the observable history to time $0$ has with $r_h(Z_k)$, so that $\rho_{\mathrm{blk}}\sigma \, \mathrm{pred}_h(k)$ is the most of the regime risk $k$ periods on that the history still sees. Then, for any stationary state, the conditional bias of the test-block score is $b^{-1}\sum_{i=1}^b (\mathbb{E}[r_h(Z_{n+g+i}) \mid \mathcal{F}_n] - \bar{r}_h)$, whose root mean square over training histories is at most
\[ \rho_{\mathrm{blk}}\sigma \; \cdot \; b^{-1}\sum_{i=1}^b \mathrm{pred}_h(g+i) \]
in the units of Lemma~\ref{lem:effsample}, and requiring it to be a fraction $\kappa_{\mathrm{emb}}$ of the attainable edge $\rho\sigma$ of Lemma~\ref{lem:sharpe} asks of the declared predictability that its mean over the block's lags be at most $\kappa_{\mathrm{emb}}\rho / \rho_{\mathrm{blk}}$. The coefficient is bounded from below by the correlation of $r_h(Z_k)$ with any predictor built from the history --- for a rule whose regime risk is a function $f$ of an observed coordinate, the lag-$k$ autocorrelation of $f$ is such a bound, and a number a sample reads. On any stationary reversible chain with relaxation time $\ell$ --- the redraw chain of Proposition~\ref{prop:occfloor} the extremal case, on which the proof computes the bias exactly --- $\mathrm{pred}_h(k) \le \rho_\ell^k$ with $\rho_\ell := 1 - 1/\ell$, the block mean of $\rho_\ell^{g+i}$ is $\rho_\ell^{g+1}\, \ell(1-\rho_\ell^b)/b$, and the requirement reads
\[ g \; \ge \; H + \ell \, \log\!\left(\frac{\rho_{\mathrm{blk}} \, \ell (1-\rho_\ell^b)}{\kappa_{\mathrm{emb}} \, \rho \, b}\right), \]
which at $\rho_{\mathrm{blk}} = \rho$ and $b = \ell$ is $H + \ell\log(0.63/\kappa_{\mathrm{emb}})$. The mechanism noise contributes only what part (i) says, and the purge is $H$ and not $\max(H, L_x)$: a feature memory of $L_x \ge b$ --- a year's lookback against a sixty-day block --- would leave nothing to score, and what the feature carries forward is priced at $2\pi_\mu$ already.
\end{proposition}
What the proposition establishes against the $1\%$-of-sample convention is the scaling: the embargo follows $\ell$ and $\rho_{\mathrm{blk}}/\rho$, not $n$. Part (ii), which charges only what the training data know about the regime risk --- of size $\rho_{\mathrm{blk}}\sigma$, not $M$ --- gives $\ell\log(0.63/\kappa_{\mathrm{emb}}) \approx 111$ at $b = \ell$ and at most $\ell\log(1/\kappa_{\mathrm{emb}}) \approx 138$ at any shorter block on the reference chain, plus the purge: two to three times the convention, not ten, at $\ell = 60$, $\kappa_{\mathrm{emb}} = 0.1$. Off the chain the logarithm goes and the object stays, a declaration the sample bounds from below, which the experiments read on real declared coordinates (Experiment 1, panel J; Appendix~\ref{app:experiments}).

\subsection{Invariance is priced, not preferred}\label{sec:invariance}
Practitioners prefer a strategy that works ``consistently across regimes'' to one with the same average and higher dispersion, and call it prudence. Theorem~\ref{thm:cvar} shows it is arithmetic, with an exchange rate.

\begin{theorem}[The invariance penalty]\label{thm:invariance}
Let the block-risk distribution of $h$ under $\bar{\pi}_n$ have mean $\bar{r}_h$ and standard deviation $\mathrm{sd}_{\mathrm{blk}}(r_h)$. Then
\[ \mathrm{CVaR}_{1/\Lambda}(r_h) = \bar{r}_h + \lambda(\Lambda) \, \mathrm{sd}_{\mathrm{blk}}(r_h) + o\big(\mathrm{sd}_{\mathrm{blk}}(r_h)\big), \qquad \lambda(\Lambda) = \Lambda\varphi\!\big(\Phi^{-1}(1-1/\Lambda)\big), \]
exactly when the block-risk law is Gaussian, and on the redraw chain of Lemma~\ref{lem:effsample} in the limit $b/\ell \to \infty$, in which the block-averaged risk of \eqref{eq:blockscore} is a mean of $b/\ell$ effectively independent terms and the Berry--Esseen error is $O((\ell/b)^{1/2})$ relative to $\mathrm{sd}_{\mathrm{blk}}$. Away from that limit the coefficient depends on the shape of the block-risk law, not only on $\Lambda$; whether a given series' block scores are in the Gaussian range at an admissible $b$ is an empirical question, left to the experiments. Hence the model-selection score is
\begin{equation}\label{eq:penalty}
\mathcal{S}_\Lambda(h) = \bar{r}_h + \lambda(\Lambda) \, \mathrm{sd}_{\mathrm{blk}}(r_h),
\end{equation}
a mean--dispersion criterion whose trade-off coefficient is a function of the recurrence bound alone.
\end{theorem}

\begin{table}[htbp]
\centering\small
\begin{tabular}{lcccccccc}
\toprule
$\Lambda$ & 1 & 1.5 & 2 & 3 & 4 & 5 & 10 & 20 \\
\midrule
tail level $1/\Lambda$ & 1.000 & 0.667 & 0.500 & 0.333 & 0.250 & 0.200 & 0.100 & 0.050 \\
penalty $\lambda(\Lambda)$ & 0.000 & 0.545 & 0.798 & 1.091 & 1.271 & 1.400 & 1.755 & 2.063 \\
\bottomrule
\end{tabular}
\caption{The exchange rate between average backtest performance and cross-regime stability, as a function of the asserted recurrence bound.}
\end{table}

The folk heuristic is thereby calibrated --- ranking by ``mean fold score minus one standard deviation'' is correct for $\Lambda \approx 3$, one and a half asserts $\Lambda \approx 6$, a statement about the world rather than about temperament --- and the invariance literature is recovered: penalising cross-environment dispersion is risk extrapolation \cite{krueger2021rex}, the maximum over environments is group DRO \cite{sagawa2020groupdro} ($\Lambda = B$), and anchor regression \cite{rothenhausler2021anchor} interpolates with a coefficient playing the role of $\lambda(\Lambda)$.

\subsection{The same belief at another scale}\label{sec:scale}
Axiom~\ref{ax:recurrence} is declared at a scale, so two researchers who both say ``$\Lambda = 4$'' have made the same assertion only if they read it at the same $b$. The dictionary is Theorem~\ref{thm:invariance} applied twice.

\begin{proposition}[The scale dictionary]\label{prop:scale}
Let the block-averaged regime risk at scale $b$ have standard deviation $\mathrm{sd}_{\mathrm{blk}}(b) = \rho_{\mathrm{blk}}\sigma\sqrt{v(b,\ell)/b}$, with $v$ the kernel of \eqref{eq:vcontig} (Proposition~\ref{prop:shuffle} (i)). In the Gaussian range of Theorem~\ref{thm:invariance}, asserting $\Lambda$ at scale $b$ imposes the penalty $\lambda(\Lambda)\,\mathrm{sd}_{\mathrm{blk}}(b)$ on the selection score, so the assertions $(\Lambda, b)$ and $(\Lambda', b')$ impose the same penalty --- are the same belief --- exactly when
\[ \lambda(\Lambda)\sqrt{v(b,\ell)/b} = \lambda(\Lambda')\sqrt{v(b',\ell)/b'}. \]
Write $\vartheta := \lambda(\Lambda)\sqrt{v(b,\ell)/b}$ for the \emph{scale-free protection} of the assertion, in units of $\rho_{\mathrm{blk}}\sigma$, and $\Lambda_1 := \lambda^{-1}(\vartheta)$ for its \emph{state-scale reading}, the $b' = 1$ member, at which $v(1,\ell) = 1$. Then $v(b,\ell) \le b$ for every $b$ and every $\ell$, so $\Lambda_1 \le \Lambda$ with equality only at $b = 1$, and $\Lambda_1$ is decreasing in $b$: at a fixed nominal $\Lambda$, a longer block delivers less protection --- while, by \S\ref{sec:axioms}, admitting fewer futures. It is the weaker promise and the stronger premise. Two consequences. \emph{(i)} Stage S3 requires $b \ge \ell^*$, and $\sqrt{v(\ell,\ell)/\ell} \le \sqrt{2/e} = 0.858$ for every $\ell \ge 3$ --- the value is $0.851$ at $\ell = 5$ and rises to the limit $\sqrt{2/e}$, and only the two degenerate cases $\ell = 1$, $2$ sit above it, at $1$ and $0.866$ --- so no assertion made at an admissible scale, on a state whose coherence time is three periods or more, reads at more than $\lambda^{-1}(0.858 \, \lambda(\Lambda))$ at state scale. \emph{(ii)} The block length is therefore not a free implementation choice. Lengthening it to make the criterion affordable --- which is what Corollary~\ref{cor:nonempty}'s floor asks for --- weakens the assertion the criterion implements, by exactly this factor.
\end{proposition}
At $\ell = 60$ a nominal $\Lambda = 4$ reads, at state scale, as $2.99$ at the shortest admissible block $b = \ell$, and as $1.92$ at $b = 252$: the shortest admissible block already costs a quarter of the assertion, a calendar year costs half, and the still longer blocks Corollary~\ref{cor:nonempty}'s floor asks for on real series cost more again (\S\ref{sec:calibration}).

The dictionary runs both ways --- to deliver $\lambda(\Lambda)$ at state scale from a criterion computed at $b$, run it at the level $1/\Lambda'$ with $\lambda(\Lambda') = \lambda(\Lambda)/\sqrt{v(b,\ell)/b}$, a buy-back that stops existing past a few multiples of $\ell$ --- and the two levers pull against each other: Corollary~\ref{cor:nonempty} wants a long block so that the dispersion penalty is affordable, this proposition says a long block is a weaker promise, and \S\ref{sec:nonempty} is where they meet.

\section{The information budget}\label{sec:budget}
Axioms \ref{ax:persistence} and \ref{ax:snr} interact to produce a hard ceiling on how much the researcher is allowed to \emph{do}. Before stating it we have to say what sample the ceiling is measured against.

\subsection{The effective sample}\label{sec:effsample}
Every statistic a backtest produces is an average over $n$ periods, and the question is how many of them count. Two things make consecutive periods redundant: a label of horizon $H$ shares its mechanism noise with the $H - 1$ labels on either side, at the scale of $\sigma$; and the latent state persists, so the \emph{regime-dependent} part of what is averaged is sampled once per sojourn, at the scale of $\rho\sigma$, because Axiom~\ref{ax:snr} caps what the state can contribute to a conditional mean.

\begin{lemma}[The effective sample of a backtest mean]\label{lem:effsample}
Let the latent chain be the redraw chain of Proposition~\ref{prop:occfloor} with coherence time $\ell$ and stationary law $\nu$, let $h$ have regime risk $r_h$ with dispersion $\rho_{\mathrm{blk}}^2 := \mathrm{Var}_\nu(r_h) / \sigma^2$, and let the label's mechanism noise have the equal-weight overlap of an $H$-period return. Write $\xi_t := L(h(X_t), Y_t) - r_h(Z_t)$ for that noise and
\[ \pi_\mu \;:=\; \sum_{k \ge H} \mathrm{Cov}(\xi_0, \xi_k) / \sigma^2 \]
for the \emph{predictability term}: what the observable past says about the loss that the state does not. Axiom~\ref{ax:invariance} does not set it to zero (\S\ref{sec:axioms}), and Axiom~\ref{ax:snr} bounds it, $|\pi_\mu| \le \rho\,\tau_\mu$ with $\tau_\mu$ the coherence time of that predictable component, since no conditional mean given the observable past and the state exceeds $\rho\sigma$ in mean square. Then the backtest mean $\hat R_n(h) = n^{-1}\sum_{t\le n} L(h(X_t), Y_t)$ has
\begin{equation}\label{eq:vcontig}
\mathrm{Var}\big(\hat R_n(h)\big) \;=\; \frac{\sigma^2}{n}\big(H + \rho_{\mathrm{blk}}^2\,v(n,\ell) + 2\pi_\mu\big), \qquad v(b,\ell) := \sum_{|j|<b}\left(1-\frac{|j|}{b}\right)\rho^{|j|} = (2\ell-1) - \frac{2\rho(1-\rho^b)\ell^2}{b},
\end{equation}
$\rho = 1 - 1/\ell$, exactly in the regime term and up to edge terms of relative order $H/n$ and $\tau_\mu/n$ in the other two, with $v(n,\ell) \to 2\ell-1$ as $n/\ell \to \infty$. Three extensions:
\begin{enumerate}
\renewcommand{\labelenumi}{(\roman{enumi})}
\item For any stationary reversible chain whose transition operator has absolute spectral gap $1/\ell$ --- $\ell$ its relaxation time, which for the redraw chain is the $\ell$ above --- the equality becomes the inequality $\le$, for every $r_h$ and with $v(n,\ell)$ unchanged; the redraw chain is the extremal case.
\item For a general chain under Axiom~\ref{ax:persistence}, $2\ell-1$ is replaced by the integrated autocorrelation time of $r_h(Z)$ --- at most $2\ell-1$ by that axiom when the state has one coordinate, $r_h$ being a bounded function of it, and charged coordinate by coordinate through (iii) when it has several --- which no mixing coefficient bounds (Example~\ref{ex:ladder}).
\item When the declared state has several coordinates the regime term is a \emph{sum} over them, $\sum_i \rho_{\mathrm{blk},i}^2(2\ell_i-1)$ with $\rho_{\mathrm{blk},i}^2 := \mathrm{Var}_\nu(\mathbb{E}[r_h\mid Z^i])/\sigma^2$ the dispersion the risk owes to coordinate $i$: exactly, for a risk additive across independent coordinates; in general the Hoeffding decomposition of $r_h$ contributes one term per subset $S$ of coordinates, and on independent redraw coordinates the term for $S$ decays at the \emph{product} of its members' autocorrelations, $1-1/\ell_S = \prod_{i\in S}(1-1/\ell_i)$ --- faster than any member, since the coincidence of two regimes ends when either does --- so the single-$\ell$ form with $\ell=\ell_{K_z}$ is an upper bound and the sum over subsets, each at its own $\ell_S$, is the honest charge (Axiom~\ref{ax:persistence}, Remark~\ref{rem:whichell}).
\end{enumerate}
We write
\begin{equation}\label{eq:neff}
\upsilon \;:=\; H + \rho_{\mathrm{blk}}^2\,(2\ell-1) + 2\pi_\mu, \qquad n_{\mathrm{eff}} \;:=\; \frac{n}{\upsilon},
\end{equation}
and call $n_{\mathrm{eff}}$ the effective sample: the standard error of a backtest mean in Sharpe units is $1/\sqrt{n_{\mathrm{eff}}}$. When the mechanism noise is heteroskedastic across regimes ($\chi_\sigma > 1$ in Axiom~\ref{ax:snr}), $\sigma^2$ in the noise term is the occupation average $\mathbb{E}_\nu[\sigma^2(Z)]$ and nothing else changes; read in units of the unconditional $\mathrm{Var}(Y)$ instead, that term is $H(1-\rho_{\mathrm{blk}}^2)$ rather than $H$, which matters only where $v(n,\ell) < H$.
\end{lemma}
The persistence of the state enters \emph{multiplied by $\rho_{\mathrm{blk}}^2$} --- a candidate whose edge does not depend on the regime pays nothing for $\ell$, one whose edge swings by $\rho\sigma$ across regimes pays $\rho^2(2\ell-1)$, comparable to the noise term rather than sixty times it --- while the overlap enters additively at full weight, and neither effect is $\ell^* = \max(\ell, H, L_x)$: the block count $n/\ell^*$ is the right sample for a statistic of \emph{blocks} (\S\ref{sec:selection}), not for one that averages periods. What $\upsilon$ comes to on real series is Experiment 3's measurement, read against the box Remark~\ref{rem:upsbox} puts $\upsilon$ in.

\subsection{Capacity}\label{capacity}
\begin{theorem}[Capacity ceiling]\label{thm:capacity}
Let $\mathcal{H}$ have effective dimension $d$ (pseudo- or Rademacher dimension; for ridge regression with penalty $\lambda$, the effective degrees of freedom $d_{\mathrm{eff}}(\lambda) = \sum_i \nu_i/(\nu_i+\lambda)$ over the eigenvalues $\nu_i$ of the feature second-moment matrix). Under Axioms~\ref{ax:invariance}, \ref{ax:persistence} and \ref{ax:snr}, empirical risk minimisation over $\mathcal{H}$ satisfies, for the squared loss in the well-specified case and features persistent at the label scale,
\[ \mathbb{E}\big[\bar R_n(\hat h)\big] - \inf_{h\in\mathcal{H}} \bar R_n(h) \;\lesssim\; \frac{\sigma^2 d}{n_{\mathrm{fit}}}, \qquad n_{\mathrm{fit}} = \frac{n}{H}, \]
up to a regime-sampling term of relative size $O(\rho^2\chi_\sigma^2)$, and this rate is attained by the least-squares estimator in a $d$-dimensional linear class. We state an upper bound; Proposition~\ref{prop:minimax} supplies the matching lower bound, in the model with $n_{\mathrm{fit}}$ independent observations and on the full sample of $n$ overlapping labels when the features persist at the label scale, which makes the budget a ceiling no estimator can beat \emph{in those models}. Since the \emph{entire} attainable risk reduction is $\mathbb{E}[\mu(X)^2] \le \rho^2\sigma^2$ by Axiom~\ref{ax:snr}, a strategy class can improve on $h\equiv 0$ only if
\begin{equation}\label{eq:ceiling}
\boxed{d \;\lesssim\; \rho^2\, n_{\mathrm{fit}} \;=\; \frac{\rho^2 n}{H}}
\end{equation}
\end{theorem}
The persistence of the state does not appear, and that is the content of the theorem: the excess risk moves by a relative $O(\rho^2\chi_\sigma^2)$, and a feature that itself persists is an \emph{observed coordinate of the state} --- its memory enters $\ell^*$, its persistence enters $\rho_{\mathrm{blk}}$, nothing else changes. An account that means to stay conservative charges the capacity term $\max(H,\hat\upsilon)$; whether real series sit below $H$ is part of what the experiments measure (Experiment 3).

\begin{remark}[Constants and rates]\label{rem:constants}
Equation \eqref{eq:ceiling} fixes the \emph{scaling}, not the constant, which depends on the objective and the feature spectrum; the constant is a computation once both are fixed, and the optimum it fixes is interior --- a \emph{budget utilisation} $\varkappa := d/(\rho^2 n_{\mathrm{fit}})$ strictly below one. Corollary~\ref{cor:halfkelly} is stated in terms of $\varkappa$ for that reason; nothing else depends on the constant. Equation \eqref{eq:ceiling} also uses the fast rate $\sigma^2 d/n_{\mathrm{fit}}$, valid for well-specified squared loss; in the agnostic regime the rate is $\sigma^2\sqrt{d/n_{\mathrm{fit}}}$ and the ceiling becomes $d \lesssim \rho^4 n_{\mathrm{fit}}$ --- \emph{quartic} in the signal, hence dramatically tighter. We use the fast rate throughout, which makes every number below an \emph{optimistic} upper bound.
\end{remark}
\subsection{Search}\label{search}
The same budget is consumed by looking.

\begin{theorem}[Search budget]\label{thm:search}
(i) Fixed list. Suppose $N$ candidates are evaluated on $n$ periods, of which at most one has true per-period Sharpe ratio $\rho$ and the rest have none. Model the in-sample Sharpe ratios as $\rho_j + \epsilon_j$ with $(\epsilon_j)$ jointly Gaussian of variance $1/n_{\mathrm{eff}}$ --- the variance Lemma~\ref{lem:effsample} gives a backtest mean --- and arbitrary correlation. Then the expected in-sample maximum under the null is at most $\sqrt{2\log N/n_{\mathrm{eff}}}$ whatever the correlation, and --- in the independent case, where that maximum is attained to within a constant; (iii) is the general list --- the genuine candidate can win with probability bounded away from zero only if
\begin{equation}\label{eq:search}
\boxed{\log N \;\lesssim\; \frac{1}{2}\,\rho^2\, n_{\mathrm{eff}}}
\end{equation}

(ii) Adaptive search. Drop the requirement that the candidates were listed in advance. Let the candidates be indexed by a countable set, let the in-sample statistics be $\phi_i = \rho_i + \epsilon_i$ with $\rho_i = \mathbb{E}\phi_i$ and each $\epsilon_i$ centred sub-Gaussian of variance proxy $\sigma^2$ and arbitrary dependence across $i$ (so that for a biased statistic, such as the tail plug-in, the bias of Theorem~\ref{thm:deployed} and Proposition~\ref{prop:eiv} is charged separately, not here), and let $T$ be \emph{any} selection rule --- a measurable function of $\phi$ and of independent randomisation, including one that chooses each candidate in the light of the evaluations before it. Then
\[ \big|\mathbb{E}[\phi_T - \rho_T]\big| \le \sigma\sqrt{2\,I(T;\phi)}, \]
where $I(T;\phi)$ is the mutual information, in nats, between the choice and the statistics it was chosen from. A rule that returns one of $N$ candidates listed in advance has $I(T;\phi) \le \log N$, and with $\sigma^2 = 1/n_{\mathrm{eff}}$ the bound is that of (i), constant included; for the block-$\mathrm{CVaR}$ plug-in of Theorem~\ref{thm:deployed}, in the Gaussian range of Proposition~\ref{prop:tailprice}, $\sigma = \varsigma_\Lambda\sqrt{\upsilon/(\upsilon_b B)} = c(\Lambda)/\sqrt{n_{\mathrm{eff}}}$ in Sharpe units, the block count that of Corollary~\ref{cor:blocks} (a). The search condition therefore reads $I(T;\phi) \lesssim \frac{1}{2}\rho^2 n_{\mathrm{eff}}$, of which \eqref{eq:search} is the case of a fixed list.

(iii) Correlated list. In the model of (i), let $r_{jk}$ be the correlation of the null statistics $\epsilon_j$, $\epsilon_k$, and $d(j,k) := \sqrt{2(1 - r_{jk})}$ the standard deviation of $\epsilon_j - \epsilon_k$ in units of $n_{\mathrm{eff}}^{-1/2}$; write $N(\delta)$ for the largest number of candidates pairwise at distance at least $\delta$, and
\[ \log N_{\mathrm{eff}} \;:=\; \sup_{0 < \delta \le \sqrt{2}} \tfrac{1}{2}\,\delta^2 \log N(\delta) \]
for the \emph{search entropy} of the list: $\log N_{\mathrm{eff}} \le \log N$ always, with equality for an independent list, and $\log N_{\mathrm{eff}} \ge \frac{1}{2}\delta^2\log K$ for a list holding $K$ candidates pairwise at distance $\ge \delta$, however many near-copies of each it also holds. Then the expected null maximum satisfies
\[ c_S\sqrt{\frac{2\log N_{\mathrm{eff}}}{n_{\mathrm{eff}}}} \;\le\; \mathbb{E}\max_{j}\epsilon_j \;\le\; \frac{C_D}{\sqrt{n_{\mathrm{eff}}}}\int_0^{2}\sqrt{\log N(\varepsilon)}\,\mathrm{d}\varepsilon \;\le\; 2C_D\sqrt{\frac{\log N}{n_{\mathrm{eff}}}}, \]
with $c_S$, $C_D$ absolute constants (Sudakov's minoration and Dudley's entropy integral), and the null maximum concentrates about its mean at the scale $n_{\mathrm{eff}}^{-1/2}$. Hence the genuine candidate wins with probability bounded away from zero only if $\log N_{\mathrm{eff}} \lesssim \frac{1}{2}\rho^2 n_{\mathrm{eff}}$, and wins with probability tending to one once $\rho$ exceeds the entropy integral by a margin of order $n_{\mathrm{eff}}^{-1/2}$: the search charge is the metric entropy of the list in the correlation metric, of which \eqref{eq:search} is the independent case and, for every list, the upper reading. In the Gaussian range of Proposition~\ref{prop:tailprice} the same holds for the block-$\mathrm{CVaR}$ plug-in at the scale $c(\Lambda)/\sqrt{n_{\mathrm{eff}}}$, with the metric that of the truncated statistics $\max(\hat{r}_j, u)$ across candidates --- no closer than the scores' when the scores' correlations are non-negative --- so that for a location family the condition of Proposition~\ref{prop:tailprice} reads $\log N_{\mathrm{eff}}^{(\Lambda)} \lesssim \rho^2 n_{\mathrm{eff}}/(2c(\Lambda)^2)$ in both directions.
\end{theorem}

Part (ii) is the form in which the rest of the paper should be read, because the search a researcher actually runs is adaptive and ``the number of trials'' is not a quantity that exists: what replaces it is the information the search used, not observable any more than $N$ was --- what is observable is the left-hand side of (ii), the bias of the selection actually run, which Proposition~\ref{prop:audit} measures. One conservatism in (i), for a null candidate whose edge is zero in every regime and not only on average: such a candidate has no regime term, so its null maximum is $\sqrt{2H\log N/n}$ rather than $\sqrt{2\upsilon\log N/n}$, and the deflation of S4 inherits the slack; a candidate that wins in one regime and loses in another averages to the same zero and pays the regime term in full. And one entry the account would otherwise miss: a menu $\Psi$ of representations declared at S1, each with its constants, and searched for the best certificate of Proposition~\ref{prop:representation}, is a fixed list of $|\Psi|N$ pairs, so (ii) charges it $\log|\Psi|$ beside $\log N$ --- what the charge cannot buy is the truth of any member's declared defect, which stays with the test of \S\ref{sec:empirical}, run member by member. Part (iii) says what $N$ is when the list is correlated, as every real library is: its search entropy, computable from the candidates' training-window correlations before any selection is made, and the object Experiment 5 reads its search slope against.
\subsection{One budget}\label{one-budget}
Theorems~\ref{thm:capacity} and \ref{thm:search} are the same statement, and the statement is a theorem. A model with $d$ effective parameters and a search over $N$ configurations together specify a strategy requiring roughly $d + 2\log N$ nats to describe given the data --- or $d + 2I(T;\phi)$ when the search was adaptive, Theorem~\ref{thm:search} (ii) --- and what a PAC-Bayes bound \cite{mcallester1999pacbayes,catoni2007pacbayes} charges against the sample is exactly a description length: the divergence of the researcher's final choice, read as a law $\pi$ on the class, from a prior $\pi_0$ written down before the data. Theorem~\ref{thm:pacbayes} proves that bound under the axioms, with Lemma~\ref{lem:effsample}'s charge --- persistence paid in proportion to regime dispersion --- inside the exponential inequality rather than in a variance, and says where the variance stops being the whole story. One convention in the box: the capacity term is Theorem~\ref{thm:capacity}'s, read against $n_{\mathrm{fit}} = n/H$, and the search term is read against $n_{\mathrm{eff}} = n/\upsilon$; written on one line against $n_{\mathrm{eff}}$ the capacity term carries the factor $\upsilon/H$, which the box takes as one and a ledger keeps (Experiment 4 prints the two budgets apart).

\begin{center}
\fbox{\parbox{0.92\linewidth}{\centering
\textbf{The information budget.}
\[ \underbrace{d_{\mathrm{eff}}}_{\text{model capacity}} \; + \; \underbrace{2\log N}_{\text{search effort}} \; + \; \underbrace{2\,\mathrm{KL}(\text{posterior} \| \text{prior})}_{\text{everything else you did}} \quad \lesssim \quad \rho^2\,n_{\mathrm{eff}} \; \text{nats.} \]
}}
\end{center}
\begin{theorem}[The budget is a theorem: a PAC-Bayes bound under the axioms]\label{thm:pacbayes}
Let the declared state follow the redraw chain of Lemma~\ref{lem:effsample}, with coherence time $\ell \ge 2$, stationary law $\nu$ and $\rho_\ell := 1 - 1/\ell$.

(i) The regime cumulant. Let the chain start from $\nu$, let $f$ be $\nu$-centred with $\mathrm{Var}_\nu(f) = v$ and $|f| \le m$, and let $\lambda \ge 0$ satisfy $(\ell-1)(e^{\lambda m}-1) < 1$. Then
\[ \mathbb{E}_\nu \exp\Big(\lambda\sum_{t\le n} f(Z_t)\Big) \;\le\; 2e^{\lambda m}\,e^{n\omega(\lambda)}, \]
where $\omega(\lambda)\in[0,\lambda m]$ is the root of $\mathbb{E}_\nu\big[(e^{\lambda f - \omega}-1)/(1-\rho_\ell e^{\lambda f-\omega})\big] = 0$. It satisfies $\omega(\lambda) \le (\lambda^2 v + \omega(\lambda)^2)\,C_\ell(\lambda m)$ with $C_\ell(a) := a^{-2}\big[e^a - 1 - a + (\ell-1)(e^a-1)^2/(1-(\ell-1)(e^a-1))\big]$, $C_\ell(0+) = \ell - 1/2$; and when $\lambda m \ell \le 1/4$ and $v \le m^2/4$,
\begin{equation}\label{eq:regimecumulant}
\omega(\lambda) \;\le\; \frac{(2\ell-1)\,\lambda^2 v}{2\,(1-4\lambda m\ell/3)}.
\end{equation}

(ii) The bound. Let Axiom~\ref{ax:invariance} hold at defect zero. For a candidate $h$ write $r_h$ for its regime risk, $R(h) := \mathbb{E}_\nu r_h$ for its stationary risk, $\hat R_n(h)$ for its backtest mean and $\xi_t(h) := L(h(X_t),Y_t) - r_h(Z_t)$ for the mechanism noise, and suppose that for every candidate in the class (a) $r_h$ takes values in an interval of length $m$ --- so $|r_h - R(h)| \le m$ and $\mathrm{Var}_\nu(r_h) \le m^2/4$ --- and $\mathrm{Var}_\nu(r_h) \le \rho_{\mathrm{blk}}^2\sigma^2$: for the PnL of a rule with $|h|\le 1$, Axiom~\ref{ax:snr} gives $m = 2\rho\bar\sigma$, and the second is the declaration of $\kappa$; (b) conditionally on the path, the noise is centred, $s_\xi^2$-sub-Gaussian and $H$-dependent: $(\xi_s)_{s\le t}$ and $(\xi_s)_{s\ge t+H}$ are independent for every $t$, as the equal-weight overlap of Lemma~\ref{lem:effsample} is. Fix a prior $\pi_0$ on the class that does not depend on the sample, $\delta\in(0,1)$ and $\lambda>0$ with $\lambda m\ell \le 1/4$. Then with probability at least $1-\delta$, simultaneously for every law $\pi$ on the class --- every randomised choice of a candidate, however it was arrived at from the sample ---
\begin{equation}\label{eq:pacbayes}
\mathbb{E}_\pi[R(h)] \;\le\; \mathbb{E}_\pi[\hat R_n(h)] + \frac{\mathrm{KL}(\pi\|\pi_0) + \log(2/\delta) + \lambda m}{\lambda n} + \frac{\lambda}{2}\,V_\lambda,
\end{equation}
where $V_\lambda := Hs_\xi^2(1+H/n) + \rho_{\mathrm{blk}}^2\sigma^2(2\ell-1)/(1-4\lambda m\ell/3)$; and the same holds with $2\omega(\lambda)/\lambda^2$ in place of the second term of $V_\lambda$ for any $\lambda$ with $(\ell-1)(e^{\lambda m}-1)<1$, $\omega$ the largest cumulant of $-(r_h - R(h))$ over the class.

(iii) Blocks. When the block records are independent the block-$\mathrm{CVaR}$ plug-in of Theorem~\ref{thm:deployed} is a function of $B$ block scores with bounded differences $2\Lambda M/B$, and the same change of measure gives, with probability $1-\delta$ and for every $\pi$ with $\mathrm{KL}(\pi\|\pi_0) \le \bar K$ fixed in advance, $\mathbb{E}_\pi \mathrm{CVaR}_{1/\Lambda}(r_h^{(b)}) \le \mathbb{E}_\pi \hat{\mathrm{CVaR}}_{1/\Lambda}(h) + \Lambda M\sqrt{2(\bar K + \log(1/\delta))/B} + \text{bias}$, the bias downward as in Theorem~\ref{thm:deployed}: the finite-sample form of that theorem's fluctuation, with $\log 2N$ replaced by a divergence.
\end{theorem}
\textbf{The constant:} at the balancing $\lambda$ the selected candidate's stationary risk exceeds its backtest by at most $\sigma\sqrt{2\,(\log N + \log(2/\delta))/n_{\mathrm{eff}}}$ --- Theorem~\ref{thm:search} (i) with its Gaussian model replaced by the axioms and its constant kept --- so the bound resolves a genuine candidate only where $2[\mathrm{KL}(\pi\|\pi_0) + \log(2/\delta)] \lesssim \rho^2 n_{\mathrm{eff}}$: the box, with search and $\mathrm{KL}$ read as one divergence, kept in nats at the rate $\sqrt{\mathrm{KL}/n_{\mathrm{eff}}}$ (the capacity term is the one entry not charged at its MDL code length $\frac{1}{2} d\log n_{\mathrm{eff}}$ \cite{rissanen1978modeling}, the fast-rate discount a search does not get).

\textbf{What the variance does not see:} the Bernstein denominator of \eqref{eq:regimecumulant} is the sojourn $m\ell$, and where the balancing $\lambda$ sits outside \eqref{eq:regimecumulant} the bound must be run with the root itself --- the computable root of the renewal equation; the worked example's own numbers sit at that crossover.

\textbf{Where the theorem sits:} exponential inequalities for Markov chains are a mature literature \cite{lezaud1998chernoff,leon2004optimal,paulin2015concentration,adamczak2008tail}, and any of them delivers (i)'s scaling on the redraw chain; what (i) adds is the exact object --- the cumulant as the computable root of a renewal equation, the variance constant $2\ell-1$ itself, the Bernstein scale identified as the sojourn. Likewise (ii) joins the PAC-Bayes bounds for dependent data \cite{ralaivola2010chromatic,seldin2012martingales,alquier2012model,alquier2024userfriendly} and differs in where the dependence is charged: those deflate the whole sample by the observable series' memory, while (ii) splits the loss at the declared state and charges each part its own clock --- $H$ for the mechanism noise, $\rho_{\mathrm{blk}}^2(2\ell-1)$ for the regime part. On the worksheet's numbers that is the difference between the twenty-three nats the budget keeps and the $\rho^2 n/\ell \approx 0.8$ a mixing-rate charge would leave.

\textbf{It is small where it matters.} A single instrument, daily bars and daily labels, twenty years ($n = 5000$), $\ell = 60$, a generous daily $R^2$ of $1\%$ and $\rho_{\mathrm{blk}} = \rho$ give $\upsilon = 2.2$ and $\rho^2 n_{\mathrm{eff}} \approx 23$ nats: about twenty parameters, or $e^{11}$ configurations tested on the mean; monthly labels give $2.3$ nats --- two parameters, three configurations. A researcher who has tried five hundred variants of a moving-average crossover has spent $2\log 500 \approx 12$ nats --- half her budget on the mean, the whole of it on the quarter-tail at the Gaussian price of \S\ref{sec:price}, and twice it at the price real block scores pay (\S\ref{sec:calibration}) --- before fitting anything; the worksheet of \S\ref{sec:calibration} carries the arithmetic at measured values, and where the budget locates the scarcity is not the periods available to fit but the periods available to \emph{see a regime}. Three ways to buy budget --- raise $n$ (with the caveat of \S\ref{sec:crosssection}), lower $\upsilon$ (shorter labels cut $H$; signals whose validity does not depend on the slow regime cut $\rho_{\mathrm{blk}}$, the only way $\ell$ is ever paid for), raise $\rho$ (worth four times the same effort on model class, the ceiling being quadratic) --- and no fourth: a better optimiser, a deeper network, more features spend budget, they do not create it.

\subsection{Cross-section versus timing}\label{sec:crosssection}
The caveat in (a) is worth stating separately, because here the theory delivers a result that is folk wisdom without a derivation.

\begin{proposition}[Breadth buys down the noise term, not the common regime term]\label{prop:breadth}
Let there be $N_a$ instruments with $\mu_{i,t} = \alpha_t + \beta_{i,t}$, where $\alpha_t$ is a common (timing) component driven by the latent state and $\beta_{i,t}$ is cross-sectional with $\sum_i \beta_{i,t} = 0$ and instrument-specific noise. Let $\Sigma_e$ be the cross-sectional covariance of the instrument-specific noise, $\sigma^2$ its average diagonal entry, and define the \emph{effective breadth} $N_a^{\mathrm{eff}} := N_a^2\sigma^2/(\mathbf{1}^\top \Sigma_e \mathbf{1})$, the number of independent instruments whose equal-weighted average would have the noise variance of the actual one; when the residual correlations average $\bar r$ this is $N_a/(1+(N_a-1)\bar r) \le 1/\bar r$. Then, for the equal-weighted cross-sectional statistic,
\[ n_{\mathrm{eff}}(\alpha) = \frac{n}{H + \rho_{\mathrm{blk}}^2(2\ell-1)}, \qquad n_{\mathrm{eff}}(\beta) \asymp \frac{n}{H/N_a^{\mathrm{eff}} + \rho_{\mathrm{blk},\beta}^2(2\ell-1)}, \]
where $\rho_{\mathrm{blk},\beta}$ is the regime dispersion, in Sharpe units, of the cross-sectional average $N_a^{-1}\sum_i \beta_i(Z)$, which for a signal whose validity does not depend on the macro regime is near zero. The capacity ceiling for a cross-sectional model therefore exceeds that for a timing model by the factor
\[ \frac{n_{\mathrm{eff}}(\beta)}{n_{\mathrm{eff}}(\alpha)} = N_a^{\mathrm{eff}} \cdot \frac{1 + \rho_{\mathrm{blk}}^2(2\ell-1)/H}{1 + N_a^{\mathrm{eff}}\rho_{\mathrm{blk},\beta}^2(2\ell-1)/H}, \]
which is $N_a^{\mathrm{eff}}$ when both regime terms are negligible against $H$, \emph{more} than $N_a^{\mathrm{eff}}$ when the timing signal's regime dispersion exceeds $N_a^{\mathrm{eff}}$ times the cross-sectional average's, and less otherwise: the breadth is paid out on the noise term and on the regime term alike, and a timing signal that carries a large regime term loses to the cross-section by more than the count of its instruments.
\end{proposition}
Adding instruments does nothing for the \emph{coherence} of the regime term, because every instrument sees the \emph{same} latent path; what it does for the regime \emph{dispersion} depends on how much of it the instruments share. Everything in the proposition is measurable, and the measurement is the one the cross-sectional library exists to make (Experiment 6; \S\ref{sec:calibration}, Appendix~\ref{app:experiments}).

\section{The price of robustness}\label{sec:price}
Sections \ref{sec:selection} and \ref{sec:budget} pull in different directions: the first says the tail is the right criterion, the second that every comparison is paid for out of one budget, and neither says what a comparison on the tail costs. This section assembles the pieces into one inequality and computes what the tail costs --- and finds that this, not fitting, is where the budget binds.

\subsection{The deployed-risk bound}\label{the-deployed-risk-bound}
\begin{theorem}[Deployed risk]\label{thm:deployed}
Assume Axioms \ref{ax:adapted}--\ref{ax:snr}. Fix the block length $b \ge \ell^*$ at which Axiom~\ref{ax:recurrence} was declared, let $B = \lfloor n/b \rfloor$, and let $\mathcal{H}_N$ be a candidate set of size $N$ that does not depend on the sample (for a sample-dependent search the selection term holds in expectation with $\log N$ replaced by $I(T;\phi)$, Theorem~\ref{thm:search} (ii)). Let $\hat h \in \operatorname{arg\,min}_{h\in\mathcal{H}_N} \hat{\mathrm{CVaR}}_{1/\Lambda}(h)$. Write $r_h^{(b)}$ for a block score of $h$ --- the block-averaged regime risk plus the conditionally mean-zero noise of \eqref{eq:blockscore} --- and $\mathrm{CVaR}_{1/\Lambda}(r_h^{(b)})$ for the tail mean of its law. Then, conditionally on the path,
\begin{equation}\label{eq:deployed}
\boxed{R^+(\hat h) \le \underbrace{\min_{h\in\mathcal{H}_N} \mathrm{CVaR}_{1/\Lambda}(r_h^{(b)})}_{\text{(i) the best the class can do, read through blocks}} + \underbrace{2\,\Delta_N}_{\text{(ii) selection error}}}
\end{equation}
where $\Delta_N$ is the larger of two suprema over the class, $\sup_{h}|\hat{\mathrm{CVaR}}_{1/\Lambda}(h) - \mathrm{CVaR}_{1/\Lambda}(r_h^{(b)})|$, the plug-in's distance from the population tail, and $\sup_{h}(\mathbb{E}[\hat{\mathrm{CVaR}}_{1/\Lambda}(h) \mid Z] - \hat{\mathrm{CVaR}}_{1/\Lambda}(h))^+$, how far below its path-conditional mean a plug-in can sit --- the second is the price of choosing $\hat h$ from the scores, and both are of the order given below. Term (i) is an upper bound on the class optimum at this scale, $\min_h \mathrm{CVaR}_{1/\Lambda}(R(h))$: $\mathrm{CVaR}$ is convex and law-invariant, so adding conditionally mean-zero noise can only raise it, and the excess --- the price of reading a block through its own noisy score --- is governed by $b$ and does not shrink with $B$ (\S\ref{sec:limits}). Nothing in the chain reads a state-scale assertion off a block statistic: the axiom is declared at $b$ and the criterion is computed at $b$. What that alignment costs is Proposition~\ref{prop:scale} --- the assertion is a weaker one than the same number at $b = 1$ --- and it is charged there rather than as a term here. Term (ii) is a fluctuation and a bias. The bias, $\mathbb{E}\hat{\mathrm{CVaR}}_{1/\Lambda}(h) - \mathrm{CVaR}_{1/\Lambda}(r_h^{(b)})$, is downward by Jensen and of smaller order than the fluctuation in the Gaussian range of Proposition~\ref{prop:tailprice}, where the plug-in's limit is centred at the population value. The fluctuation is the deviation of each plug-in from its expectation, and the block scores being worth $B\upsilon_b/\upsilon$ independent ones, hence in that range the tail truncation of a score too (Corollary~\ref{cor:blocks} (a)),
\[ \Delta_N \;\approx\; c(\Lambda)\,\mathrm{sd}_{\mathrm{blk}}(r^{(b)})\,\sqrt{\frac{2\,\upsilon\log N}{\upsilon_b\,B}} \;=\; c(\Lambda)\,\sigma\,\sqrt{\frac{2\log N}{n_{\mathrm{eff}}}}, \]
with $c(\Lambda)$ the constant of Proposition~\ref{prop:tailprice} and $\mathrm{sd}_{\mathrm{blk}}(r^{(b)})^2 = \sigma^2\upsilon_b/b$ (Lemma~\ref{lem:effsample} at $n = b$): the dependence of the records enters through the effective sample, the same $\upsilon$ the mean pays, here and nowhere else, and the block length cancels. On independent block records the finite-sample form is McDiarmid's --- with probability at least $1-\delta$, $\Delta_N \le \Lambda M\sqrt{2(\log 2N + \log(1/\delta))/B} + C_0\Lambda M/\sqrt{B}$ with $C_0$ an absolute constant, the second term the bias --- which Theorem~\ref{thm:pacbayes} (iii) uses.
\end{theorem}
\begin{remark}[The block count is an identity, and the block length pulls two ways]\label{rem:blocklen}
The selection error is computed on $B\upsilon_b/\upsilon$ blocks, the count Corollary~\ref{cor:blocks} (a) derives from Lemma~\ref{lem:effsample} at the $b$ in use, and nothing asks the block to be longer than that: a bound stated under a mixing coefficient of the process would carry a residual $4M(n/b)\beta(b)$ and force $b \gtrsim \ell\log(n/b)$, while here the dependence is the one number $n_{\mathrm{eff}}$ already carries, and the fluctuation is $c(\Lambda)\sigma\sqrt{2\log N/n_{\mathrm{eff}}}$ at every block length (Remark~\ref{rem:independence} says why no further coherence is declared). Two things follow. A sample reads the count through the scores' own coherence across blocks: Experiment 1's panel B reads it on the reference series at $b = 60$, for the scores and for the tail statistic $\max(\hat r_j, u)$ alike, inside their null and below the $\upsilon/\upsilon_b = 1.52$ the worksheet's declaration $\rho_{\mathrm{blk}} = \rho$ implies at $\ell = 60$ --- the same reading of $\kappa$ the scores' autocorrelation gives (\S\ref{sec:empirical}). And the block length pulls two ways, neither of them the count: a longer block lowers the dispersion the tail reads, $\sqrt{v(b,\ell)/b}$ (Proposition~\ref{prop:shuffle}), and the strength of the assertion the criterion is implementing (Proposition~\ref{prop:scale}: at $\ell = 60$ the move from $b = 60$ to $b = 300$ takes a declared $\Lambda = 4$ from a state-scale reading of $3.0$ to one of $1.8$), while a shorter one raises the score-noise bias of Proposition~\ref{prop:eiv} (ii), whose lever runs out at $b \approx \ell$. Stage S3's block length is therefore not free in two directions at once, and \S\ref{sec:nonempty} is where the two meet.
\end{remark}
Inequality \eqref{eq:deployed} is the statement the paper is organised around, and its two terms carry four of the five stages of \S\ref{sec:canonical}: (i) is lowered by a representation with small $\Lambda$ (S1) and a class rich enough to contain a good rule (S2); (ii) is the price of \emph{having looked}, growing with $N$ and with the tail level (S4), and computed on the blocks contiguity leaves --- of length $\ge \ell^*$, purged and embargoed, worth $n_{\mathrm{eff}}/b$ (S3) --- which is what random $k$-fold cross-validation cannot control (Proposition~\ref{prop:shuffle}). Nothing else appears: that is the sense in which the architecture is forced.

The theorem also says where robustness is charged. Increasing $\Lambda$ \emph{raises} (i) pointwise and raises (ii); what it buys is that (i) becomes a bound on a future the researcher actually believes possible. Raising $\Lambda$ beyond what one believes is a pure loss: ``more robustness is safer'' is wrong here in both directions.

\subsection{What the tail costs}\label{what-the-tail-costs}
The remaining question is the size of the constant in term (ii). The asymptotic law is the standard limit theory of the empirical expected shortfall \cite{rockafellar2000cvar,chen2008nonparametric,zwingmann2016asymptotics}, and we claim no novelty for it; what we have not found stated is its \emph{reading} --- that $\varsigma_\Lambda/s$ is the divisor on the effective sample available to \emph{search}, that the constant is far below its Lipschitz bound for smooth laws, and that it is not monotone in $\Lambda$.

\begin{proposition}[The price of the tail]\label{prop:tailprice}
Let $r_1,\dots,r_B$ be the block risks, i.i.d.~under $\bar\pi_n$ with variance $s^2$ --- losses, so that $\mathrm{CVaR}$ is the upper tail as in \S\ref{sec:selection} --- and let $u$ be their $(1-1/\Lambda)$-quantile. Then
\[ \sqrt{B}\,\bigg(\hat{\mathrm{CVaR}}_{1/\Lambda} - \mathrm{CVaR}_{1/\Lambda}\bigg) \;\Rightarrow\; N\big(0,\varsigma_\Lambda^2\big), \qquad \varsigma_\Lambda^2 = \Lambda^2\,\mathrm{Var}\big(\max(r,u)\big), \]
so writing $c(\Lambda):= \varsigma_\Lambda/s$ for the inflation of the ranking statistic's standard error relative to the block mean, the effective sample available to the \emph{selection} stage is $n_{\mathrm{eff}}/c(\Lambda)^2$ and \eqref{eq:search} becomes
\begin{equation}\label{eq:searchrobust}
\log N \;\lesssim\; \frac{\rho^2\,n_{\mathrm{eff}}}{2\,c(\Lambda)^2}.
\end{equation}
Furthermore $c(1)=1$ and $c(\Lambda)\le \Lambda$ for every $\Lambda$ and every block-risk law, because $x\mapsto \max(x,u)$ is $1$-Lipschitz.
\end{proposition}
The step from $\varsigma_\Lambda/\sqrt{B}$ to ``$n_{\mathrm{eff}}/c(\Lambda)^2$'' is Lemma~\ref{lem:effsample} at block length $b$ together with Corollary~\ref{cor:blocks} (a): the fluctuation runs on $B\upsilon_b/\upsilon$ blocks and $\varsigma_\Lambda^2 = c(\Lambda)^2\sigma^2\upsilon_b/b$, so $\varsigma_\Lambda^2\upsilon/(\upsilon_b B) = c(\Lambda)^2\sigma^2\upsilon/n = c(\Lambda)^2\sigma^2/n_{\mathrm{eff}}$ --- the block score's record coherence is exactly the factor that turns the lemma at $n = b$ into the lemma at $n = Bb$, and $n_{\mathrm{eff}}$ computed at the full length is the whole account of the records' dependence, on the tail rows of a ledger as on its mean row. Experiment 1 reads the scores' and the tail statistic's coherence inside their null on the reference series (panel B). Two cross-references: $s^2$ is the variance of the observed block \emph{scores}, not of the latent risks, the subject of Proposition~\ref{prop:eiv}; and the $c(\Lambda)$ of a real library is a measurement, \S\ref{sec:calibration}.

\begin{remark}[Orientation]\label{rem:orientation}
Losses and rewards give the same $c(\Lambda)$: for the mirrored reward $-r$, $\max(r,u) = -\min(-r,-u)$ and the variance is unchanged. The prose works in losses to match \S\ref{sec:selection}; the experiments work in rewards, computing $\Lambda^2\mathrm{Var}(\min(r,q))/s^2$ with $q$ the $1/\Lambda$-quantile. Check which convention is in force before reading a quantile level.
\end{remark}
\begin{corollary}[The search condition beyond location families]\label{cor:locfam}
In the setting of Theorem~\ref{thm:deployed}, let one candidate $h^\star \in \mathcal{H}_N$ satisfy $\mathrm{CVaR}_{1/\Lambda}(r_{h^\star}^{(b)}) \le \mathrm{CVaR}_{1/\Lambda}(r_h^{(b)}) - \Delta_{\mathrm{CVaR}}$ for every other candidate, with $\Delta_{\mathrm{CVaR}} > 0$ the tail-mean separation in the plug-in's own units of block risk. Whenever $\Delta_{\mathrm{CVaR}} > 2\Delta_N$ the plug-in ranks $h^\star$ first, so in the Gaussian range of Theorem~\ref{thm:deployed}'s bound on $\Delta_N$ it does so once
\[ \log N \;\lesssim\; \frac{\Delta_{\mathrm{CVaR}}^2\, B\,\upsilon_b}{2\,\varsigma_\Lambda^2\,\upsilon} \;=\; \frac{\Delta_{\mathrm{CVaR}}^2\, n_{\mathrm{eff}}}{2\,c(\Lambda)^2\sigma^2}, \]
and on independent block records, with probability at least $1-\delta$, once $\Delta_{\mathrm{CVaR}} > 2\Lambda M\sqrt{2(\log 2N + \log(1/\delta))/B} + 2C_0\Lambda M/\sqrt{B}$.
\eqref{eq:searchrobust} is the location-family special case: $\mathrm{CVaR}_{1/\Lambda}$ is translation-equivariant, so candidates whose block-risk laws differ by a location shift are separated at the tail by exactly the margin the mean gives, $\Delta_{\mathrm{CVaR}} = \rho\sigma$ for a candidate at the ceiling; the equality of the two displays is Corollary~\ref{cor:blocks} (a)'s identity, as above.
\end{corollary}
What the corollary does and does not license. Substituting $c(\Lambda)/\sqrt{n_{\mathrm{eff}}}$ for $1/\sqrt{n_{\mathrm{eff}}}$ in \eqref{eq:search} prices the ranking statistic's \emph{noise} while holding its \emph{signal} fixed --- exact in the location family and only there. Where robustness genuinely buys something, $\Delta_{\mathrm{CVaR}}$ exceeds the mean separation and \eqref{eq:searchrobust} is conservative --- while the score noise flattens the separation as it flattens the penalty (Proposition~\ref{prop:eiv} (ii)). The necessity direction is Theorem~\ref{thm:search} (iii)'s: in the Gaussian range and for a location family the tail's charge is the search entropy of the truncated statistics at the scale $c(\Lambda)$, no smaller than the scores' when the scores are non-negatively correlated, so \eqref{eq:searchrobust} is two-sided there and one-sided beyond it. The special case remains the one the budget arithmetic uses.

The bound $c(\Lambda)\le \Lambda$ is attained, and where a problem sits inside it is a property of the \emph{shape} of the block-risk law rather than of $\Lambda$. Two cases bracket the range.

\textbf{Smooth block risks: the tail is cheap.} For Gaussian block risks, with $u = \Phi^{-1}(1-1/\Lambda)$,
\[ c(\Lambda)^2 = \Lambda^2\bigg[u^2\Big(1-\frac{1}{\Lambda}\Big) + \frac{1}{\Lambda} + u\varphi(u) - \Big(u\Big(1-\frac{1}{\Lambda}\Big)+\varphi(u)\Big)^2\bigg], \]
which is small: $c(2)=1.17$, $c(4)=1.43$, $c(8)=1.78$, $c(20)=2.47$ --- over the practical range $c(\Lambda)\approx \Lambda^{1/4}$, so the tail costs about a factor $\sqrt{\Lambda}$ of effective sample, far less than the factor $\Lambda$ one would guess from the fact that $\hat{\mathrm{CVaR}}_{1/\Lambda}$ averages only $B/\Lambda$ blocks, because the estimator \emph{locates} its own tail using all $B$ of them.

\textbf{A rare discrete regime: the tail is dear.} Now let the block-risk law put mass $p$ on a bad regime and the rest on a good one. For $1/\Lambda > p$ the tail mean is a \emph{count} of bad blocks rescaled by $\Lambda$, and $c(\Lambda) = \Lambda$ exactly; for $1/\Lambda < p$ the tail sits strictly inside the bad regime and $c(\Lambda)$ collapses towards zero. The price of the tail is therefore \emph{not monotone in $\Lambda$}: it peaks where the tail level straddles a regime boundary and falls away on either side. Which branch a real library sits on is a measurement and not a consequence of the axioms --- the dichotomy is what the theory settles, and no more; \S\ref{sec:calibration} reports what the libraries of Experiments 1, 2 and 6 read, and solves the two corollaries below at that reading.

\begin{corollary}[The budget, with robustness priced in]\label{cor:budgetrobust}
Combining \eqref{eq:ceiling} and \eqref{eq:searchrobust}:
\[ d_{\mathrm{eff}} + 2c(\Lambda)^2\log N \;\lesssim\; \rho^2 n_{\mathrm{eff}}. \]
The tail enters at selection and nowhere else: stage S2 fits on the mean (Theorem~\ref{thm:capacity}), and what the tail costs is the factor $c(\Lambda)^2$ on the search, not on the fit. The capacity term is read against $n_{\mathrm{fit}} = n/H$ and enters this line at the factor $\upsilon/H$ (\S\ref{one-budget}); the display takes $\upsilon = H$, and a ledger at any other value keeps the two denominators apart.
\end{corollary}
\begin{corollary}[A budget caps the robustness it can buy]\label{cor:lamcap}
Selecting among $N$ candidates at recurrence level $\Lambda$ requires $2c(\Lambda)^2\log N \le \rho^2 n_{\mathrm{eff}}$. Wherever the tail statistic is at least as noisy as the mean, $c(\Lambda)\ge 1$ --- every Gaussian entry of the table in \S\ref{sec:calibration} is such a level --- and the constraint is then infeasible once $\rho^2 n_{\mathrm{eff}} < 2\log N$: a researcher below her budget cannot buy robustness, because she cannot afford the comparison that would use it.
\end{corollary}
The two failures are distinct. When $\rho^2 n_{\mathrm{eff}} < 2\log N$ the binding constraint is the \emph{search}, already at $\Lambda = 1$: robustness is then not expensive but unreachable. When $\rho^2 n_{\mathrm{eff}} \ge 2\log N$ the constraint is a price, and Corollary~\ref{cor:lamcap} says how much $\Lambda$ it buys.

The arithmetic is uncomfortable and, we think, correct: \emph{robustness and search compete for the same budget}. The instinct to hedge a selection criterion against a worse future is right, and at realistic budgets it must be traded against the number of ideas one can test. Nor is the price the plug-in's to avoid: on a law the axioms admit, every selection rule pays at least the factor $\Lambda - 1$ of it, the latent state handed to it included (Proposition~\ref{prop:tailfloor}).

\subsection{The third cost: reading regimes through noisy blocks}\label{sec:eiv}
Proposition~\ref{prop:tailprice} prices the \emph{variance} of the tail statistic. There is a bias as well, with a sign and, in the Gaussian case, a closed form: the ranking statistic is the tail of block \emph{scores}, and a block score is the block's regime risk plus estimation noise that no amount of history removes; the block length is a weak lever against it, running out at $b\approx \ell$ (part (ii) below).

\begin{proposition}[The block-noise bias is a dilation penalty]\label{prop:eiv}
Let $r$ be the block-averaged regime risk of a candidate and $r^{(b)} = r + \xi$ its block score, with $\mathbb{E}[\xi\mid r] = 0$ and $\mathrm{Var}(\xi\mid r) = \sigma_\xi^2(r)$, allowed to depend on $r$. Then:

\begin{enumerate}
\item[(i)] $\mathrm{CVaR}_\alpha(r^{(b)}) \ge \mathrm{CVaR}_\alpha(r)$ for every $\alpha\in(0,1]$ and every law of $(r,\xi)$.

\item[(ii)] If $r\sim N(\bar r, s^2)$ and $\xi\sim N(0,\sigma_\xi^2)$ independent, then
\[ \mathrm{CVaR}_{1/\Lambda}(r^{(b)}) - \mathrm{CVaR}_{1/\Lambda}(r) = \lambda(\Lambda)\big(\sqrt{s^2+\sigma_\xi^2} - s\big), \]
which is $\approx \lambda(\Lambda)\sigma_\xi^2/(2s)$ when $\sigma_\xi \ll s$ and $\approx \lambda(\Lambda)\sigma_\xi$ when $\sigma_\xi \gg s$. What ends the decay is the coherence time, not the size of the noise. Writing $\sigma_\xi^2 = \sigma^2 H/b$ for the block-mean noise and $s^2 = \rho_{\mathrm{blk}}^2\sigma^2 v(b,\ell)/b$ for the block-signal variance (Lemma~\ref{lem:effsample}): for $b\lesssim \ell$, $v\approx b$ holds $s$ fixed and the bias decays as $b^{-1/2}$ while $\sigma_\xi \gtrsim s$ --- the whole of that range at the calibration of \S\ref{sec:calibration}, the crossover $b \approx H/\rho_{\mathrm{blk}}^2$ sitting past $\ell$ --- and as $b^{-1}$ only beyond it; for $b\gg \ell$, $v$ saturates at $2\ell-1$, so $s$ falls as $b^{-1/2}$ and the bias falls with it, while the ratio that governs the whole phenomenon does not fall at all:
\[ \sigma_\xi^2/s^2 = H/\big(\rho_{\mathrm{blk}}^2 v(b,\ell)\big) \;\to\; H/\big(\rho_{\mathrm{blk}}^2(2\ell-1)\big) = H/\big(2\kappa^2\mathrm{SR}_{\mathrm{ann}}^2\tau_{\mathrm{yr}}\big), \]
free of $b$ and, like the box of Remark~\ref{rem:upsbox}, free of the sampling frequency. A longer block is therefore not a remedy for (ii) beyond $b\approx \ell$.

\item[(iii)] The bias does not depend on $B$: at fixed $b$ the plug-in $\hat{\mathrm{CVaR}}_{1/\Lambda}$ is not a consistent estimator of $\mathrm{CVaR}_{1/\Lambda}(r)$. And for two candidates with the same law of $r$ and noise variances $\sigma_1^2 < \sigma_2^2$, it penalises the noisier one by $\lambda(\Lambda)(\sqrt{s^2+\sigma_2^2} - \sqrt{s^2+\sigma_1^2})$ in expectation: heteroskedastic noise across candidates is a ranking bias, in the direction of the candidate whose bad blocks are also its noisy blocks.
\end{enumerate}
\end{proposition}
Part (iii) is the reading the experiments give the real series' excess over the Gaussian value: studentising each block by its own noise removes the excess on every asset (Experiment 2), and with it exactly the stressed blocks the tail exists to read --- the studentised statistic is that of a volatility-targeted deployment, not of the rule Theorem~\ref{thm:cvar} is about. The other lever, a narrower tail level, does not restore what the noise removes: under (ii) the criterion at level $1/\Lambda'$ ranks by $\bar r - \lambda(\Lambda')\sqrt{s^2+\sigma_\xi^2}$, whose sensitivity to the block-signal dispersion $s$ is flattened by the factor $1/\sqrt{1+\sigma_\xi^2/s^2}$ at every level, so the ranking drifts toward the mean's wherever the noise ratio is large --- as it is on Experiment 1's series (\S\ref{sec:limits}) --- and a level narrow enough to buy the flattening back averages too few blocks to be a tail mean at all (Proposition~\ref{prop:occfloor}'s floor, Corollary~\ref{cor:nonempty} (ii)). The bias is priced, not removed; \S\ref{sec:limits} records what stays open.

\section{Sizing: robust Kelly and the half-Kelly fixed point}\label{sec:sizing}

\begin{proposition}[Robust fractional Kelly]\label{prop:kelly}
Let the edge estimate satisfy $\hat{\mu} = \mu + \tau\sigma\,\xi$ with $\xi\sim N(0,1)$ independent of $\mu$, and let the prior over edges induced by the strategy universe be $\mu\sim N(0,s^2\sigma^2)$. A bettor taking $h = f\hat{\mu}/\sigma^2$ attains expected log growth $g(f)= fs^2 - \frac{1}{2}f^2(s^2 + \tau^2)+ O(f^2 s^2 (s^2 + \tau^2) + f^3(s^2 + \tau^2)^{3/2})$, maximised at the growth-optimal fraction of the \citet{kelly1956new} solution,
\begin{equation}\label{eq:kelly}
f^\star = \frac{s^2}{s^2 + \tau^2}.
\end{equation}
The fraction is not log growth's: for any objective whose quadratic approximation is the mean less $\frac{\gamma}{2}$ times the second moment (\S\ref{sec:setting}), the bettor taking $h = f\hat{\mu}/(\gamma\sigma^2)$ has objective $g(f)/\gamma$ to the same order, the same $f^\star$, and only the scale $1/\gamma$ of the position is the objective's. Adding Axiom~\ref{ax:recurrence} replaces $\hat{\mu}$ by $\hat{\mu} - \lambda(\Lambda)\,\mathrm{sd}_{\mathrm{blk}}(\hat{\mu})$, where the dispersion charged is the larger of the one observed and the one the researcher has declared,
\[ \mathrm{sd}_{\mathrm{blk}}(\hat{\mu}) := \max\big(\mathrm{sd}_{\mathrm{blk}}^{\mathrm{obs}}(\hat{\mu}), \; \kappa\rho\sqrt{v(b,\ell)/b}\big)\;; \]
adding Definition~\ref{def:reflexivity} replaces it by $\delta\hat{\mu} - c'$.
\end{proposition}

Equation \eqref{eq:kelly} is a Bayesian shrinkage statement and is not new; the connection to the budget is. The penalty charges the dispersion \emph{as observed} --- score noise counted as if it were regime dispersion, the safe direction --- floored at the one declared: having declared $\rho_{\mathrm{blk}} = \kappa\rho$ at S1, the researcher has committed to that much dispersion whether or not the sample shows it. Whether the floor ever binds on a real ledger is the experiments' measurement (Experiments 1--4), and the worksheet prints the readings side by side (\S\ref{sec:calibration}).

\begin{corollary}[The Kelly fraction is at most one over one plus budget utilisation]\label{cor:halfkelly}
In Sharpe units, a strategy with $d$ effective degrees of freedom estimated on $n$ periods has estimation variance $\tau^2 \ge d/n_{\mathrm{fit}}$, $n_{\mathrm{fit}} = n/H$ the capacity sample of Theorem~\ref{thm:capacity}, with equality for an efficient estimator, while the prior variance of the edge across the pool is $s^2 \le \rho^2$ by Axiom~\ref{ax:snr}, with equality only if the pool sits at the ceiling. Writing
\[ \varkappa := \frac{d}{\rho^2 n_{\mathrm{fit}}} \]
for the fraction of the capacity budget the researcher has spent, we have $\tau^2/s^2 \ge \varkappa$ and \eqref{eq:kelly} becomes
\begin{equation}\label{eq:kappakelly}
f^\star \le \frac{1}{1+\varkappa},
\end{equation}
with equality in the efficient, at-the-ceiling case. A researcher who spends her budget in full, $\varkappa = 1$, bets at most half Kelly; and at that point $g(1)= \frac{1}{2}(s^2 - \tau^2)\le 0$, so betting \emph{full} Kelly on one's own estimate has non-positive expected log growth, zero only in the equality case (Proposition~\ref{prop:fullkelly}).
\end{corollary}

\begin{remark}[Saturation is a ceiling, not an optimum]\label{rem:kappa}
Theorem~\ref{thm:capacity} does not say $\varkappa = 1$ is optimal: \eqref{eq:ceiling} is the point at which the class stops improving on $h\equiv 0$, and the optimal $d$ is interior, so the utilisation an optimising researcher runs at is below one and the cap on $f^\star$ sits above one half. The derived object is the bound \eqref{eq:kappakelly}, not the number $\frac{1}{2}$: \emph{half Kelly is the most a researcher who has spent her whole budget should bet}. That the industry's revealed fraction is near one half is evidence about how close competent pipelines run to saturation, prediction P5 of \S\ref{sec:empirical}.

One more boundary, because stages S2 and S5 of \S\ref{sec:canonical} would otherwise charge for it twice. The fraction $f^\star = s^2/(s^2 + \tau^2)$ is exactly the factor by which the posterior mean shrinks the estimate, so betting $f^\star$ Kelly on the unshrunk estimate \emph{is} betting full Kelly on its Bayes shrinkage: the shrinkage of S2 and the fraction of S5 are one decision. If S2 already outputs $\tilde{\mu} = \zeta\hat{\mu}$ for a shrinkage factor $\zeta \in (0,1]$, the growth-optimal fraction \emph{of $\tilde{\mu}$} is $s^2/(\zeta(s^2 + \tau^2))$, which is $1$ at $\zeta = f^\star$; applying \eqref{eq:kappakelly} to $\tilde{\mu}$ as well shrinks twice, in the safe direction but slack. ``Half Kelly'' is half Kelly \emph{on an unshrunk, OLS-type estimate}, equivalently full Kelly on the posterior; Proposition~\ref{prop:fullkelly} (ii) and prediction P5 are about that estimate.
\end{remark}

The half-Kelly convention \cite{maclean2011kelly} is usually defended by drawdown tolerance or parameter uncertainty in the abstract. Corollary~\ref{cor:halfkelly} converts it into a measurement, because $\varkappa$ is computable from the worksheet of \S\ref{sec:calibration}: a researcher who bets full Kelly is asserting $\varkappa \approx 0$, that she has spent almost none of her budget --- checkable, and usually false.

\subsection{Deploy an ensemble, not a fit}\label{sec:ensemble}

Corollary~\ref{cor:halfkelly}'s equality case assumes $\tau^2 = d/n_{\mathrm{fit}}$, the estimation variance of an \emph{efficient} estimator. Practical learners are not efficient, and one of them is conspicuously not.

\begin{proposition}[Averaging over training perturbations spends no budget]\label{prop:ensemble}
Let $\hat{h}^{(1)},\dots,\hat{h}^{(m)}$ be fits of the same hypothesis class, under the same hyperparameters, to $m$ perturbations of the training set (leave-one-out, subsample, or seed) in the manner of \citet{breiman1996bagging}, and let $\bar{h} = \frac{1}{m}\sum_j \hat{h}^{(j)}$. Then $\bar{h}$ has the same $\log N$ as any single member --- no data-dependent comparison was added --- and the same $d_{\mathrm{eff}}$ in the sense of the covariance trace of \S\ref{sec:calibration}, which is the quantity \eqref{eq:ceiling} constrains (this is \emph{not} a claim that the convex hull of the class is no richer than the class; for trees it is strictly richer), while its estimation variance is $\tau_{\mathrm{ens}}^2 = \tau_\infty^2 + (\tau^2 - \tau_\infty^2)/m$, where $\tau_\infty^2$ is the variance that survives averaging. Ensembling therefore raises $f^\star$ of \eqref{eq:kelly} without consuming any budget; and it cannot push $\tau^2$ below the minimax estimation variance of the class, which for a $d$-dimensional well-specified linear class is $d/n_{\mathrm{fit}}$, so it moves an inefficient learner towards the floor and never past it.
\end{proposition}

The consequence is sharpest for boosted trees, chaotically sensitive to any perturbation of the training set --- expected behaviour under Axiom~\ref{ax:snr}, a single fit being a \emph{draw} from a distribution rather than its centre --- so what is deployed and what is backtested should both be the ensemble: not a way of buying capacity but of not wasting capacity already paid for (the covariance trace is measurable, and a ledger that charges the ensemble the full fit's dimension overcharges it, \S\ref{sec:calibration}). This is the one implementation rule in the paper that the accumulated practice of the field does not contain.

\subsection{When the procedure can return a position at all}\label{sec:nonempty}

Stages S1--S4 decide what may be compared and at what price. Whether anything survives the comparison is a separate question with a closed-form answer; asking it of a real configuration is Experiment 4's job.

\begin{corollary}[Non-emptiness of S5]\label{cor:nonempty}
Let the ledger select a candidate from $N$ proposals with backtest mean $\hat{\mu}$, decay $\delta$ and cost $c'$ (Def.~\ref{def:reflexivity}), and let its block scores at block length $b$ have dispersion $\mathrm{sd}_{\mathrm{blk}}$. Stage S5 returns a non-zero position if and only if
\begin{equation}\label{eq:nonempty}
\delta\hat{\mu} - c' - \sqrt{2\log N/n_{\mathrm{eff}}} \; > \; \lambda(\Lambda)\,\mathrm{sd}_{\mathrm{blk}}.
\end{equation}
Write $E_N$ for the left-hand side. Two bounds follow before $b$ is chosen. \textbf{(i) A search ceiling.} $E_N > 0$ requires
\[ N < \exp\big(\frac{n_{\mathrm{eff}}(\delta\hat{\mu} - c')^2}{2}\big), \]
whatever the block length and whatever $\Lambda$: past that many proposals the haircut alone exhausts the edge. \textbf{(ii) A block-length and sample floor.} The block score of a block of $b$ periods carries the mechanism noise of its own periods, so $\mathrm{sd}_{\mathrm{blk}}^2 \ge \mathrm{sd}_{\mathrm{within}}^2/b$ with $\mathrm{sd}_{\mathrm{within}}$ the within-block noise sd of the score; hence \eqref{eq:nonempty} requires
\[ b > \big(\frac{\lambda(\Lambda)\,\mathrm{sd}_{\mathrm{within}}}{E_N}\big)^2, \qquad n > \Lambda b, \]
the second because the $1/\Lambda$ tail must average more than one block --- at $B \le \Lambda$ the statistic is a minimum, not a tail mean. Both are necessary, not sufficient: they bound the noise floor of the penalty, not the latent dispersion the penalty is for.
\end{corollary}

The corollary is the sister of Corollary~\ref{cor:lamcap}: that one asks whether a search of size $N$ can afford a tail of depth $\Lambda$, this one whether, having afforded it, anything is left to trade --- both computable from a ledger's declared inputs before any block length is chosen, so the emptiness or otherwise of a real single-asset configuration is a prediction rather than a surprise (Experiment 4; \S\ref{sec:limits-market} reads the verdict). What the two bounds say together is that the way out is \emph{width, not length}: seventy years of one index is not a research programme, and Proposition~\ref{prop:breadth} says what the same $n$ buys spread across instruments --- the reason Experiment 6 is on a cross-section.

\begin{corollary}[The invariance budget]\label{cor:invbudget}
Write $E_N := \delta\hat{\mu} - c' - \sqrt{2\log N/n_{\mathrm{eff}}}$ for the edge stage S5 is left with after the search haircut, in units of the rule's own per-period standard deviation. Since $\mathrm{sd}_{\mathrm{blk}} \ge \rho_{\mathrm{blk}}\sqrt{v(b,\ell)/b}$ in those units, \eqref{eq:nonempty} forces
\[ \vartheta \le E_N/\rho_{\mathrm{blk}} \]
with $\vartheta$ the scale-free protection of Proposition~\ref{prop:scale} --- a statement about the dispersion the deployment window will actually show, which is not known when the position is opened. What is known is the one declared: Proposition~\ref{prop:kelly} charges the penalty at $\mathrm{sd}_{\mathrm{blk}} \ge \kappa\rho\sqrt{v(b,\ell)/b}$, so the same step gives
\[ \vartheta \le E_N/(\kappa\rho), \]
an identity between declared prices: it holds on every path on which S5 opens a position, whatever the block length, the sample size and the declared $\Lambda$. If in addition the realised edge respects the ceiling it was declared under, $\hat{\mu} \le \rho$, then
\[ \vartheta \le \delta/\kappa, \qquad \text{that is} \quad \Lambda_1 \le \lambda^{-1}(\delta/\kappa), \]
and at the conservative declaration $\kappa = 1$ and no decay this is $\Lambda_1 < 2.63$: a procedure that opens a position has, read at state scale, asserted less than that. The invariance ratio, not $\Lambda$, is what decides how much robustness a procedure can both afford and deliver.
\end{corollary}

This is the sharpest form of the arithmetic \S\ref{sec:price} began: Corollary~\ref{cor:lamcap} caps the robustness a budget can \emph{buy}, Corollary~\ref{cor:nonempty} says whether anything survives the purchase, and this corollary says the purchase, if it goes through, was smaller than advertised, by a factor fixed by the declaration before any data are seen. It relocates the declaration that matters to the invariance ratio: $\kappa = 1$ is here the difference between a procedure that can deliver the $\Lambda = 4$ it declares and one that, at its own declared prices, cannot, at any $b$ and any $n$ --- Experiment 4's ledger, a real configuration declared at $\kappa = 1$, is the case in point, and its refusal is Appendix~\ref{app:experiments}'s record.

\section{The canonical form}\label{sec:canonical}

Sections \ref{sec:selection}--\ref{sec:sizing} each forced one stage of a procedure. This section states what they force together. We do not claim that any procedure attaining the bound of Theorem~\ref{thm:deployed} ``factors through five stages'' --- any hypothesis under which that holds is nearly circular. We claim, and prove, that each stage is \emph{necessary}: for each there is a world satisfying the axioms in which a procedure that omits it does strictly, and quantifiably, worse.

\begin{definition}[Procedures and observational equivalence]\label{def:procedure}
A \emph{procedure} is a measurable map $\mathcal{P}$ from the observable past $\{(X_t, Y_t)\}_{t\le n}$, together with auxiliary randomness independent of the data, to a deployed rule $\hat{h}: \mathcal{X} \to \mathbb{R}$ and a fraction $\hat{f} \ge 0$; its \emph{deployed position} at $t > n$ is $a_t = \hat{f}\,\hat{h}(X_t)$, and Axiom~\ref{ax:adapted} is the requirement that $\mathcal{P}$ be measurable in the past alone. Two procedures are \emph{observationally equivalent} if under every law satisfying Axioms \ref{ax:adapted}--\ref{ax:snr} the joint laws of their deployed positions $\{a_t\}_{t > n}$ coincide, and \emph{asymptotically equivalent} if the total-variation distance between those laws tends to zero as $n \to \infty$. A stage is a property of procedures; Theorem~\ref{thm:canonical} says, stage by stage, what omitting it means.
\end{definition}

\begin{theorem}[Canonical form: each stage is necessary]\label{thm:canonical}
Assume Axioms \ref{ax:adapted}--\ref{ax:snr}. For each of the five stages S1--S5 listed below there is a law satisfying the axioms, with the constants shown, under which a procedure omitting the stage does strictly worse than one including it by the following margins, each in the currency Remark~\ref{rem:canonstatus} names, S5's in log growth's units.

\begin{enumerate}
\item (S1: no recurrence-improving representation, i.e.~no finite $\Lambda$ is asserted.) For every estimator, with probability at least $p^\star - \epsilon_n(1)$ the deployed risk is the worst in its class, the observable past coinciding with any law one likes on all periods but one (Theorem~\ref{thm:impossible} with $\Lambda \ge n$); for every estimator with replace-one sensitivity $O(1/n)$ that probability is $p^\star - O(1/n)$.

\item (S2: an unshrunk class of dimension $d > 4\rho^2 n_{\mathrm{fit}}$.) In the model in which the rate of Theorem~\ref{thm:capacity} is attained --- $n_{\mathrm{fit}} = n/H$ independent observations --- no estimator over the class, shrunk or not, has worst-case excess risk below one half of that of $h \equiv 0$; within budget, $d = \varkappa\rho^2 n_{\mathrm{fit}}$, empirical risk minimisation attains $\lesssim \varkappa\rho^2\sigma^2$, and that rate is minimax (Proposition~\ref{prop:minimax}).

\item (S3: evaluation units exchangeable under permutation of time --- random $k$-fold, leave-one-out, the i.i.d.~bootstrap --- at any asserted $\Lambda$; or aggregation by the mean.) The selection statistic is consistent for $\bar{R}_n$ and for nothing else, so the procedure is asymptotically equivalent to the $\Lambda = 1$ procedure of Corollary~\ref{cor:corner} whatever $\Lambda$ it asserts; under a two-state law its realised future loss exceeds its own statistic by $(u + v)(\Lambda - 1) p$, so the guarantee of Theorem~\ref{thm:cvar} fails, while the contiguous-block statistic at $b \asymp \ell$ keeps it (Proposition~\ref{prop:shuffle}).

\item (S4: an unbudgeted search, $\log N_{\mathrm{eff}} \gg \rho^2 n_{\mathrm{eff}}/2$, $N_{\mathrm{eff}}$ the search entropy of the list, which is $N$ for an independent one.) Under the Gaussian model of the ranking statistic, the genuine candidate's probability of winning vanishes, whatever the correlation among the candidates (Theorem~\ref{thm:search} (i), (iii)).

\item (S5: the full position on the point estimate, $f = 1$ --- full Kelly at $\gamma = 1$.) Expected log growth falls short of the optimum by $\frac{1}{2}\,\tau^4/(s^2 + \tau^2) > 0$, and is negative whenever $d > \rho^2 n_{\mathrm{fit}}$ (Proposition~\ref{prop:fullkelly}); at another $\gamma$ both margins are divided by $\gamma$ and neither sign changes.
\end{enumerate}

Conversely the five stages together attain the bound: Theorem~\ref{thm:deployed} for S1--S4 and Proposition~\ref{prop:kelly} for S5.
\end{theorem}

The five stages are these.

\begin{enumerate}
\item \textbf{An adapted, recurrence-improving feature map, with its constants declared.} $x_t = \phi(\mathcal{F}_t^{\mathrm{obs}})$, chosen by domain judgement --- no data can choose it (Corollary~\ref{cor:notestimable}) --- and \emph{declared} together with five numbers: the $\Lambda$ it asserts, the \emph{scale} $b$ at which that assertion is made (Axiom~\ref{ax:recurrence}), the \emph{defect} $\varepsilon_0$ to which the mechanism is invariant within one of its states (Axiom~\ref{ax:invariance}), the coherence times $\ell_1 \le \dots \le \ell_{K_z}$ of the coordinates it keeps in the state, and the invariance ratio $\kappa$ it exposes the candidates to --- a declaration like the first two, and the floor at which Proposition~\ref{prop:kelly} then charges the dispersion. (The signal ceiling $\rho$, the remaining constant of \S\ref{declared-constants}, is declared where the budget spends it, at S2 and S4.) Each is priced where it is spent: the trade among $\Lambda$, $b$ and the coherence times in Remark~\ref{rem:whichell}; the scale in Proposition~\ref{prop:scale}, the same number at a longer block a weaker claim; the trade between $\Lambda$ and $\varepsilon_0$ in Proposition~\ref{prop:representation}, which charges a defect $2M$; $\kappa$ in Corollary~\ref{cor:invbudget}. Three of the five can afterwards be refuted by the sample: $\kappa$ by block-score autocorrelation, $\varepsilon_0$ by the split of Remark~\ref{rem:lamlower}, and the $\ell_i$ by the integrated autocorrelation times of the cells' indicators (Experiment 1 reads them; \S\ref{sec:empirical}). This is what stationarising transformations \emph{are}: fractional differencing, spreads and ratios rather than levels, cross-sectional ranks, volatility normalisation, each mapping a feature whose latent states do not recur into one whose law repeats. Omitting this stage sets $\Lambda = \infty$. Adaptedness is two-sided: free of lookahead \emph{and} of staleness (\S\ref{sec:axioms}). \emph{Forced by: Prop.~\ref{prop:representation}, term (i).}

\item \textbf{A capacity-bounded, shrunk predictor, deployed as an ensemble.} Any hypothesis class is admissible provided $d_{\mathrm{eff}} \lesssim \rho^2 n/H$, the label's overlap and nothing else (Theorem~\ref{thm:capacity}); the natural way to achieve that with a rich class is heavy explicit shrinkage rather than selection, and what is deployed and backtested is an average over training perturbations. \emph{Forced by: Thm.~\ref{thm:capacity}, Prop.~\ref{prop:minimax}, Prop.~\ref{prop:ensemble}.}

\item \textbf{Block-structured, purged and embargoed evaluation, aggregated by $\mathrm{CVaR}_{1/\Lambda}$.} Contiguous blocks at the scale $b \ge \ell^{*}$ declared in S1 --- the selection error of Theorem~\ref{thm:deployed} is computed on the effective sample whatever $b$ is, and the block length pulls two ways (Remark~\ref{rem:blocklen}): lengthening $b$ past what S1 declared silently weakens the assertion by Proposition~\ref{prop:scale}, shortening it below $\ell$ raises the score-noise bias of Proposition~\ref{prop:eiv} (ii); a purge of $H$ periods for the noise (Proposition~\ref{prop:embargo} (i); feature memory is priced, not purged), and for the dependence of the loss itself an embargo of the lag at which the declared predictability falls to $\kappa_{\mathrm{emb}}\rho/\rho_{\mathrm{blk}}$ --- $H + \ell\log(\rho_{\mathrm{blk}}\ell(1 - \rho_\ell^{b})/\kappa_{\mathrm{emb}}\rho b)$ on the reference chain (Proposition~\ref{prop:embargo} (ii)); rank on the tail mean, not the mean. \emph{Forced by: Thm.~\ref{thm:cvar}, Cor.~\ref{cor:blocks}, Prop.~\ref{prop:embargo}, Prop.~\ref{prop:shuffle}.}

\item \textbf{A pre-committed, budgeted search with an explicit deflation.} $2c(\Lambda)^2\log N$ enters the same budget as $d_{\mathrm{eff}}$, and the reported statistic must be deflated by the selection bias \cite{bailey2014deflated,harvey2016cross}. Pre-committing to a small candidate set is a way of buying capacity and, by Corollary~\ref{cor:lamcap}, a precondition for asserting any $\Lambda > 1$ at all. \emph{Forced by: Thm.~\ref{thm:search}, Prop.~\ref{prop:tailprice}, term (ii).}

\item \textbf{Robust sizing at a fraction of the estimate the budget bounds, net of decay and cost.} $f^\star \le 1/(1+\varkappa)$, with equality for an efficient estimator on a pool at the ceiling, applied to $\delta\hat{\mu} - \lambda(\Lambda)\mathrm{sd}_{\mathrm{blk}} - c'$ with $\hat{\mu}$ the \emph{unshrunk} estimate --- the shrinkage of S2 and this fraction are the same operation (Remark~\ref{rem:kappa}), and a predictor S2 has already shrunk to its posterior mean is bet in full; at most $\frac{1}{2}$ at budget saturation, $\varkappa = 1$. Passing S4 does not imply a position: the search condition is that a candidate \emph{at the ceiling} $\rho$ can be told from the null maximum, while the haircut $\sqrt{2\log N/n_{\mathrm{eff}}}$ is taken from the estimate of the candidate actually selected --- a distinction the experiments' ledger realises (Experiment 4; \S\ref{sec:limits-market}). Corollary~\ref{cor:nonempty} states when this cannot happen: a search of more than $\exp(n_{\mathrm{eff}}(\delta\hat{\mu}-c')^2/2)$ proposals leaves nothing to size at any block length. \emph{Forced by: Prop.~\ref{prop:kelly}, Cor.~\ref{cor:halfkelly}, Prop.~\ref{prop:fullkelly}, Def.~\ref{def:reflexivity}.}
\end{enumerate}

Theorem~\ref{thm:canonical} is a \emph{necessity} statement about stages, not a uniqueness statement about implementations: within S2 the choice between ridge, heavily regularised boosting and a shrunk factor model is not determined by the axioms, only the budget is. That residual freedom is where practitioners legitimately differ, and it is much smaller than the space in which they currently do. Experiment 4 walks the five stages once on a real series; Appendix~\ref{app:exp4} is its ledger.

\subsection{The three lemmas that were missing}\label{sec:lemmas}

Theorems~\ref{thm:impossible} and \ref{thm:search} were already necessity statements for S1 and S4. The other three stages needed one each.

\begin{proposition}[The capacity ceiling is minimax]\label{prop:minimax}
Let $N \ge 1$, $d \ge 1$, $\sigma > 0$ and $r > 0$, and consider $N$ independent observations $(X_i, Y_i)$ with $\mathbb{E}[X_iX_i^\top] = I_d$ and $Y_i = \theta^\top X_i + \varepsilon_i$, $\varepsilon_i\sim N(0,\sigma^2)$ independent of $X_i$, the model in which Theorem~\ref{thm:capacity}'s rate is attained with $N = n_{\mathrm{fit}} = n/H$; Axiom~\ref{ax:snr} restricts $\theta$ to the ball $\|\theta\| \le r = \rho\sigma$. Then every estimator $\hat{\theta}$ --- any measurable function of the sample, shrunk or not --- satisfies
\begin{equation}\label{eq:minimax}
\sup_{\|\theta\|\le r} \mathbb{E}_\theta\|\hat{\theta}-\theta\|^2 \;\ge\; \min\left\{\frac{\sigma^2 d}{8N}, \; \frac{r^2}{2}\right\},
\end{equation}
the left side being the excess prediction risk $\mathbb{E}[(\hat{h}(X)-\mu(X))^2]$ of Theorem~\ref{thm:capacity} and $r^2 = \sup_\theta \mathbb{E}[\mu(X)^2]$ the worst-case excess risk of $h\equiv 0$. The second branch is the active one exactly when $d \ge 4\rho^2 N$.

(ii) The same holds on the full sample of overlapping labels. Let $Y_t = \theta^\top x_t + \xi_t$ for $t\le n$, with $\xi_t = H^{-1/2}\sum_{s=1}^H \eta_{t+s}$ the equal-weight overlap of Lemma~\ref{lem:effsample}, $\eta$ i.i.d.~$N(0,\sigma^2)$, and let each feature be an $H$-period moving sum of a $\pm 1$ sequence constant on blocks of length $H$, scaled to unit mean square --- a feature persistent at the label scale. Then every estimator that uses all $n$ labels satisfies
\begin{equation}\label{eq:minimaxfull}
\sup_{\|\theta\|\le r} \mathbb{E}_\theta\|\hat{\theta}-\theta\|^2 \;\ge\; \min\left\{\frac{\sigma^2 d}{24\,n_{\mathrm{fit}}}, \; \frac{r^2}{2}\right\}, \qquad n_{\mathrm{fit}} = n/H,
\end{equation}
the second branch active when $d \ge 12\rho^2 n_{\mathrm{fit}}$; the left side is the excess prediction risk under the design up to the extreme eigenvalues of its second-moment matrix.
\end{proposition}

Below the budget, no estimator improves on the $\sigma^2 d/N$ rate of Theorem~\ref{thm:capacity} by more than a constant; above it, a class of dimension $d \ge 4\rho^2 N$ cannot be learned to better than half the risk of doing nothing, whatever the estimator, and the only escape is a smaller effective class. The value of $N$ is $n/H$, the overlap and nothing else; the block count $n/\ell^{*}$ cannot be it, because Axiom~\ref{ax:snr} caps what the state contributes to a conditional mean (\S\ref{sec:effsample}).

\begin{proposition}[Shuffling destroys the dispersion the tail is computed from]\label{prop:shuffle}
Let the latent chain be the redraw chain of Proposition~\ref{prop:occfloor} with coherence time $\ell$, $\rho = 1 - 1/\ell$, and stationary law $\nu$; let $r: \mathcal{Z} \to [-M, M]$ be a regime risk and $r_t := r(Z_t)$.

\begin{enumerate}
\item (Dispersion.) For contiguous blocks of length $b$, the block-averaged regime risk $b^{-1}\sum_{t\in\mathrm{block}} r_t$ has variance $\mathrm{Var}_\nu(r)\,v(b,\ell)/b$ with $v$ the kernel of \eqref{eq:vcontig}, increasing from $v(1,\ell) = 1$ to $2\ell-1$ as $b/\ell \to \infty$. For $k$ folds drawn uniformly without replacement from $\{1,\dots,n\}$, the fold-averaged regime risk has, conditionally on the path, variance $s_n^2(r)(k-1)/(n-1)$, where $s_n^2(r)$ is the empirical variance of $r_t$ along the path, with $\mathbb{E}s_n^2(r) = \mathrm{Var}_\nu(r)(1 - v(n,\ell)/n)$. With $k = B = n/b$ units in each case, the ratio of the two variances is $v(b,\ell)(1+O(b/n))$: about $0.7\,\ell$ at $b=\ell$, and $2\ell-1$ for $b\gg\ell$.

\item (Collapse.) Let fold scores be $\hat{r}_F(h) = |F|^{-1}\sum_{t\in F}(r_h(Z_t) + \xi_t(h))$ with the $\xi_t$ conditionally independent, mean zero and bounded by $M$; let $\mathcal{H}$ be finite, $\Lambda \ge 1$, and $k = k_n$ with $k\log k/n \to 0$. Then, conditionally on the path, $\max_{h\in\mathcal{H}} |\hat{\mathrm{CVaR}}_{1/\Lambda}^{\mathrm{shuf}}(h) - \bar{R}_n(h)| \to 0$ in probability, $\hat{\mathrm{CVaR}}^{\mathrm{shuf}}$ being the mean of the worst $\lceil k/\Lambda\rceil$ fold scores. Hence, when $\bar{R}_n$ has a unique minimiser, the procedure that selects by the shuffled-fold $\hat{\mathrm{CVaR}}_{1/\Lambda}$ is asymptotically equivalent to the one that selects by the mean --- the $\Lambda=1$ procedure of Corollary~\ref{cor:corner} --- whatever $\Lambda$ it asserts.

\item (Failure.) Let $\nu = (1-p,p)$ on two states with $\Lambda p \le 1$, and let $\mathcal{H} = \{A,B\}$ with $r_A \equiv a$ and $r_B = a - u$ on the common state, $a+v$ on the rare one (losses), where $u(1-p) > vp$ and $u(1-\Lambda p) < v\Lambda p$: the mean prefers $B$, the $\Lambda$-tilted future prefers $A$. Then with probability tending to one the shuffled procedure selects $B$, its statistic converges to $\bar{R}_n(B) = a - u(1-p) + vp$, and under the worst admissible future of Theorem~\ref{thm:cvar} its deployed loss is $R^+(B) = a - u(1-\Lambda p) + v\Lambda p$: the guarantee $R^+(\hat{h}) \le \hat{\mathrm{CVaR}}_{1/\Lambda}(\hat{h})$ fails by $(u+v)(\Lambda-1)p$, and the selected rule is worse than the class optimum $A$ by $v\Lambda p - u(1-\Lambda p) > 0$. The contiguous-block statistic, in the limit $B\to\infty$ with $b$ fixed, selects $A$ whenever
\begin{equation}\label{eq:blockwins}
(u+v)\,\Lambda p\,(1-1/\ell)^{b-1} - u \;>\; \sqrt{\Lambda}\,\sigma_\xi/\sqrt{b},
\end{equation}
$\sigma_\xi^2$ the per-period noise variance --- satisfied at $b=\ell$ for, e.g., $p=0.1$, $\Lambda=4$, $u=0.6$, $v=5$, $\ell=40$ and $\sigma_\xi \le 0.7$. The condition is sufficient only: it credits the tail with blocks lying entirely inside the rare state. In the Gaussian range $b\gg\ell$ the tail penalty of $B$ over $A$ is $(u+v)\lambda(\Lambda)\sqrt{p(1-p)v(b,\ell)/b}$ against the noise $\sqrt{\Lambda}\sigma_\xi/\sqrt{b}$ --- both of order $b^{-1/2}$ --- and the block statistic keeps selecting $A$ as long as that penalty exceeds the $b$-free mean advantage $u-(u+v)p$ of $B$: at $b=300$ it is $1.02$ against $0.04$ ($\lambda(4)=1.27$, $v(300,40)=68.6$), and the collapse to the mean is reached at these numbers only near $b\approx 2\times 10^5$.
\end{enumerate}
\end{proposition}

Part (i) is the arithmetic behind Corollary~\ref{cor:blocks}'s phrase ``a distribution whose dispersion has been destroyed'', (ii) is why it matters --- a tail computed on units that all look alike is a mean, so random $k$-fold at any asserted $\Lambda$ is the $\Lambda=1$ procedure in disguise --- and (iii) prices it on the simplest law where the rankings disagree.

\begin{proposition}[Full Kelly is never optimal, and is negative over budget]\label{prop:fullkelly}
In the model of Proposition~\ref{prop:kelly}, within its quadratic approximation:

\begin{enumerate}
\item for every fraction $f$, $g(f^\star) - g(f) = \frac{1}{2}(s^2+\tau^2)(f-f^\star)^2$; in particular $g(1) = \frac{1}{2}(s^2-\tau^2)$ and $g(f^\star)-g(1) = \frac{1}{2}\,\tau^4/(s^2+\tau^2) > 0$ whenever $\tau > 0$;

\item Axiom~\ref{ax:snr} applied to every candidate in the pool gives $s^2 \le \rho^2$, and the Cramér--Rao bound gives $\tau^2 \ge d/n_{\mathrm{fit}}$ for the unbiased estimate the fraction multiplies, in the model of Proposition~\ref{prop:minimax} (this is the floor of Proposition~\ref{prop:ensemble}). Hence $\tau^2/s^2 \ge \varkappa = d/(\rho^2 n_{\mathrm{fit}})$, $f^\star \le 1/(1+\varkappa)$, and $g(1) \le \frac{1}{2}s^2(1-\varkappa) < 0$ whenever $d > \rho^2 n_{\mathrm{fit}}$: a researcher who has overspent the budget and bets full Kelly on her estimate has negative expected log growth, for every prior on the pool and every estimator.
\end{enumerate}
\end{proposition}

\begin{remark}[What Theorem~\ref{thm:canonical} does and does not say]\label{rem:canonstatus}
Three boundaries. \emph{Necessity, not factorisation.} That a good procedure ``factors through'' five stages would be a uniqueness claim, and we have neither a proof of it nor a use for it; Definition~\ref{def:procedure}'s equivalence is used only to say that two implementations of a stage are interchangeable. \emph{Five margins, four currencies.} S1 and S3 are about the deployed risk $R^+$; S2 is about excess prediction risk in squared loss for a well-specified linear class on $n_{\mathrm{fit}}$ independent observations, the model Theorem~\ref{thm:capacity} is attained in; S4 retains the Gaussian model of Theorem~\ref{thm:search}; S5 is about any objective with the quadratic approximation of Proposition~\ref{prop:kelly}, in log growth's units (\S\ref{sec:setting}). \emph{Stages, not their implementations.} S3 is proved against evaluation units exchangeable under permutation of time; walk-forward and combinatorial purged cross-validation are implementations of S3, not omissions of it, and an implementation with $b\gg\ell$ is one whose tail has been softened by the factor $\sqrt{v(b,\ell)/b}$ of Proposition~\ref{prop:shuffle} (i) --- a milder statistic of the same kind, not the mean in disguise. \emph{And one clause per stage.} Each margin above is proved for one clause of its stage's description, and the other clauses are derived, not shown necessary: S2's margin is the capacity bound, while heavy shrinkage rather than selection and the deployed ensemble are how the bound is met without waste (Proposition~\ref{prop:ensemble} says the ensemble spends nothing, not that a single fit fails); S3's is contiguity with tail aggregation, the purge and the embargo being what Proposition~\ref{prop:embargo} prices once blocks are contiguous; S5's is the fraction, the dispersion penalty and the decay-and-cost shift being what Proposition~\ref{prop:kelly} charges the estimate the fraction multiplies. A procedure that keeps a stage's proved clause and drops a derived one has a bound that is looser, not a margin that is proved.
\end{remark}

Proposition~\ref{prop:minimax} is stage S2's lower bound; the selection stages have one too. Corollary~\ref{cor:lamcap} charges a search run on the $1/\Lambda$ tail the factor $c(\Lambda)^2$, and that charge could in principle have been the plug-in's clumsiness. It is not: on a law the axioms admit, resolving a gap in the tail costs a factor $\Lambda - 1$ more data than resolving the same gap in the mean, whatever the selection rule --- the latent path handed to it included.

\begin{proposition}[The tail's charge is necessary: a lower bound for the selection stages]\label{prop:tailfloor}
Fix $\Lambda \ge 2$, $\rho\in(0,1)$, $\sigma>0$, $N\ge 5$ and $0<\delta\le\rho\sigma/2$. There is a family $P_1,\dots,P_N$ of laws for $\{(X_t,Y_t,Z_t)\}$ and $N$ fixed candidate rules of common position scale, each $P_i$ satisfying Axioms \ref{ax:adapted}, \ref{ax:invariance}, \ref{ax:persistence} and \ref{ax:snr} with the same constants --- a two-state latent chain with stress mass $1/\Lambda$ redrawn independently each period ($\ell=1$), invariance defect zero, conditional noise $\sigma^2$ in every state, $\chi_\sigma^2 \le 1+\rho^2$, the feature law common to the whole family --- and Axiom~\ref{ax:recurrence} declared at $\Lambda$, the worst case below running over the futures it admits, such that under each $P_i$:

\begin{enumerate}
\item every candidate's stationary mean PnL is zero: there is no gap in means for a mean ranking to find;

\item candidate $i$'s worst-case deployed risk --- the $\mathrm{CVaR}_{1/\Lambda}$ of Theorem~\ref{thm:cvar} at $b=1$ --- is $0$ and every other candidate's is $\delta$, the trapped candidates' stress atom filling the level-$1/\Lambda$ tail exactly;

\item any selection rule $\hat{T}$ --- adaptive, randomised, and permitted to observe the latent path --- that returns the tail-optimal candidate with probability at least $1/2$ under every $P_i$ requires
\begin{equation}\label{eq:tailfloor}
n \;\ge\; \frac{(\Lambda-1)\,\sigma^2\,\log(N/4)}{2\,\delta^2},
\end{equation}
by Fano's inequality on the family;

\item carried by the mean instead --- the family altered only in that candidate $i$'s edge $\delta$ is the same in both states --- the same gap is resolved by the backtest-mean ranking once $n \ge C\sigma^2\log N/\delta^2$, with an absolute constant $C$, and no rule resolves it at a smaller order: the floor is \eqref{eq:tailfloor} without the factor $\Lambda-1$;

\item a rule that observes the state attains \eqref{eq:tailfloor} up to $\Lambda/(\Lambda-1)$ and an absolute constant, by comparing the candidates' stress-state means --- Theorem~\ref{thm:cvar}'s criterion read on the true state at $b = 1$: $n \le C\,\Lambda\,\sigma^2\log N/\delta^2$ suffices.
\end{enumerate}
\end{proposition}

The bound survives handing the rule the latent path, so the factor is not the price of latency but the scarcity Proposition~\ref{prop:occfloor} counts --- $n/\Lambda$ periods carry the whole gap; the cost of \emph{not} seeing the state, reading the tail through noisy blocks, is charged on top (\S\ref{sec:eiv}), and only the first layer is irreducible. The floor puts $\Lambda-1$ beneath every rule --- the order of the dear branch of \S\ref{sec:price}'s dichotomy, the one the experiments find real libraries on --- and item (i) is Theorem~\ref{thm:cvar}'s dichotomy made quantitative: a researcher who declares $\Lambda>1$ and ranks by the mean has not taken a cheaper route to the same place; there is no route. The $1/\alpha$ factor in the sample complexity of a $\mathrm{CVaR}_\alpha$ selection is known for independent arms \cite{agrawal2021tailbai}; what the proposition adds is that the family realising it lives inside the axioms, with the mean-blind clause (i) the form Theorem~\ref{thm:cvar} needs.

\subsection{What the accumulated practice looks like from here}
\label{sec:canon}
The canonical form is a claim about what a correct procedure must contain, and the field has been building procedures for decades without it. Read the other way, the five stages say where each accumulated practice comes from --- and that the derivation lands on practices arrived at independently, by experience, is the main evidence that the axiomatisation captures the right structure.

\begin{table}[htbp]
\centering
\begin{tabularx}{\linewidth}{@{}>{\raggedright\arraybackslash}p{0.36\linewidth}>{\raggedright\arraybackslash}X@{}}
\toprule
\textbf{Practice} & \textbf{Derived from} \\
\midrule
Stationarising transforms; fractional differentiation \cite{deprado2018advances} & S1: minimise $\Lambda(\phi)$ subject to invariance defect (Prop.~\ref{prop:representation}) \\
Sample weights by average uniqueness \cite{deprado2018advances} & Lemma~\ref{lem:effsample}: overlapping labels enter $n_{\mathrm{eff}}$ at full weight $H$ --- the overlap, not the state's persistence, is what the weights correct \\
Purged $k$-fold; embargo \cite{deprado2018advances} & Prop.~\ref{prop:embargo}: purge term $H$; embargo from the declared predictability --- $\ell\log(\cdot)$ on the reference chain \\
Combinatorial purged CV; backtest paths & Cor.~\ref{cor:blocks}: blocks estimate the $\bar{\pi}_n$-law of $r_h$ \\
Deflated Sharpe ratio \cite{bailey2014deflated}; PBO \cite{bailey2017pbo}; the $t > 3$ hurdle \cite{harvey2016cross} & Thm.~\ref{thm:search}: the haircut $\sqrt{2\log N/n_{\mathrm{eff}}}$, on the effective sample of Lemma~\ref{lem:effsample} \\
``Backtesting is not a research tool'' & Thm.~\ref{thm:search} with $n_{\mathrm{eff}}$ small \\
Separating side from size (meta-labelling) & Budget allocation: spend $d_{\mathrm{eff}}$ on the sign, size by Prop.~\ref{prop:kelly} \\
Fractional Kelly \cite{maclean2011kelly} & Prop.~\ref{prop:kelly}, Cor.~\ref{cor:halfkelly} \\
Feature importance measured on blocks & Cor.~\ref{cor:blocks} applied per feature \\
Structural break tests & Detection of a change in $\bar{\pi}_t$; a diagnostic for $\Lambda$ \\
\bottomrule
\end{tabularx}
\caption{Practices the field arrived at independently, by experience, and the result here that implies them. We regard this as the main evidence that the axiomatisation captures the right structure.}
\end{table}

\section{Computing the budget}
\label{sec:calibration}
Sections~\ref{sec:selection}--\ref{sec:canonical} derived the objects, and every one of them has a constant in it. This section computes the budget on a worked example, says where real block scores put the tail's price, and reads off what a budget of that size can then buy.

The budget can be computed mostly from configuration, before any modelling. Four quantities are read off the design: the \textbf{label horizon} $H$, and whether the protocol steps by less than $H$; the \textbf{nominal feature count}; the \textbf{shrinkage}, converted to $d_{\mathrm{eff}}$ ($\sum_i \nu_i/(\nu_i + \lambda)$ for ridge; for boosted trees the trace of the smoother, $\sigma^{-2}\sum_i \mathrm{Cov}(\hat{y}_i,y_i)$ over label perturbations); and the \textbf{search size} $N$, the number of proposals evaluated over the system's life, which nobody records and everybody should --- for a correlated library its search entropy $\log N_{\mathrm{eff}}$ in place of $\log N$ (Theorem~\ref{thm:search} (iii)), computable from the candidates' training-window correlations. Three must be asserted: $\ell$, $\rho$ and the regime dispersion $\rho_{\mathrm{blk}}$, for which $\rho$ is the conservative choice; a fourth, the predictability term $\pi_\mu$ of Lemma~\ref{lem:effsample}, every worksheet here sets to zero, which is the anti-conservative end of Remark~\ref{rem:upsbox}'s box and is said there. A menu of candidate representations declared at S1 enters the same account, at $2\log|\Psi|$ nats beside $2\log N$ (\S\ref{search}): a handful of coordinates costs a couple of nats, and the point is that they are on the bill.

A worked example at monthly frequency: thirty years is $n \approx 360$, monthly labels make $H = 1$, a macro coherence time of four months gives $2\ell - 1 = 7$, a target annual information ratio of $1.0$ is $\rho^2 \approx 0.083$ per month, so with $\rho_{\mathrm{blk}} = \rho$, $\upsilon \approx 1.6$, $n_{\mathrm{eff}} \approx 230$ and the timing budget is $19$ nats (perhaps $45$--$60$ cross-sectionally for half a dozen correlated macro assets). Against that, a twenty-column feature table under heavy shrinkage carries $d_{\mathrm{eff}} \approx 5$--$15$, and a few thousand evaluated proposals cost $2\log N \approx 16$ on the mean and $3.6 \times 16 \approx 58$ on the quarter-tail: the feature table is affordable, the search is affordable on the mean and not once it is made robust. Experiment 4 keeps this worksheet on a real daily series, stage by stage.

Two sensitivities the worksheet should carry. The budget is quadratic in the declared $\rho$: a researcher who declared $\rho$ at her library's best training-era Sharpe of $0.040$ would have $7$ nats --- two or three candidates at $\Lambda = 4$, not Experiment 4's twenty-seven; the ceiling is a belief about the signal the representation admits, not a reading of the library's best member, and Theorem~\ref{thm:search} says in which direction the in-sample estimate errs. And a declared $\rho_{\mathrm{blk}}$ predicts the autocorrelation of consecutive block scores, which is how Experiment 1 rejects $\rho_{\mathrm{blk}} = \rho$ on its library (\S\ref{sec:selection}); declared and refuted should be said, and the ledger reported at both values.

\textbf{Where real block scores sit: on the dear side --- and why.} The dichotomy says where $c(\Lambda)$ \emph{can} be; only data say where it \emph{is}. Experiment 1 measures it on seventy-one daily timing rules over thirty years of one equity index ($b = 60$, $B = 124$) by the plug-in \eqref{eq:plugin}: the median is $c(\Lambda)^2 \approx 0.9\,\Lambda$, halfway between the two branches on a logarithmic scale. Experiment 2 finds the same on nineteen of twenty series and traces the whole excess to the heteroskedasticity of the score noise, the bias of Proposition~\ref{prop:eiv} (iii): studentised by its own within-block noise, the plug-in sits at the Gaussian value on every asset. Experiment 6 takes the measurement off that substrate: eighty-two long-only cross-sectional portfolios on annual blocks of monthly returns reproduce $c(\Lambda)^2 \approx 0.9\Lambda$, while $133$ published long-short signals over the same window sit at the Gaussian value. So the dear branch is a property of block scores that \emph{carry the market}, not of daily data and not of our own rules, and which branch the reader is on is a question about their own library.

\textbf{What the budget can then buy.} Corollaries~\ref{cor:budgetrobust} and \ref{cor:lamcap} are an arithmetic once $c$ is fixed, and the table solves the second of them for $\Lambda$, under the Gaussian law and under the measured one above.

\begin{table}[htbp]
\centering\small
\begin{tabular}{rcccc}
\toprule
 & \multicolumn{4}{c}{largest affordable $\Lambda$} \\
\midrule
budget $\rho^2 n_{\mathrm{eff}}$ & $N = 2$ & $N = 10$ & $N = 100$ & $N = 1000$ \\
\multicolumn{5}{@{}p{0.9\linewidth}@{}}{\emph{Gaussian block risks, $c(\Lambda)$ from the closed form}} \\
2 & 2.2 & -- & -- & -- \\
7.5 & 17.0 & 2.8 & -- & -- \\
12.6 & 34.2 & 6.4 & 2.0 & -- \\
19 & 58.3 & 11.7 & 4.1 & 2.0 \\
23 & 74.2 & 15.3 & 5.5 & 2.9 \\
50 & 194 & 43.2 & 17.1 & 9.7 \\
250 & $> 10^3$ & 317 & 137 & 82.8 \\
\multicolumn{5}{@{}p{0.9\linewidth}@{}}{\emph{measured block risks, $c(\Lambda)^2 = \max(1, 0.9\Lambda)$ (Experiment 1)}} \\
2 & 1.6 & -- & -- & -- \\
7.5 & 6.0 & 1.8 & -- & -- \\
12.6 & 10.1 & 3.0 & 1.5 & -- \\
19 & 15.2 & 4.6 & 2.3 & 1.5 \\
23 & 18.4 & 5.5 & 2.8 & 1.8 \\
50 & 40.1 & 12.1 & 6.0 & 4.0 \\
250 & 200 & 60.3 & 30.2 & 20.1 \\
\bottomrule
\end{tabular}
\caption{Corollary~\ref{cor:lamcap}, solved for $\Lambda$ under two block-risk laws. A dash means no $\Lambda \ge 1$ is affordable: the search term alone exceeds the budget. The records' dependence is inside $n_{\mathrm{eff}}$ (Corollary~\ref{cor:blocks} (a)) and charged nowhere else. The single-asset daily setting of \S\ref{sec:budget} (budget $23$) can assert $\Lambda = 4$ under the measured law only with two dozen candidates, the monthly macro setting of \S\ref{sec:calibration} (budget $19$) with fourteen, and monthly labels on daily bars (budget $2.3$) can afford a two-way comparison only up to $\Lambda = 1.8$. The measured law buys less $\Lambda$ than the Gaussian for the same budget, by a factor growing from $1.3$ to $6.5$ across the table, because the Gaussian $c$ grows like $\Lambda^{1/4}$ and the measured one like $\Lambda^{1/2}$. Under the measured law the largest affordable $\Lambda$ is linear in the budget, $\Lambda = \rho^2 n_{\mathrm{eff}}/(1.8\log N)$; under the Gaussian law it grows faster than linearly.}
\end{table}

\begin{remark}[{Which $\ell$: the representation sets it, together with $\Lambda$ and $\rho_{\mathrm{blk}}$}]
\label{rem:whichell}
Axiom~\ref{ax:persistence} declares a coherence time per coordinate, so $\ell$ is that of the slowest coordinate the modeller puts into $Z$ --- and the paper's example states (a credit freeze, a central bank buying corporate bonds) are slow ones. Monthly readings of eight candidate coordinates over thirty years (Experiment 1, panel E) give AR(1) coherence times of $40$ trading days for realised volatility, $99$ for $\log$ VIX, about $150$ and $220$ for the MOVE index and breakeven inflation, and of order a thousand for the CPI, the term spread and the policy rate. The headline $\ell = 60$ is declared of the volatility coordinate, a margin above its measured $40$. Lemma~\ref{lem:effsample} charges a slow coordinate only to what depends on it, so the question has a table rather than a number (daily labels, $n = 5000$, $\rho^2 = 0.01$; the budget is $50/\upsilon$ nats):

\begin{center}
\begin{tabular}{lrrrr}
\toprule
\begin{tabular}[c]{@{}l@{}}$\upsilon$ at $\ell$ \\ $\rho_{\mathrm{blk}}$\end{tabular} & $0$ & $0.03$ & $0.1$ & $0.3$ \\
\midrule
$60$ (volatility) & $1$ & $1.1$ & $2.2$ & $11.7$ \\
$250$ (breakevens) & $1$ & $1.4$ & $6.0$ & $46$ \\
$1300$ (the rate cycle) & $1$ & $3.3$ & $27$ & $235$ \\
\bottomrule
\end{tabular}
\end{center}

A rule whose edge does not move with the rate cycle pays nothing for its persistence, however slow; one whose edge depends on the \emph{conjunction} of two coordinates is charged near the fast coordinate's rate, because the coincidence ends when either regime does (Lemma~\ref{lem:effsample}). So $\ell$ is part of the choice of representation, which moves three constants at once: dropping a coordinate converts what the edge owed it into an invariance defect charged at $2M\varepsilon(\psi)$; keeping it charges $\rho_{\mathrm{blk}}^2(2\ell - 1)$. The conservative $\kappa = 1$ is priced twice --- a factor of two in $\upsilon$ here, and Corollary~\ref{cor:invbudget}'s cap $\Lambda_1 < 2.63$ on any position S5 opens --- and it moves verdicts: on Experiment 4's series, substituting the regime dependence the library measures for the declared $\rho_{\mathrm{blk}} = \rho$ moves Corollary~\ref{cor:nonempty}'s search ceiling from $4.4$ candidates to $19$, and the Bartlett ratio the same series measures takes it to $60$. It remains the right default, since $\rho_{\mathrm{blk}}$ is estimated on the sample the search will then bend; a refusal produced under it is partly the declaration's.
\end{remark}

\begin{remark}[This is a computation, not a validation]
\label{rem:computation}
That mature research pipelines, built by trial and error, tend to land near their computed budget is not evidence on its own: the constants are wide, $d_{\mathrm{eff}}$ for a boosted model is an estimate, and $N$ is rarely recorded honestly. That competent processes converge to budget saturation is a \emph{prediction} (P5, \S\ref{sec:empirical}), to be tested across many systems with pre-registered inputs.
\end{remark}

\begin{remark}[{The box the axioms put $\upsilon$ in}]
\label{rem:upsbox}
Both ends of the $\upsilon$ of \eqref{eq:neff} are fixed before any data are seen, and both are Axiom~\ref{ax:snr}'s ceiling multiplied by a coherence time. \emph{The regime term}: $\rho_{\mathrm{blk}} \le \rho$; with $\rho = \mathrm{SR}_{\mathrm{ann}}/\sqrt{P}$ for $P$ periods a year and $\ell = \tau_{\mathrm{yr}} P$, the regime term is $\rho^2(2\ell - 1) = 2\,\mathrm{SR}_{\mathrm{ann}}^2\tau_{\mathrm{yr}} - \rho^2$, free of the sampling frequency. \emph{The predictability term} of Lemma~\ref{lem:effsample} obeys $|\pi_\mu| \le \rho\,\tau_\mu$ in either sign, $\tau_\mu$ the coherence of the loss's own predictable component. So
\begin{equation}
H - 2\rho\,\tau_\mu \;\le\; \upsilon \;<\; H + 2\,\mathrm{SR}_{\mathrm{ann}}^2\,\tau_{\mathrm{yr}} + 2\rho\,\tau_\mu,
\label{eq:upsbox}
\end{equation}
with $\mathrm{SR}_{\mathrm{ann}}$ the \emph{declared} annual Sharpe ceiling. Nothing else can move $\upsilon$: the regime part is the variance of an average and cannot be negative, whatever the chain does and whichever coordinates the modeller declares, so \emph{the only road below $H$ is $\pi_\mu$, and it is Axiom~\ref{ax:snr}'s road} --- and the same term at the other sign, a loss whose own past predicts its persistence rather than its reversal, is the only road above the regime's ceiling. The $\rho^2 = 10^{-2}$ of this worksheet, at a ceiling of $1.6$ and $\tau_\mu = 1$, puts a daily rule in $[0.80, 2.40]$.

Three consequences. The $\upsilon$ factor of P2 is not identified on any admissible series --- its whole range at daily frequency is a factor of $3.0$ at this ceiling, and $1.8$ at a ceiling of one annual Sharpe unit, against the ten- to sixtyfold spread a slope would need --- so what is falsifiable is the box itself (\S\ref{sec:empirical}). Above the box is dependence the axioms cannot produce: Experiment 3's high-yield fund at $\hat{\upsilon}(252) = 3.3$ is the stale-NAV smoothing of a matrix-priced fund, a persistence that is not the regime's and must not be spent as effective sample. Below $H$ is the predictability term, and it is inside the box: the median across Experiment 3's twenty series is $\hat{\upsilon}(252) = 0.86$, which at the declared $\rho = 0.1$ asks for a predictable component of $\tau_\mu = 0.72$ periods, and what that component \emph{is} can be read directly --- the first-order autocorrelation of the daily excess return, $-0.082$ on the S\&P 500 and negative on sixteen of twenty series. Every worksheet computes $\upsilon$ at a declared $\pi_\mu = 0$, which is \emph{anti}-conservative; an account meaning to stay conservative should charge $\max(H, \hat{\upsilon})$, and Experiment 4's measured ratio sits exactly at the box's lower end, so its verdict does not move.
\end{remark}

One object the body uses is defined here. Writing $\hat{r}_j$ for the block scores of a candidate, $\hat{q}$ for their empirical $1/\Lambda$-quantile in the reward orientation of Remark~\ref{rem:orientation}, the \textbf{plug-in} for the tail price of Proposition~\ref{prop:tailprice} is
\begin{equation}
\hat{c}^2 = \Lambda^2\hat{\mathrm{Var}}(\min(\hat{r},\hat{q})) / \hat{\mathrm{Var}}(\hat{r}),
\label{eq:plugin}
\end{equation}
which costs nothing beyond the block scores already computed and, by Proposition~\ref{prop:eiv}, estimates the price paid for the tail of block \emph{scores}, not the tail of the latent risk law.

\section{Falsifiability and an empirical programme}
\label{sec:empirical}
An axiom system that cannot fail is not worth having. This section says two things: which of the premises a sample can refute, and how this paper tests them. The logic first, because an axiomatic paper is answerable to data differently than an empirical one. If the derivations are sound, the conclusions inherit their truth from the premises, so what a sample can interrogate is the premises; and the premises here are not propositions but \emph{declarations}: read existentially the invariance axiom is satisfied by the date itself (Remark~\ref{rem:vacuous}) and the recurrence axiom at $\Lambda = \infty$ asserts nothing, so each axiom acquires empirical content only at the constants a researcher declares, and testing a declaration is not a detour around testing the axioms --- it is the only form in which an axiom can be tested at all.

The experiments accordingly do three jobs, and keeping them apart is load-bearing. On real series they test \emph{axiom instances} at their declared constants, which is the axioms in their falsifiable form. They test the \emph{premises no axiom contains} --- the invariance ratio $\kappa$, the decay pair of Definition~\ref{def:reflexivity} --- on the same footing, because the conclusions are functions of every declared constant, not only of those called axioms. And they test \emph{derived predictions}, where a refutation strikes, by \emph{modus tollens}, the conjunction of the axioms, the declarations and the approximations the derivations import beyond them --- the only empirical route to the declarations that admit no direct test, $\Lambda$ first among them (Corollary~\ref{cor:notestimable}).

One rule spans all three: nothing here confirms. Refutation needs only the one-way implication from premises to prediction; confirmation would need its converse, which is never available, so every verdict below is ``refuted'' or ``not refuted where the test has power'', and a prediction that survives leaves the axioms exactly as assumed as they were.

\textbf{What a sample can say about a declaration.} The answer is not uniform, and setting it out is the honest form of the claim that the system is falsifiable at all.

\begin{center}
\footnotesize
\setlength{\tabcolsep}{3pt}
\begin{longtable}{@{}>{\raggedright\arraybackslash}p{0.13\linewidth}>{\raggedright\arraybackslash}p{0.15\linewidth}>{\raggedright\arraybackslash}p{0.29\linewidth}>{\raggedright\arraybackslash}p{0.39\linewidth}@{}}
\toprule
\textbf{Declared} & \textbf{Refutable?} & \textbf{By what} & \textbf{Status here} \\
\midrule
\ref{ax:adapted} & No & --- & A property of the data pipeline, not of the series: a stale feature (\S\ref{sec:axioms}) is present, finite and wrong, and passes every coverage check. The audit is provenance, and none is recorded here. \\
\ref{ax:invariance}, $\varepsilon_0$ & \textbf{Yes, outright} & split-sample distance between each declared cell's law of $(X,Y)$ (Rem.~\ref{rem:lamlower}) & \textbf{Not contradicted at the test's own level}: on three declared coordinates of one series and across twenty, the rejections at that level are the false-positive count of the tests run plus the two matrix-priced bond funds, whose staleness is priced elsewhere (Rem.~\ref{rem:lamlower}; Exp.~1, panel H4; Exp.~2, panel D). \\
\ref{ax:recurrence}, $\Lambda$ at $b$ & No estimate of the ratio; a lower bound, and a contradiction of one's own declaration & Prop.~\ref{prop:occfloor}'s occupation floor; $\hat{\Lambda}_{\mathrm{split}}$ against a shuffled-block null (Rem.~\ref{rem:lamlower}) & Experiment 4's declaration, read at state scale, clears the floor ($2.99$ against $2.24$) and every edge it prices is gone at the self-consistent one ($1.17$); Experiment 1's split statistic reads the floor on performance bins and \emph{exceeds} it on named coordinates (Rem.~\ref{rem:lamlower}). \\
\ref{ax:persistence}, $\ell_i$ & Yes, once the representation is declared & per-coordinate stationary block bootstrap \cite{politis1994stationary} or regime-switching fit \cite{hamilton1989new}, once the coordinates are named (Exp.~1, panel E); the integrated autocorrelation time of the cells' indicators against the declared $2\ell_i - 1$ (Exp.~1, panel J); and, through Lemma~\ref{lem:effsample}, the block scores' own coherence across blocks against the $\upsilon/\upsilon_b$ the declaration implies (Cor.~\ref{cor:blocks} (a); Exp.~1, panel B) & Eight coordinates measured (Remark~\ref{rem:whichell}); on the three declared ones the indicators' times are of the order declared, the calm cell the longer one, the rank of the coordinate up to $3.5$ times the AR(1) number. The axiom is refined, not contradicted; the declaration Experiment 4 makes of that coordinate, $\ell = 60$, \textbf{is contradicted} by its rank ($276$ against the $119$ its $2\ell - 1$ asserts), and the ledger at the refined $\ell = 138$ is a row of that experiment's sensitivity table, moving its verdict further the same way; \emph{which} coordinates belong in the state is not a question a sample answers (Rem.~\ref{rem:whichell}). The scores' coherence across blocks reads $1.10$ at $b = 60$ and the tail statistic's $1.05$, inside their null and below the $1.52$ that $\ell = 60$ with $\rho_{\mathrm{blk}} = \rho$ implies --- the reading of $\kappa$ the next row makes, taken at the block scale. The scores' magnitudes read $1.72$ --- the score noise's scale is coherent across blocks, a dependence that is not a Gaussian pair's, so the tail statistic's count is read on the sample rather than derived from the scores' (\S\ref{sec:limits}). The embargo rests on the declared predictability, which panel J bounds from below: Experiment 4's $523$ days not contradicted, the AR(1) chain's $112$ contradicted on every coordinate. \\
\ref{ax:snr}, $\rho$ & Partly: an honest out-of-sample record, and only the historical end $\rho_1$ & a track record, biased upward by Thm.~\ref{thm:search} & Declared, not tested against; nothing relates the historical end $\rho_1$ to the state-wise $\rho$. \\
\midrule
$\kappa$ --- declared, not an axiom & \textbf{Yes} & block-score autocorrelation (Prop.~\ref{prop:embargo} (ii)) & \textbf{Contradicted.} The conservative $\kappa = 1$ is rejected on Experiment 1's library, whose measured block-score dependence implies $\kappa_h \approx 0.3$ (Appendix~\ref{app:experiments}). Capped above by Prop.~\ref{prop:kapceil}. \\
$\chi_\sigma$ & Yes & within-block dispersion against the block score (Exp.~2, panel A) & Its existence, not its value: the whole excess of $\hat{c}$ over its Gaussian value is heteroskedastic score noise, gone once each block is studentised by its own. \\
\bottomrule
\caption{What a sample can say about each declaration. The rows above the rule are the axioms; the two below are declarations the axioms do not contain.}
\end{longtable}
\end{center}

The record in one paragraph, the table's rows carrying the detail. \textbf{The one test that can contradict an axiom outright, run at its own level, contradicts none} (Remark~\ref{rem:lamlower}), and it keeps its power: a defect above $0.04$--$0.09$ in total variation would still be seen at this sample. \textbf{The declaration a sample does contradict is the invariance ratio}, not an axiom but a stage-S1 number (Proposition~\ref{prop:kapceil}): what a sample reaches are the constants, not the structure. \textbf{And the axiom the rest of the paper rests on is the one furthest out of reach}: $\Lambda$ admits no estimate at all (Corollary~\ref{cor:notestimable}); its two instruments bound it from below and contradict declarations made beneath them.

\textbf{How the predictions are tested.} Five predictions of the assembled system, each stated in the form that would refute it, with a one-line verdict; the designs, the effect each test could have resolved, and the evidence behind each verdict are in Appendix~\ref{app:predictions}.

\begin{enumerate}
\item \textbf{CVaR selection dominates mean selection out of sample, by a margin growing with measured cross-block dispersion.} \emph{Refuted if} the two perform indistinguishably after controlling for the mean, or if the advantage does not increase in $\mathrm{sd}_{\mathrm{blk}}$. \emph{Status:} a null with power behind it rather than in front of it, at the one level whose tail is never thin, and a loss at the narrower levels for a reason that is the block count and not the criterion; the second clause fails on the one library where there is an advantage to grow (Experiment 6).

\item \textbf{Out-of-sample degradation scales as $\sqrt{d_{\mathrm{eff}} H/n}$ for the fit and $\sqrt{2\log N\,\upsilon/n}$ for the search, not as $\sqrt{d_{\mathrm{eff}}\ell^{*}/n}$.} \emph{Refuted if} the slopes are not of order one; if the state's persistence $\ell$ --- from a regime-switching fit or the block-bootstrap autocorrelation --- adds explanatory power beyond $H$ and $\upsilon$, which Lemma~\ref{lem:effsample} says it should not; or, for the $\upsilon$ factor, if $\hat{\upsilon}$ falls outside the box \eqref{eq:upsbox} puts it in, which is the only part of that factor with any bite (Remark~\ref{rem:upsbox}). \emph{Status:} not refuted where it can be read --- the fit and search slopes are of order one and $\ell$ adds nothing to either --- and the $\upsilon$ factor is not identified on any admissible series, which is a prediction of the theory rather than a shortage of data and moves the falsifiable content to the box (Experiment 5).

\item \textbf{The complexity ceiling is higher for cross-sectional than for timing strategies by Proposition~\ref{prop:breadth}'s ratio --- $N_a^{\mathrm{eff}}$ when the regime terms are negligible, and less by the regime term the cross-sectional average keeps.} \emph{Refuted if} out-of-sample performance peaks at comparable complexity in both dimensions. \emph{Status:} tried, and not refuted where it can be read. Both inputs are measured rather than assumed --- the effective breadths, and the regime terms, which remove most of the factor breadth alone would predict for the library's \emph{average} --- and the ladder is now fitted: the two cross-sections' out-of-sample peaks sit eight-fold apart on the $d$ axis against the $7.8\times$ their measured terms predict ($13.2\times$ on the breadth ratio alone), the displacement of about a decade the prediction asks for, while the timing arm's curve never leaves one standard error of zero and reads only as an upper bound (Experiment 6, panels A and G).

\item \textbf{The tail's price follows $c(\Lambda)$ paid on the tail of the observed block scores.} \emph{Refuted if} the out-of-sample degradation ratio of CVaR-selected against mean-selected candidates --- each criterion's in-sample minus out-of-sample value \emph{in its own units}, since Corollary~\ref{cor:budgetrobust} prices the selection optimism of the statistic that did the selecting --- does not track $\hat{c}(\Lambda')^2$, $\Lambda'$ the level the selecting criterion was run at. \emph{Status:} read on our own rule library, the one never filtered by publication, the decay absorbed by a no-selection control: the pooled ratio lands on the plug-in at both levels it was run at and just separates from the matched Gaussian null at one of them, the cross-library ordering beyond one regime path's power (Experiment 2, panel E).

\item \textbf{Mature research pipelines sit near $\tau^2/s^2 \approx 1$ for the unshrunk estimate}, and the industry's revealed Kelly fraction near one half \emph{of that estimate}. The qualifier is the content: half Kelly on an OLS-type forecast is full Kelly on its posterior shrinkage (Remark~\ref{rem:kappa}), so the prediction concerns the former. \emph{Refuted if} the ratio is systematically far from one in either direction. \emph{Status:} not tried, and not testable without the shrinkage factor of the reported forecast. This is the prediction \S\ref{sec:calibration} declines to treat as already confirmed.
\end{enumerate}

\textbf{The instruments.} Four carry the tests above. The first measures the search term instead of bounding it: Theorem~\ref{thm:search}'s fixed-list form needs an effective number of independent trials, which is not observable, and its adaptive form replaces that by the information the search used, which is not observable either; there is a distribution-free instrument that measures the bias itself and needs no count.

\begin{proposition}[The search term is measurable]
\label{prop:audit}
Let $A_{\mathrm{full}}$ be the out-of-sample performance of a system whose search --- hyperparameters \emph{and} feature selection --- was run on all data, and let $A_{\mathrm{trunc}}$ be the performance obtained by re-running the entire search with the candidate set truncated at the out-of-sample boundary, freezing the result, and evaluating on the untouched window. Then, in expectation, $A_{\mathrm{full}} - A_{\mathrm{trunc}}$ is the search component of budget overdraft --- the selection bias of the full search on the held-out window, which is non-negative for every law of the data --- plus the gain in true risk that the extra data bought the full search's choice; whenever that gain is non-negative, which is what the extra data is for, the difference is a distribution-free \emph{upper} bound on the search component. Moreover the $k$-th re-run spends $2\log k$ nats of the held-out window's own budget, so the audit is a consumable, not a monitorable statistic.
\end{proposition}

It charges the researcher for the search she actually ran rather than the one she can describe, and its cost is the second half of the proposition: a held-out window consulted repeatedly is training data by another name, so the selection statistic of \S\ref{sec:selection} and the performance estimate here are different objects, and information flows from the first to the second and never back. Experiment 4 runs the audit exactly once, on a seventy-two-candidate search with a held-out era of seven years: the bound comes out at $0.80$ annual Sharpe, the size of the edges the search was choosing among, and it is printed and not used.

\begin{remark}[A researcher's own sample can contradict her declaration]
\label{rem:lamlower}
Split the sample into contiguous halves and compute the two empirical distributions of block risks. Let $\hat{\Lambda}_{\mathrm{split}}$ be the smallest constant for which each half's block-risk law is within density ratio $\hat{\Lambda}_{\mathrm{split}}$ of the pooled law. This is not an estimate of $\Lambda$ --- Corollary~\ref{cor:notestimable} forbids one --- but it is a statement the data can make: asserting $\Lambda < \hat{\Lambda}_{\mathrm{split}}$ asserts that the future will resemble the past more closely than the first half of one's own sample resembled the second. That is a coherent belief, but one that should be stated out loud. The same split is the forecast-breakdown test of \citet{giacomini2009detecting}, asked of the occupation measure rather than of the loss.

We recommend reporting performance as a function of $\Lambda$ --- a curve, not a number, with $\hat{\Lambda}_{\mathrm{split}}$ marked on it; Figure~\ref{fig:lambdacurve} draws it for the four candidates of Experiment 4, with Proposition~\ref{prop:occfloor}'s floor marked. It must be read against its null, the same statistic after the blocks have been shuffled, which measures the sampling floor of Proposition~\ref{prop:occfloor}: on Experiment 1's series the statistic on the block scores' own quantiles sits inside the null's $95\%$ band at every $K$, reading the floor and not the world, while on the block means of three \emph{named} coordinates --- realised volatility, the log VIX, the term spread, at $K = 5$ --- it reads $1.60$, $1.68$ and $1.68$ against a null of $1.44$, and clears it on all three. What a candidate's own performance quantiles cannot separate, coordinates a modeller can name do; a researcher declaring $\Lambda < 1.6$ on any of them is contradicted by that sample, one declaring it on the blocks is not, and is entitled to no comfort from the fact. It is a diagnostic, not a quantity to be tuned.

The same split tests the other constant of the representation, and that test is a refutation rather than a comparison. For each cell of the declared $\psi$, compare the empirical law of $(X_t,Y_t)$ over the periods the first half assigns to it with the law over the periods the second half does: under Axiom~\ref{ax:invariance} both are within $\varepsilon(\psi)$ of the pooled $Q_\psi(\cdot\mid k)$, so half their distance is a lower bound on the defect --- on the marginal form of it, which is in turn a lower bound on the path form the axiom declares, so the inequality runs the refuting way --- and a declared $\varepsilon_0$ below that number is contradicted by the researcher's own sample. This is the one place where an axiom of this paper, rather than a prediction built on the axioms, can be contradicted outright. The distance between two finite samples is biased upward and must be read against the null obtained by reassigning the halves at random --- in chunks, inside each cell, and the null's own level is priced in Appendix~\ref{app:experiments} (Experiment 1, panel H4).

Experiment 1's panel H runs it on the three coordinates that experiment declares, at $K = 2$ to $20$ cells. At the chunked null the term spread declared alone is contradicted at every $K$, by $0.033$ to $0.067$, realised volatility at two and three cells, the log VIX marginally at five and ten; at the test's own level (panel H4) those verdicts reduce to one row in fifteen, and across twenty assets to the false-positive count of the tests run (Experiment 2, panel D). \textbf{The level, not the mechanism, carried the rejections}, and this sample contradicts no declaration of a live market's state.

Two panels say what the uncorrected defect was made of, and both survive the correction as structure. It is not drift: the net defect is flat in the separation between six five-year segments (panel H2). And it is the volatility the term-spread cells mix: on realised volatility $\times$ term spread at $3\times 3$ cells the net defect is nil against $0.039$ and $0.028$ for the term spread alone, while $\Lambda(\psi)$ rises from $1.40$ to $1.73$ (panel H3) --- Proposition~\ref{prop:representation} run forwards: refining along the coordinate the cells were mixing buys the defect back, and is paid for in $\Lambda$.

One rule survives the correction, and it is a rule the axioms do not supply rather than a verdict: \textbf{a representation may not be chosen for the verdict it gives.} The term spread is the coordinate whose rejection was loudest at the uncorrected level, and the corrected level withdraws the charge, not the rule. What a representation may be chosen for is its certificate, at a price: a menu declared in advance may be searched at $\log|\Psi|$ on the budget (\S\ref{search}), and what that charge cannot buy is the truth of any member's declared defect --- that stays with this test, run member by member.
\end{remark}

\textbf{What a stationary world already implies about $\Lambda$.} From above $\Lambda$ cannot be estimated; from below it can. Because $\bar{\pi}^{+}$ is an occupation measure (Remark~\ref{rem:occupation}), a finite deployment window drifts from its own history even when nothing about the world changes, by an amount set by how many independent visits it contains. That amount is the smallest $\Lambda$ a researcher is entitled to hold.

\begin{proposition}[{The occupation floor on $\Lambda$}]
\label{prop:occfloor}
Let the latent chain be stationary on a finite set with law $\nu$ and kernel $P = (1 - 1/\ell)I + (1/\ell)\mathbf{1}\nu^\intercal$ --- at each step, with probability $1/\ell$, the state is redrawn from $\nu$ --- so that Axiom~\ref{ax:persistence} holds with coherence time $\ell$. Fix a state $k$ with $\nu(k) = p$, let $\bar{\pi}^{+}(k)$ be its occupation over a deployment window of $m$ periods, suppose the history is long enough that $\bar{\pi}_n(k) = p$, and write $\rho := 1 - 1/\ell$. Then:
\begin{enumerate}
\renewcommand{\labelenumi}{(\roman{enumi})}
\item $\mathbb{E}\bar{\pi}^{+}(k) = p$ and
\[
\mathrm{Var}\,\bar{\pi}^{+}(k) = p(1-p)[(2\ell-1)/m - 2\rho(1-\rho^m)\ell^2/m^2] \le p(1-p)(2\ell-1)/m,
\]
so the window holds $m/(2\ell-1)$ effective visits.
\item For every $\Lambda > 1$,
\[
\mathbb{P}(\bar{\pi}^{+}(k) \ge \Lambda p) \le \mathbb{E}\exp(-2(\Lambda-1)^2 p^2/W),
\]
where $W = \sum_j (L_j/m)^2$ over the lengths $L_j$ of the window's sojourns between redraws, a random variable with $\mathbb{E}W = \mathrm{Var}\,\bar{\pi}^{+}(k)/(p(1-p))$.
\item As $m/\ell \to \infty$, $\sqrt{m/(2\ell-1)}\,(\bar{\pi}^{+}(k) - p)/\sqrt{p(1-p)} \Rightarrow N(0,1)$. Hence the smallest $\Lambda$ that holds on state $k$ with probability $1-\epsilon$ is
\[
\Lambda_\epsilon(p,\ell,m) = 1 + z_{1-\epsilon}\sqrt{(1-p)(2\ell-1)/(pm)} + o(\sqrt{\ell/m}),
\]
with $z_{1-\epsilon}$ the standard Gaussian quantile.
\end{enumerate}
\end{proposition}

The scale: a ten-percent state with quarterly persistence ($p = 0.1$, $\ell = 60$) over a five-year deployment gives $\Lambda_{0.05} = 2.5$ and $\Lambda_{0.01} = 3.1$; over one year, $4.4$ and $5.8$; and for ten such cells to hold at once, $z_{1-\epsilon}$ becomes $z_{1-\epsilon/K}$ and the five-year $95\%$ floor is $3.4$. A researcher who sets $\Lambda = 1$ for such a deployment is not expressing a belief about drift; she is denying sampling error, and Remark~\ref{rem:lamlower}'s curve should start at $\Lambda_\epsilon$. Below $m/\ell \approx 10$ the exact bound (ii) should be used, the occupation of a rare state being right-skewed there; for any stationary chain, (iii) holds with $2\ell - 1$ replaced by the integrated autocorrelation time of the state indicator; and a researcher who declares at $b > 1$ has two numbers to set against the floor, which are different objects. The floor is a density ratio, and a declaration at $b$ implies the state-scale density-ratio bound at the same nominal $\Lambda$ (\S\ref{sec:axioms}), so the declaration itself clears the floor or not at that number; the state-scale reading $\Lambda_1$ of Proposition~\ref{prop:scale} is the $\Lambda$ whose Gaussian penalty the criterion at $b$ actually pays, and it says how much of the declared protection the criterion delivers, not what ratio was declared. Experiment 4 declares $\Lambda = 4$ at $b = 60$ against a floor of $2.24$ on its held-out window, and pays for $2.99$: both clear the floor, and the gap between them is Proposition~\ref{prop:scale}'s. Neither comparison is made at the block scale itself, where the future occupation is a $B$-point hull and the floor is not computed (\S\ref{sec:limits}, item 2).

\section{Limitations and open problems}
\label{sec:limits}
Eight limitations, in four groups: where the finite-sample theory is worst-case or proved for one chain; the tail read through noisy blocks; the status of the axiomatisation; and the one real series the paper walks end to end. Each is a place where a claim this paper makes is weaker than the claim it wants, none is closed by rewording, and each says in its last column why it is not repaired here. There are three answers. \emph{Impossible}: a theorem below rules the repair out, so what stays open is a different object, not a better estimate. \emph{No sample}: the repair needs a record that does not exist, and the entry says how long or how wide. \emph{Unwritten}: it is a theorem someone could write, and the entry says what it would take --- those are the ones we would do next. Three decide whether the framework stands: which representation (item 5), the third cost of the tail (item 4) and the empty position on the one real series (item 8), the last two set out in prose.

Two kinds of complaint are deliberately absent. \textbf{A declared constant is not a limitation for being unestimable}: that the state-wise end of $\rho$ is not read off a track record is the framework working as designed --- the axioms exist because these quantities cannot be estimated, each is declared where it is spent (Axiom~\ref{ax:snr}), and what a sample does instead is strike a declaration, never certify one (\S\ref{sec:empirical}). \textbf{And a direction not taken is not a gap in what is proved}: an equilibrium treatment of decay after \citet{grossman1980impossibility}, or a turnover-aware budget reducing $\rho$ before the budget is computed, would each be a paper, and neither moves a theorem here --- both are charged where they enter, at stage S5, Definition~\ref{def:reflexivity}'s $\delta$ declared and its cost term $c'$ measured, $3.5\%$ of the declared $\rho$ on Experiment 4's series.

\subsection{Worst-case bounds, and results proved for one chain}
\label{sec:limits-finite}
\begin{center}
\footnotesize
\begin{longtable}{@{}>{\raggedright\arraybackslash}p{0.21\linewidth}>{\raggedright\arraybackslash}p{0.61\linewidth}>{\raggedright\arraybackslash}p{0.11\linewidth}@{}}
\toprule
\textbf{what is open} & \textbf{where it stands, and what prices it} & \textbf{why not here} \\
\midrule
\endhead
1. Bounded losses & Theorem~\ref{thm:deployed} bounds the block scores by $M$; returns are heavy-tailed and a leveraged rule's loss is unbounded. The bound trades for a $p$-th moment (at $B^{-1/2}$ for $p \ge 2$, $B^{-(p-1)/p}$ below) but not for the constant, and below the second moment $c(\Lambda)$ is infinite, so Corollary~\ref{cor:lamcap}'s arithmetic has no analogue there. What would close the first half is a truncated plug-in carrying the truncation into $c(\Lambda)$; nothing closes the second, and no moment condition touches item 4. & Unwritten, then impossible \\
2. The recurrence axiom is a hypothesis on the path, and its failure probability is computed for one chain & A window of $m$ periods holds only $m/(2\ell - 1)$ effective visits, so Axiom~\ref{ax:recurrence} can fail in a stationary world. Proposition~\ref{prop:occfloor} makes it a number for the redraw chain --- and, with $2\ell - 1$ replaced by the state indicator's integrated autocorrelation time, for any stationary chain --- but says nothing about $\Lambda$ above the floor, and at $b > 1$ the occupation is a $B$-point hull, so the same window buys fewer degrees of freedom than the count suggests. What theory gives without a sample is that floor; above it the axiom is a hypothesis about the path a researcher actually got. & No sample \\
3. The effective sample and the embargo are exact for one chain, and elsewhere are declarations & Lemma~\ref{lem:effsample} and Proposition~\ref{prop:embargo} (ii) are exact for the redraw chain, upper bounds for reversible ones, and hold elsewhere with $2\ell - 1$ replaced by a declared coherence time that no decoupling coefficient bounds (Example~\ref{ex:ladder}) --- one the axiom's per-coordinate clause fixes for a one-coordinate state and for independent coordinates (Lemma~\ref{lem:effsample} (iii)), and otherwise leaves to the declaration. The block count of Theorem~\ref{thm:deployed} is the same lemma's identity and sits in the same position. Every such declaration the sample bounds from below only --- on the declared coordinates the embargo's closed form certifies nothing (Experiment 1, panel J), and the $1\%$ convention is wrong in what it scales with, not merely in size --- and the $\upsilon$ below one on real daily series is the predictability term $\pi_\mu$ inside Remark~\ref{rem:upsbox}'s box, which every worksheet computes at $\pi_\mu = 0$, anti-conservative where it holds. What would close it is a general-chain form of either formula; the obstruction is that both read a two-block covariance the axioms leave free. & Unwritten \\
\bottomrule
\end{longtable}
\end{center}

\subsection{The tail read through noisy blocks}
\label{sec:limits-noise}
\textbf{4. A third cost of the tail, priced but not removed.} Proposition~\ref{prop:eiv} gives the block-noise ranking bias its sign, closed form and $b$-dependence (\S\ref{sec:eiv}). It does not vanish as $B$ grows, and neither lever the criterion has removes it. The block length is not one: at the $\ell$ a researcher faces, the ratio $\sigma_\xi^2/s^2$ floors at $H/(\rho_{\mathrm{blk}}^2(2\ell - 1))$, and past $b \approx \ell$ a longer block destroys the dispersion the tail reads faster than it removes the noise (Proposition~\ref{prop:eiv} (ii)). The tail level is not one either: the criterion's sensitivity to the dispersion it exists to read is flattened by $1/\sqrt{1 + \sigma_\xi^2/s^2}$ at every level, and at the ratio real block scores carry --- $6.6$ on Experiment 1's series --- a level narrow enough to buy that back averages too few blocks to be a tail mean, which is what the narrow levels of Experiment 6's panel C show: the criterion loses there by an amount that moves with the block count and not with the window.

So what remains open is not a sharper estimate inside the tail-mean class but a criterion outside it: a reading of the regime tail that does not pass through the tail mean of noisy block scores. Any rule that weights a block by its own noisiness discounts exactly the stressed blocks Theorem~\ref{thm:cvar} needs read --- studentising removes the excess over the Gaussian value on every asset (Experiment 2) and with it what the tail exists to read --- so the way out, if there is one, is not a reweighting.

The same bounds what the calibration can be said to have measured. Everything \S\ref{sec:calibration} reports about where the price sits is measured on block \emph{scores}, and on the raw block-mean statistic; a block \emph{risk} law is measured nowhere and at fixed $b$ cannot be, the score noise carrying $6.6$ times the block signal's variance (Experiment 1, panel B), so $\hat{c}(\Lambda)$ is a statement about scores by necessity, not by choice. Two of the three libraries are rules of our own design besides, which the cross-section of Experiment 6 fixes and the daily libraries do not.

\subsection{The status of the axiomatisation}
\label{sec:limits-axioms}
\begin{center}
\footnotesize
\begin{longtable}{@{}>{\raggedright\arraybackslash}p{0.21\linewidth}>{\raggedright\arraybackslash}p{0.61\linewidth}>{\raggedright\arraybackslash}p{0.11\linewidth}@{}}
\toprule
\textbf{what is open} & \textbf{where it stands, and what prices it} & \textbf{why not here} \\
\midrule
\endhead
5. $\Lambda$ is asserted per representation, and nothing says which representations belong on the menu & Two researchers can assert the same $\Lambda$ and mean different things, and declaring $\varepsilon_0$ alongside it makes the coarsest admissible $\psi$ well posed but not existent. Measuring both ends on the same partitions (Experiment 1, panel H), refining raises $\Lambda$ along any coordinate, while the defect end sits inside its null at the test's own level everywhere but the stalest cells, so the sample prices only the $\Lambda(\psi)$ side. The $\varepsilon(\psi)$--$\Lambda(\psi)$ trade of Proposition~\ref{prop:representation} is stated and not solved; a search of a \emph{declared} menu is on the budget at $\log|\Psi|$ (\S\ref{search}); and the scale is priced by Proposition~\ref{prop:scale}, with an interior optimum of order $\ell$ implied by Proposition~\ref{prop:eiv} (ii), so a state-scale guarantee at $\Lambda$ needs a block short enough to deliver it, which \S\ref{sec:nonempty} says she cannot afford. What no charge buys is the menu itself, and what would buy it is a rule choosing a coarsening without reading its verdict. & Unwritten \\
6. The canonical form is a necessity theorem, not a factorisation & Theorem~\ref{thm:canonical} does not say every good procedure decomposes into five recognisable steps (Remark~\ref{rem:canonstatus}); two margins are proved in models narrower than the axioms, and necessity constrains stages, not implementations. What would close it is a converse: every procedure meeting the bound put into five-stage form. & Unwritten \\
7. The selection error is sized in the Gaussian range, and the tail statistic's block count is derived only there & Theorem~\ref{thm:deployed} spends the scores' record coherence --- Lemma~\ref{lem:effsample}'s identity $\upsilon/\upsilon_b$ (Corollary~\ref{cor:blocks} (a)) --- as a block count in a fluctuation stated where Proposition~\ref{prop:tailprice}'s limit holds, where the tail truncation of a score inherits the scores' count because the scores form a Gaussian pair across blocks with non-negative correlations; the finite-sample form of the same bound is stated on independent records only. Neither half holds on Experiment 1's series as measured: the scores' magnitudes are coherent across blocks at $1.7$ while the scores themselves are not, a dependence no Gaussian pair has, so the tail statistic's count there is read on the sample ($1.05$, panel B) and not derived. What would close it is a concentration inequality for the tail plug-in whose constant is the scores' coherence and not a mixing coefficient, which the present tools supply only under mixing. Nor are the axioms proved minimal: Remark~\ref{rem:independence} exhibits, for each, a model satisfying the others, which does not show the strength of Axiom~\ref{ax:invariance} least. Few of the inequalities here are checked by anything but their proofs (\texttt{check\_constants.py} spot-checks the directions of six statements), and we do not claim Proposition~\ref{prop:kapceil} is the last such bound. & Unwritten \\
\bottomrule
\end{longtable}
\end{center}

\subsection{The market side}
\label{sec:limits-market}
\textbf{8. On the one real series the paper walks, the procedure takes no position.} Experiment 4 runs S1--S5 on thirty years of one index at $\Lambda = 4$: the selected rule fails S5 on the search haircut, and under the dispersion penalty as Proposition~\ref{prop:kelly} commits to it \emph{no} candidate in the library has a position, each edge surviving only to $\Lambda \approx 1.1$--$1.3$, which is Proposition~\ref{prop:occfloor}'s floor for that window read at the level the curve is drawn at, $p = 1/\Lambda$ (the fixed point $1.17$; at $p = 0.1$ the floor is $2.24$, further out still).

We take the literal reading, and state it as a \emph{reductio}: at $\Lambda = 4$ a long-only index, or a timing rule on it, has no robust edge at this block length --- the equity premium is a statement about a mean, robustness at $\Lambda$ is a statement about the worst $1/\Lambda$ of a reweighted future, and a declaration of $\Lambda = 4$ on a daily single-asset series is a declaration about an empty set.

The two escapes have been measured and neither holds. On a declared representation --- Theorem~\ref{thm:cvar} read at state scale on the cells of a named coordinate --- at the least admissible state-wise declaration every band that differs from the block value contains zero, and at ten cells all three coordinates are negative with bands excluding zero (Experiment 4, panel G). And the emptiness was computable before the block length was chosen (\S\ref{sec:nonempty}), the search ceiling biting first and still biting at the regime dependence the library measures.

That the procedure \emph{can} clear S5 is a matter of arithmetic rather than of this record: Corollary~\ref{cor:nonempty} says what a configuration must supply --- an edge, a block length and a search size that clear its two bounds --- and no configuration in this record supplies it.

What separating the three causes would take is the one thing this record cannot supply. Length is priced and refused: \S\ref{sec:nonempty} asks seventy years of the same index for the gap between the latent and the observable dispersion, and Proposition~\ref{prop:breadth}'s answer is that the years are not worth it. Width is the substitute, and Experiment 6 buys as much of it as the published record holds --- $133$ long-short signals against $82$ long-only portfolios over sixty years --- where the verdict on P1 is a null with power. So this stays open for want of a sample, and the escape it points to is not a longer history but a wider one.

\section{Conclusion}
\label{conclusion}
Two conclusions. \textbf{The axioms force the canonical form.} The belief that makes systematic trading possible --- that past regularities persist --- is not a claim that the world is stationary; it is a claim about a mechanism invariant conditional on a state one cannot see, held to five declared constants. So stated, it forces the architecture: the selection objective is the block-CVaR at level $1/\Lambda$ (Theorem~\ref{thm:cvar}), the research a sample can pay for is a computable budget (\S\ref{sec:budget}), robustness is a purchase priced in the same currency (\S\ref{sec:price}), the bet is $1/(1+\varkappa)$ of the estimate (\S\ref{sec:sizing}), and each of the five stages S1--S5 is necessary --- a procedure omitting one does strictly worse under a law the axioms admit, by the margins of Theorem~\ref{thm:canonical} --- while what remains free is exactly what the axioms leave free: the declared constants, and the implementation inside the budget (Remark~\ref{rem:canonstatus}).

\textbf{The data have so far contradicted declarations, not the structure.} What a sample can interrogate are the declared constants (\S\ref{sec:empirical}), and the record of the experiments is that at the honest level no axiom is refuted on a live market's state. The one constant a sample can contradict outright, the invariance defect, is contradicted on no declared coordinate at the level the test is required to carry (\S\ref{sec:empirical}); what real series do reject are declarations --- the conservative $\kappa = 1$, and the coherence time $\ell = 60$ Experiment 4 declares of its volatility coordinate, which that coordinate's own rank reads at $138$. Where the evidence has moved against a prediction --- the tail criterion buys nothing over mean selection on the cross-sectional library at the one level whose tail can be read --- no theorem is refuted, the theorem being about a worst case a held-out sample need not contain; but the case for paying $\lambda(\Lambda)$ at selection now rests on the theorem rather than on our evidence. The framework would be most usefully right if practitioners began reporting performance as a function of $\Lambda$: one extra column in a table, and the one assumption on which everything else rests becomes visible.

\vspace{1em}
{\footnotesize\noindent\textbf{Acknowledgements.} The drafting, the review
that produced the current \S8 including the errors-in-variables cost of
Proposition~\ref{prop:eiv}, and the open-problem list were carried out with
Claude (Anthropic), working through Claude Code; responsibility for the
content is the author's.\par}
\vspace{1em}

\bibliographystyle{unsrtnat}
\bibliography{references}

\clearpage
\appendix
\section{Proofs}
\label{app:proofs}
\begin{proof}[Proof of Lemma~\ref{lem:sharpe}]
With $s^2 := \mathbb{E}[\mu(X)^2]/\sigma^2$ and $\sigma_c^2$ the common value of $\mathrm{Var}(Y \mid X)$, conditioning the PnL $W = \mu(X) Y/(\gamma \sigma^2)$ of $h^\star$ on $X$ gives $\mathbb{E} W = \mathbb{E}[\mu^2]/(\gamma \sigma^2)$ and $\mathrm{Var}(W) = (\sigma_c^2\, \mathbb{E}[\mu^2] + \mathrm{Var}(\mu^2))/(\gamma \sigma^2)^2$, so $\mathrm{SR}(h^\star) = s (1 + \delta)^{-1/2}$ with $\delta := -\mathrm{Var}(\mu)/\sigma^2 + (k_\mu - 1) s^2$, $k_\mu := \mathbb{E}[\mu^4]/(\mathbb{E}[\mu^2])^2$ --- an identity, $\gamma$ cancelling. Since $0 \le \mathrm{Var}(\mu) \le s^2 \sigma^2$ and $s \le \rho$ ($\mu(X) = \mathbb{E}[\mu_\pi \mid X]$, Jensen, Axiom~\ref{ax:snr}), $|\delta| \le k_\mu s^2$, and the mean value theorem on $u \mapsto (1+u)^{-1/2}$ gives the remainder $k_\mu \rho^3/(2(1-\rho^2)^{3/2})$; $k_\mu = 3$ for a Gaussian forecast, so at $\rho^2 = 10^{-2}$ the remainder is $1.5 \times 10^{-3}$ against $s = 0.1$. The same lines hold under any mixture in the state-scale ball, with $\mu_\pi$, $\sigma_\pi$ and the mixture form of the axiom (\S\ref{sec:axioms}). The hypothesis on $\mathrm{Var}(Y \mid X)$ is what the identity needs, not a convenience: if $\sigma^2(x) := \mathrm{Var}(Y \mid X = x)$ varies, the mean--variance optimum is $\mu(x)/(\gamma \sigma^2(x))$, of Sharpe ratio $\sqrt{\mathbb{E}[\mu^2/\sigma^2(X)]}$ to the same order, and Axiom~\ref{ax:snr} bounds $\mathbb{E}[\mu^2]/\sigma^2$, not that --- a feature that forecasts the variance buys a Sharpe ratio the ceiling does not price, and a track record that did so reads a $\rho_1$ above the one this lemma converts.
\end{proof}

\begin{proof}[Proof of Theorem~\ref{thm:impossible}]
\emph{The construction in outline.} Adjoin to the latent space one fresh state $z_h^\dagger$ per strategy, whose kernel is built to oppose exactly $h$; let $h^\star$ be the estimator's modal output under $P^0$. Run the latent chain in $\mathcal{Z}_0$ for $t \le n$ except for one contiguous visit to $z_{h^\star}^\dagger$ of length $k = \lceil n/\Lambda \rceil$, and in $z_{h^\star}^\dagger$ for $t > n$. The kernel is a fixed, $t$-free object and the latent path is chosen without looking at the data, so Axiom~\ref{ax:invariance} holds; the past is a $k$-replacement of $P^0$ by construction; the density ratio of future to past occupation is exactly $n/k$; and on the event $\{\hat{h}_n = h^\star\}$, whose $P_\Lambda$-probability is within $\epsilon_n(k)$ of $p^\star$ by the definition of the sensitivity, the estimator has deployed the one strategy the future is built to oppose.

Set $\mathcal{Z} = \mathcal{Z}_0 \sqcup \{z_h^\dagger\}_{h \in \mathcal{H}}$ with $\mathcal{Z}_0$ carrying the assumed latent representation of $P^0$. Define
\[
Q(\mathrm{d}x,\mathrm{d}y \mid z_h^\dagger) = P_X^0(\mathrm{d}x)\,\mathcal{L}\big(Y \mid \mathbb{E}[Y] = -\rho\sigma\,\mathrm{sgn}\,h(x),\; \mathrm{Var}(Y) = \sigma^2\big),
\]
with $\sigma^2 := \mathrm{Var}_{P^0}(Y)$. Two things about this kernel. It satisfies Axiom~\ref{ax:snr} with the stated constants: its conditional mean $\mu^\dagger(x) = -\rho\sigma\,\mathrm{sgn}\,h(x)$ has $\mathbb{E}[\mu^\dagger(X)^2] \le \rho^2\sigma^2 \le \rho^2\mathrm{Var}(Y)$, the variance of $Y$ under it being $\sigma^2 + \mathrm{Var}(\mu^\dagger(X)) \ge \sigma^2 \ge \underline{\sigma}^2$. (Given $z_h^\dagger$ the draw is independent of the past by construction, so its conditional mean given the past and the state is $\mu^\dagger(X)$.) Since $P_\Lambda$'s states are those of $P^0$ together with this one, Axiom~\ref{ax:snr} holds for $P_\Lambda$ state by state, hence in its mixture form on every state-scale ball (\S\ref{sec:axioms}), in particular on the ball of radius $n/k$ that Axiom~\ref{ax:recurrence} will declare. (This is where the axiom has to be stated state by state: a mixture of the new ball that puts weight $w$ on $z_h^\dagger$ restricts to $\mathcal{Z}_0$ with density up to $(n/k - 1)/(1-w)$ against $\bar{\pi}_n^0$, outside the ball of radius $n/k$ as soon as $w > k/n$, so a bound asserted over $P^0$'s own mixtures would not reach it.) And it moves the conditional mean by at most $\Delta = 2\rho\bar{\sigma}$ in root mean square: with $\mu^0(x) := \mathbb{E}_{P^0}[Y \mid X = x]$, Minkowski and Axiom~\ref{ax:snr} under $P^0$ give $\|\mu^\dagger - \mu^0\|_{L^2(P_X^0)} \le \|\mu^\dagger\|_{L^2} + \|\mu^0\|_{L^2} \le \rho\sigma + \rho\sigma \le 2\rho\bar{\sigma}$ --- no more than the ceiling's own size, which is the sense in which the ambush needs no unusual leverage. (A pointwise bound would need $|\mu^0(x)| \le \rho\bar{\sigma}$ at every $x$, which Axiom~\ref{ax:snr} does not assert; nothing below needs it, and truncating $\mu^\dagger$ to obtain one would spoil the equality in (iii).) Each $Q(\cdot \mid z_h^\dagger)$ is a fixed kernel indexed by $h$ and does not depend on $t$.

Let $h^\star$ attain $p^\star$. Build the latent path as follows, \emph{without reference to the data}: for $t \le n$ follow the $\mathcal{Z}_0$-representation except on a single contiguous visit to $z_{h^\star}^\dagger$ of length $k = \lceil n/\Lambda \rceil$, placed at the start of the window; for $t > n$ sit in $z_{h^\star}^\dagger$. Conditionally on this path let the observations be independent draws from $Q(\cdot \mid Z_t)$; this is Axiom~\ref{ax:invariance} by construction. The insertions are deterministic and the observations are conditionally independent draws given the path, so every bounded function of a coordinate has the autocorrelations it has under the $\mathcal{Z}_0$-representation, and Axiom~\ref{ax:persistence} holds at the same $\ell_i$.

Occupations: $z_{h^\star}^\dagger$ is fresh, so $\bar{\pi}_n(z_{h^\star}^\dagger) = k/n$ exactly and $\bar{\pi}^+ = \delta_{z_{h^\star}^\dagger}$, whence $\mathrm{d}\bar{\pi}^+/\mathrm{d}\bar{\pi}_n = n/k \le \Lambda$ and Axiom~\ref{ax:recurrence} holds with constant $n/k$, giving (ii). The observable past differs from $P^0$ only on the $k$ periods of the visit, where the conditional law given the path is the fixed kernel $Q(\cdot \mid z_{h^\star}^\dagger)$, which satisfies Axiom~\ref{ax:snr} with the stated constants and moves the conditional mean by at most $\Delta$; so it is a $k$-replacement of $P^0$ in the sense of \eqref{eq:replacek}, and by the definition of $\epsilon_n(k)$ the law of $\hat{h}_n$ under $P_\Lambda$ is within total variation $\epsilon_n(k)$ of its law under $P^0$, giving (i); in particular $P_\Lambda(\hat{h}_n = h^\star) \ge p^\star - \epsilon_n(k)$. (A bound on the total variation between the two laws of the \emph{past} would not do: the $k$ replaced coordinates sit at a fixed distance, so that quantity does not shrink with $\Lambda$, and only the data-processing step through $\hat{h}_n$ can make it small.)

For (iii), on that event, conditionally on the path, $R^+(\hat{h}_n) = R^+(h^\star) = \rho\sigma\,\mathbb{E}|h^\star(X)|$, while for every $h'' \in \mathcal{H}$, $R^+(h'') = -\mathbb{E}[h''(X)Y \mid z_{h^\star}^\dagger] = \rho\sigma\,\mathbb{E}[h''(X)\,\mathrm{sgn}\,h^\star(X)] \le \rho\sigma\,\mathbb{E}|h''(X)|$, which the common position scale makes equal to $\rho\sigma\,\mathbb{E}|h^\star(X)|$. So $R^+(h^\star) = \sup_{h''} R^+(h'')$.
\end{proof}

\begin{proof}[Proof of Corollary~\ref{cor:notestimable}]
(a) Write $k$ for the distinguished state, $p = \nu(k)$, $O^+ := \bar{\pi}^+(k)$, $O_n := \bar{\pi}_n(k)$. By the Markov property and the redraw structure the past influences the future only through the current sojourn, a term of order $\ell/m$ in $O^+$; so conditionally on the past, $O^+$ has the non-degenerate law of Proposition~\ref{prop:occfloor} up to that term, with conditional variance $p(1-p)(2\ell-1)/m + O(\ell^2/m^2)$ not depending on $n$. Since $O^+/O_n \le \Lambda_{n,m} \le \max_z \bar{\pi}^+(z)/\bar{\pi}_n(z)$ and $O_n \to p$ almost surely, the conditional law of $\Lambda_{n,m}$ given the past has a non-degenerate limit, to which a past-measurable $\hat{\Lambda}_n$ --- a constant given the past --- cannot converge in probability. (b) The event $\{O_n = 0\}$ has probability at least $(1-p)^{n+1}$, and on it $\{O^+ > 0\}$ has conditional probability at least $1 - (1 - p/\ell)^m$, each period of the window being a redraw landing on $k$ with probability $p/\ell$; on the intersection $\bar{\pi}^+$ charges a state $\bar{\pi}_n$ does not, and $\Lambda_{n,m} = \infty$.
\end{proof}

\begin{proof}[Proof of Lemma~\ref{lem:robust} and Theorem~\ref{thm:cvar}]
The lemma is the robust representation of average value at risk; see \citet[Theorem~4.52]{follmer2016stochastic} or \citet{rockafellar2000cvar}. For the theorem, fix the path and let $b$ divide $n$ (otherwise the $\le b-1$ periods outside every block are dropped from $\bar{\pi}_n$). Axiom~\ref{ax:invariance} and \eqref{eq:regimerisk} make $R^+(h) = \int r_h\, \mathrm{d}\bar{\pi}^+$ linear in the occupation measure, and Axiom~\ref{ax:recurrence} writes $\bar{\pi}^+ = \sum_j w_j \bar{\pi}_{n,j}$ with $\max_j w_j \le \Lambda/B$, so $R^+(h) = \sum_j w_j R_j(h)$ for an admissible $w$; the admissible set is exactly the lemma's for $P$ uniform on $B$ atoms and $\alpha = 1/\Lambda$, whose supremum is the mean of the worst $\lceil B/\Lambda \rceil$ atoms with the boundary atom at fractional weight. The supremum is attained by a deployment path that re-runs each of the worst blocks $\Lambda m/(Bb)$ times when that is an integer, and approached as $m$ grows otherwise, the realisable $w$ having entries in $(b/m)\mathbb{N}$. Monotonicity in $\Lambda$ is the growth of the constraint set; at $\Lambda = 1$ the only admissible $w$ is uniform and the value is $\bar{R}_n(h)$; at $\Lambda \ge B$ all weight may sit on one block.
\end{proof}

\begin{proof}[Proof of Proposition~\ref{prop:representation}]
With $\|P - P'\|_{\mathrm{TV}} := \sup_A |P(A) - P'(A)|$, $\left|\int f\, \mathrm{d}P - \int f\, \mathrm{d}P'\right| \le 2\|f\|_\infty \|P-P'\|_{\mathrm{TV}}$ for bounded $f$ (Jordan decomposition). Write $\mathcal{G} := \sigma(\psi(Z))$, $r_h^\psi(z') := \int L(h(x), y)\, Q_\psi(\mathrm{d}x, \mathrm{d}y \mid z')$, and $\bar{\pi}^{+,\psi}$, $\bar{\pi}_{n,j}^\psi$ for the images under $\psi$ of the occupation measures of \eqref{eq:occupation}. Three steps. The definition of $\varepsilon(\psi)$ and the display above give $|\mathbb{E}[L(h(X_t), Y_t) \mid \mathcal{G}] - r_h^\psi(\psi(Z_t))| \le 2M\varepsilon(\psi)$ almost surely for every $t$, so $R_\psi^+(h) = \mathbb{E}[R^+(h) \mid \mathcal{G}] = \int r_h^\psi\, \mathrm{d}\bar{\pi}^{+,\psi} + E$ with $|E| \le 2M\varepsilon(\psi)$ (tower property), and likewise for each historical block, which is what Corollary~\ref{cor:blocks} becomes at a positive defect. Taking images under $\psi$ is linear, so the coarsened path satisfies Axiom~\ref{ax:recurrence} at the same scale with the same weights, and its own constant $\Lambda(\psi)$ is at most $\Lambda$: every $w$ admissible for the fine path is admissible for the coarse one. Theorem~\ref{thm:cvar} used only that linearity and that mixture condition, so it holds in $\mathcal{Z}'$ with constant $\Lambda(\psi)$, and adding $E$ gives the proposition; nothing used any property of the kernel $Q_\psi$, and tightness is not claimed --- $2M\varepsilon(\psi)$ is a worst case over the sign of the defect at every date. With $\psi$ the identity and $Q_\psi = Q$, $\varepsilon(\psi) \le \varepsilon_0$ and $\Lambda(\psi) \le \Lambda$: Theorem~\ref{thm:cvar} at a positive defect, the exchange rate \S\ref{sec:axioms} quotes.
\end{proof}

\begin{proof}[Proof of Lemma~\ref{lem:effsample}]
Write $L_t := L(h(X_t), Y_t) = r_h(Z_t) + \xi_t$ with $\xi_t$ the mechanism noise, conditionally mean zero given the path. The two parts are uncorrelated, so $\mathrm{Var}(\hat{R}_n) = \mathrm{Var}(n^{-1}\sum_t r_h(Z_t)) + \mathrm{Var}(n^{-1}\sum_t \xi_t)$. This split gives both ends of \eqref{eq:upsbox}. The first term is the variance of an average, hence non-negative for every chain and every $n$, and it survives a mis-declared state: if the modeller omits a coordinate of $Z$, the regime risk carried by that coordinate moves into $\xi$, whose variance is then the sum of its own non-negative regime term and the noise below. The second term is $\sigma^2 H/n$ from the overlap, by the computation below, plus $2\sigma^2\pi_\mu/n$ from the covariances at distance $\ge H$, which Axiom~\ref{ax:invariance} does not set to zero (\S\ref{sec:axioms}).

For those: for $t \ge s+H$ the variable $\xi_s$ is $\mathcal{F}_t^{\mathrm{obs}} \vee \sigma(Z)$-measurable ($Y_s$ is known at $s+H$), so $\mathrm{Cov}(\xi_s,\xi_t) = \mathbb{E}[\xi_s\, p_t]$ with $p_t := \mathbb{E}[\xi_t \mid \mathcal{F}_t^{\mathrm{obs}}, Z] = \mathbb{E}[L_t \mid \mathcal{F}_t^{\mathrm{obs}}, Z] - r_h(Z_t)$ the part of the loss the observable past predicts beyond the state. Cauchy--Schwarz gives $|\mathrm{Cov}(\xi_s,\xi_t)| \le \|\xi_s\|_2 \|p_t\|_2 = \sigma\|p_t\|_2$, and $\|p_t\|_2 \le \rho\sigma$ is Axiom~\ref{ax:snr}: for the PnL loss with $|h| \le 1$, $p_t$ is the within-state centring of $-h(X_t)\, \mathbb{E}[Y_t \mid \mathcal{F}_t^{\mathrm{obs}}, Z]$, so its mean square is at most $\mathbb{E}[(\mathbb{E}[Y_t \mid \mathcal{F}_t^{\mathrm{obs}}, Z])^2] \le \rho^2\, \mathbb{E}[\mathrm{Var}(Y \mid Z)] \le \rho^2\sigma^2$. This is where the axiom has to be stated state by state: its mixture form bounds the conditional mean given the past alone, a trader who must infer the state, while $p_t$ describes one who knows it, and under the mixture form alone the only bound is the trivial $\|p_t\|_2 \le \sigma$. So $|\mathrm{Cov}(\xi_s,\xi_t)| \le \rho\sigma^2$ at every lag $\ge H$, and $|\pi_\mu| \le \rho\tau_\mu$ with $\tau_\mu := \sum_{k \ge H} |\mathrm{Cov}(\xi_0,\xi_k)|/(\rho\sigma^2)$ the coherence of that component in periods; at finite $n$ the term is $2\sum_{H \le k < n} (1 - k/n)\mathrm{Cov}(\xi_0,\xi_k)/\sigma^2$, the infinite sum up to a relative $\tau_\mu/n$. Hence $\upsilon \ge H - 2\rho\tau_\mu$ at every $n$ and for every chain: the lower end of the box is Axiom~\ref{ax:snr}'s, not Axiom~\ref{ax:invariance}'s. Read in units of the unconditional $\mathrm{Var}(Y)$ rather than of the occupation-average conditional variance, the overlap term is $H(1 - \rho_{\mathrm{blk}}^2)$ and the floor falls by a further $\rho_{\mathrm{blk}}^2 H$, below $10^{-3}$ at any $\rho_{\mathrm{blk}}$ this paper measures.

For the first, $\mathrm{Cov}(r_h(Z_s), r_h(Z_t)) = \rho^{|s-t|}\mathrm{Var}_\nu(r_h)$ under the redraw chain (proof of Proposition~\ref{prop:occfloor}), and summing over $s,t \le n$ gives $\mathrm{Var}_\nu(r_h)\, n^{-1}\sum_{|j|<n}(1-|j|/n)\rho^{|j|} = \mathrm{Var}_\nu(r_h)\, v(n,\ell)/n$, the closed form of $v$ being the geometric sums $(1+\rho)/(1-\rho) = 2\ell-1$ and $(1-\rho)^{-2} = \ell^2$. For the second, the equal-weight overlap of an $H$-period return makes $\xi_t = H^{-1/2}\sum_{s=1}^H e_{t+s}$ with the $e$ conditionally independent of variance $\sigma^2$, so $\mathrm{Cov}(\xi_s,\xi_t) = \sigma^2(1-|s-t|/H)^+$ and $\mathrm{Var}(n^{-1}\sum_t \xi_t) = \sigma^2 n^{-1}\sum_{|j|<H}(1-|j|/H)(1-|j|/n) = (\sigma^2/n)\,(H - (H^2-1)/(3n))$, since $\sum_{|j|<H}(1-|j|/H) = H$ and $\sum_{|j|<H}|j|(1-|j|/H) = (H^2-1)/3$; the edge term is the relative $O(H/n)$ the statement allows.

For a general chain under Axiom~\ref{ax:persistence} the first sum is $n^{-1}$ times the integrated autocorrelation time $\sum_{j\in\mathbb{Z}} \mathrm{Corr}(r_h(Z_0), r_h(Z_j))$, which the axiom declares. A mixing coefficient would bound it only through the covariance inequality $|\mathrm{Cov}(f(Z_0), g(Z_j))| \le 4\|f\|_\infty\|g\|_\infty\beta(j)$ for bounded $f,g$ --- from a decay summable over every lag, and up to a factor $\|r_h\|_\infty^2/\mathrm{Var}_\nu(r_h)$ that Example~\ref{ex:ladder} shows no such coefficient removes; the axiom declares the time itself.

For several coordinates that are independent redraw chains, expand $r_h$ in its Hoeffding decomposition $\sum_S r_S$, each $r_S$ having zero conditional mean given any proper sub-collection of its coordinates; the terms are orthogonal at every pair of lags, because a redraw of one coordinate in $S$ between the two times kills the conditional mean, and $\mathrm{Cov}(r_S(Z_0), r_S(Z_j)) = \mathrm{Var}(r_S)\prod_{i\in S} \rho_i^{|j|}$ since the term survives only if no coordinate of $S$ was redrawn. So each subset is charged $v(n,\ell_S)$ with $1 - 1/\ell_S = \prod_{i\in S}(1 - 1/\ell_i)$, and the singletons give the sum in the statement.

For a reversible chain the bound is function-free. Let $P$ be the transition operator on $L^2(\nu)$, self-adjoint by reversibility, with eigenvalues $1 = \lambda_0 > |\lambda_i|$ and orthonormal eigenfunctions $\varphi_i$; expand $r_h - \nu(r_h) = \sum_{i\ge 1} a_i \varphi_i$. Then $\mathrm{Cov}(r_h(Z_0), r_h(Z_j)) = \langle r_h, P^{|j|} r_h \rangle_\nu = \sum_i a_i^2 \lambda_i^{|j|}$, so the regime part of the variance is $n^{-1}\sum_i a_i^2 v_{\lambda_i}(n)$ with $v_\lambda(n) := \sum_{|j|<n}(1-|j|/n)\lambda^{|j|}$, and $v_\lambda(n) \le v_{|\lambda|}(n) \le v_{\lambda^*}(n)$ term by term for $|\lambda| \le \lambda^* := \max_{i\ge 1}|\lambda_i|$. Writing $\ell := 1/(1-\lambda^*)$ for the relaxation time, $v_{\lambda^*}(n)$ is exactly the $v(n,\ell)$ of \eqref{eq:vcontig}, so the regime part is at most $\mathrm{Var}_\nu(r_h)\, v(n,\ell)/n$, with equality if and only if $r_h - \nu(r_h)$ lies in the eigenspace of $\lambda^*$. The redraw chain has $P = (1-1/\ell)I + (1/\ell)\mathbf{1}\nu^\top$, whose non-trivial spectrum is the single value $1-1/\ell$; every $r_h$ lies in that eigenspace, which is why the lemma is exact there and why no reversible chain with the same relaxation time does worse. The relaxation time is bounded by the worst-case mixing time \citep[Thm.~12.5]{levin2017markov}.
\end{proof}

\begin{proof}[Proof of Theorem~\ref{thm:capacity}]
Take the linear class and the loss $(y - \theta^\top x)^2$, let $\theta_z$ be the loading in state $z$ and $\bar{\theta}_n := n^{-1}\sum_t \theta_{Z_t}$ the occupation-weighted loading, which minimises $\bar{R}_n$ over the class; the excess risk of $\hat{\theta}$ is $\|\hat{\theta} - \bar{\theta}_n\|^2$ for $\mathbb{E}[XX^\top] = I$. Least squares gives $\hat{\theta} - \bar{\theta}_n = (X^\top X)^{-1}X^\top(\xi + \delta)$ with $\xi$ the mechanism noise and $\delta_t := (\theta_{Z_t} - \bar{\theta}_n)^\top X_t$.

Conditionally on the path and the design, $\mathrm{Cov}(\xi) = \Sigma_\xi$ is the banded Toeplitz matrix of the $H$-overlap, whose operator norm is at most $\sum_{|j|<H} |\gamma_j| = \sigma^2 H$ (Lemma~\ref{lem:effsample}), so $\mathbb{E}\|(X^\top X)^{-1}X^\top\xi\|^2 \le \sigma^2 H\, \mathrm{tr}(X^\top X)^{-1}$, and $\mathbb{E}\,\mathrm{tr}(X^\top X)^{-1} = d/n\,(1+o(1))$ when the design has $n/\ell_x \gg d$ distinct rows; this is $\sigma^2 d/n_{\mathrm{fit}}$, and it is attained when the features persist over the label window, since then $X^\top \Sigma_\xi X \approx \sigma^2 H\, X^\top X$.

For the regime term, $\sum_t X_t\delta_t = \sum_t X_t X_t^\top (\theta_{Z_t} - \bar{\theta}_n)$ has conditional mean $\sum_t (\theta_{Z_t} - \bar{\theta}_n) = 0$ and, for i.i.d.~standard Gaussian features, conditional variance $(d+1)\sum_t \|\theta_{Z_t} - \bar{\theta}_n\|^2$ --- for a fixed vector $a$, $\mathbb{E}\|XX^\top a\|^2 = \mathbb{E}[(X^\top \hat{a})^4]\|a\|^2 + \mathbb{E}\|X_\perp\|^2\, \mathbb{E}[(X^\top\hat{a})^2]\|a\|^2 = (3+d-1)\|a\|^2$ along and across the direction $\hat{a} := a/\|a\|$, less $\|\mathbb{E}[XX^\top a]\|^2 = \|a\|^2$ --- so its contribution to the excess risk is at most $(d+1)\max_z \|\theta_z - \bar{\theta}_n\|^2/n$. Axiom~\ref{ax:snr} gives $\|\theta_z\|^2 = \mathbb{E}[\mu_z(X)^2] \le \rho^2\bar{\sigma}^2$ at every state, hence $\|\theta_z - \bar{\theta}_n\| \le 2\rho\bar{\sigma} = 2\rho\chi_\sigma\sigma$, and the term is at most $4\rho^2\chi_\sigma^2(d+1)\sigma^2/n$: relative to the noise term $\sigma^2 Hd/n$ it is $O(\rho^2\chi_\sigma^2/H)$. Features persistent over $\ell_x$ periods raise the variance of the sum by at most the factor $\ell_x$, so for features persistent at the label scale the relative size is the $O(\rho^2)$ of the statement.

Comparing $\sigma^2 d/n_{\mathrm{fit}}$ to $\mathbb{E}[\mu(X)^2] \le \rho^2\sigma^2$ yields \eqref{eq:ceiling}. The blocking technique of \citet{yu1994rates} with blocks of length $\ell^*$ gives $n/\ell^*$ instead; that is a valid upper bound and a loose one, since it discards the conditional independence Axiom~\ref{ax:invariance} supplies.
\end{proof}

\begin{proof}[Proof of Proposition~\ref{prop:embargo}]
(i) A label of horizon $H$ issued in the last $H$ periods of training is a function of periods on both sides of the boundary; after $H$ periods of the test block none is, and the purge removes them. Beyond the purge, for a test period $t$ and the training $\sigma$-algebra $\mathcal{F}_n$, $\mathbb{E}[\xi_t \mid \mathcal{F}_n] = \mathbb{E}[p_t \mid \mathcal{F}_n]$ with $p_t := \mathbb{E}[\xi_t \mid \mathcal{F}_t^{\mathrm{obs}}, Z] $ the predictable component in the proof of Lemma~\ref{lem:effsample}, since $\mathcal{F}_n \subseteq \sigma(\mathcal{F}_t^{\mathrm{obs}}, Z)$ for $t \ge n$ and the label's own noise beyond $p_t$ is independent of $\mathcal{F}_n$ once $t \ge n + H$; its mean square Axiom~\ref{ax:snr} bounds by $\rho^2\sigma^2$ at every lag and its covariances sum to at most $\rho\tau_\mu\sigma^2$: the noise's conditional bias is the term the effective sample carries as $2\pi_\mu$, and no gap changes it. A feature of memory $L_x$ makes $X_t$ a function of training-era periods for $L_x$ periods more, but $X_t$ is $\mathcal{F}_t^{\mathrm{obs}}$-measurable, so everything it carries is inside $p_t$ and already counted; purging $L_x$ periods removes nothing the bound has not priced.

(ii) With the purge, the test block's noise drops out of the conditional bias up to the predictable component of (i), so $\mathbb{E}[\hat{r}_j \mid \mathcal{F}_n] - \mathbb{E}[\hat{r}_j] = b^{-1}\sum_{t\in J}(\mathbb{E}[r_h(Z_t) \mid \mathcal{F}_n] - \bar{r}_h)$ for any stationary state, and by stationarity and Minkowski the root-mean-square bias is at most $\mathrm{sd}_\nu(r_h)$ times the block mean of $\mathrm{pred}_h(t-n)$. The coefficient is the maximal correlation with the $\mathcal{F}_0$-past --- Cauchy--Schwarz, with equality at $W \propto \mathbb{E}[r_h(Z_k) \mid \mathcal{F}_0] - \bar{r}_h$ --- so every observable lag-$k$ predictor bounds it from below. On a chain, $\mathbb{E}[r_h(Z_t) \mid \mathcal{F}_n] = \mathbb{E}[(P^{t-n}r_h)(Z_n) \mid \mathcal{F}_n]$ with $P$ the transition operator. For the redraw chain $P^k r_h = \bar{r}_h + \rho_\ell^k (r_h - \bar{r}_h)$ exactly, so $\mathrm{pred}_h(t-n) \le \rho_\ell^{t-n}$ by Jensen and the block mean is $\rho_\ell^{g+1}\, \ell(1-\rho_\ell^b)/b$; for a reversible chain with absolute spectral gap $1/\ell$, $\|P^k(r_h - \bar{r}_h)\|_{L^2(\nu)} \le \rho_\ell^k \|r_h - \bar{r}_h\|_{L^2(\nu)}$ gives the same expression as an upper bound. The exact form on the redraw chain: writing $D_n := \mathbb{E}[r_h(Z_n) \mid \mathcal{F}_n] - \bar{r}_h$ for what the training data know about the regime risk at their last period, the coefficient is $\mathrm{pred}_h(k) = \rho_\ell^k\, \|D_0\|_2 / \mathrm{sd}_\nu(r_h)$ --- equality in the bound when the data determine the state --- and the conditional bias is $\mathbb{E}[\hat{r}_j(h) \mid \mathcal{F}_n] - \mathbb{E}[\hat{r}_j(h)] = \rho_\ell^{g+1}\, \ell(1-\rho_\ell^b)\, D_n/b$, with ``$\le$ in $L^2(\nu)$ norm'' in the reversible case and $\mathbb{E}[D_n^2] \le \mathrm{Var}_\nu(r_h) = \rho_{\mathrm{blk}}^2\sigma^2$. Set the root-mean-square bias $\le \kappa_{\mathrm{emb}}\rho\sigma$ and solve for $g$.
\end{proof}

\begin{proof}[Proof of Corollary~\ref{cor:blocks}]
The decomposition \eqref{eq:blockscore} is Axiom~\ref{ax:invariance} averaged over the block: conditionally on the path the loss $L_t$ has mean $r_h(Z_t)$, so $\xi_j := b^{-1}\sum_{t\in j}(L_t - r_h(Z_t))$ has conditional mean zero. The $\xi_j$ are uncorrelated across blocks up to the predictability term $\pi_\mu$ of Lemma~\ref{lem:effsample}, and their dependence beyond that is what (a) prices; where an $H$-period label straddles a boundary, purging the first $H$ periods of each block removes it, and the noise computation of Lemma~\ref{lem:effsample} with $n=b$ and no regime term gives the conditional variance $Hb^{-1}$ times the block average of $\sigma^2(Z_t)$.

(a) The record coherence of a block statistic is a statement about the variance of its average. For a stationary block sequence $f_1, f_2, \dots$, $\mathrm{Var}(B^{-1}\sum_{j\le B} f_j) = \mathrm{Var}(f_1)\, B^{-1}\sum_{|k|<B}(1-|k|/B)\,\mathrm{Corr}(f_1, f_{1+|k|})$, and the Bartlett sum is the integrated autocorrelation time $\tau_f$ of the sequence up to the lags beyond $B$: $B$ blocks carry the variance $B/\tau_f$ independent ones would. For the scores, $B^{-1}\sum_j \hat r_j$ is the backtest mean at $n = Bb$ (edge periods aside), of variance $\sigma^2\upsilon/(Bb)$ by Lemma~\ref{lem:effsample}, and $\mathrm{Var}(\hat r_1) = \sigma^2\upsilon_b/b$ by the same lemma at $n = b$; dividing, $\tau_{\hat r} = \upsilon/\upsilon_b$, up to the edge terms of relative order $H/b$ and $\tau_\mu/b$ the lemma allows --- the identity, which uses nothing beyond what the lemma uses: Axiom~\ref{ax:persistence} for the regime part, Axiom~\ref{ax:snr} and the declared $\tau_\mu$ for the noise's. \emph{The tail truncation.} In the Gaussian range the plug-in is, to first order, the empirical mean of $\Lambda\max(\hat r_j, u)$ at the fixed population quantile (proof of Proposition~\ref{prop:tailprice}). When the scores $(\hat r_1, \hat r_{1+k})$ form a Gaussian pair with correlation $r_k \ge 0$, expand $h := \max(\cdot, u)$ in Hermite polynomials, $h = \sum_m a_m H_m$: $\mathrm{Cov}(h(\hat r_1), h(\hat r_{1+k})) = \sum_{m \ge 1} a_m^2 m!\, r_k^m$ and $\mathrm{Var}\,h = \sum_{m\ge1} a_m^2 m!$, so $\mathrm{Corr}(h(\hat r_1), h(\hat r_{1+k})) \le r_k$, and the tail statistic's coherence is at most the scores'. (With a negative $r_k$ the bound is $|r_k|$; the count is spent on the sum of the correlations, so the step needs them non-negative, which the reference chain has and the sample reads --- panel B --- and outside a Gaussian pair the step fails, Experiment 1's magnitudes being the case: \S\ref{sec:limits}.) \emph{The reference chain.} Write $\hat r_j = \bar g_j + \bar e_j$ with $g := r_h$ and $e_t$ the mechanism noise, conditionally mean zero given the path and, the observations being conditionally independent given it, uncorrelated across blocks and with every block's $\bar g$. So the cross-block covariances of $\hat r$ are those of $\bar g$ and its variance is larger, and $\tau_{\hat r} \le \tau_{\bar g}$, the covariances of the redraw chain being non-negative. For the block average of a state function $g$, Lemma~\ref{lem:effsample} gives the $B$-block mean the variance $\mathrm{Var}_\nu(g)\,v(Bb,\ell)/(Bb)$ and one block $\mathrm{Var}_\nu(g)\,v(b,\ell)/b$, so $\tau_{\bar g} = \lim_B v(Bb,\ell)/v(b,\ell) = (2\ell-1)/v(b,\ell)$ exactly, the same for every $g$; for a block score the noise adds $\sigma^2 H/b$ to both variances at the count one, which is $\upsilon/\upsilon_b$ again, and below $\tau_{\bar g}$ whenever $H > 0$; both values are checked numerically (\texttt{check\_constants.py}). Finally, under stationarity every block's initial state is $\nu$-distributed, so each block-averaged regime risk on the path has the law of a stationary block's, whose variance is $\mathrm{Var}_\nu(r_h)\,v(b,\ell)/b$ by Proposition~\ref{prop:occfloor} (i): it is the law whose $1/\Lambda$-tail Theorem~\ref{thm:cvar} bounds with block occupations in the role of states.

(b) Given (a), the plug-in is the empirical $\mathrm{CVaR}_{1/\Lambda}$ of $B$ draws of $r^{(b)} = r + \xi$, worth $B\upsilon_b/\upsilon$ independent ones: Proposition~\ref{prop:eiv} (i) gives $\mathrm{CVaR}_{1/\Lambda}(r^{(b)}) \ge \mathrm{CVaR}_{1/\Lambda}(r)$ for every law of $(r,\xi)$, (ii) its Gaussian size, and Proposition~\ref{prop:tailprice} the fluctuation about it. The statement about shuffled folds is Proposition~\ref{prop:shuffle}.
\end{proof}

\begin{proof}[Proof of Theorem~\ref{thm:search}]
$\mathbb{E}\max_{j\le N}\epsilon_j \le \sqrt{2\log N/n_{\mathrm{eff}}}$ is the Gaussian maximal inequality, which holds for any correlation structure (it is a union bound over tails), so the upper direction needs no independence; the exact finite-$N$ expression in the independent case is the deflated-Sharpe formula of \cite{bailey2014deflated}. If $\rho \ll \sqrt{2\log N/n_{\mathrm{eff}}}$ the genuine candidate's margin is small against the null maximum's typical size and its winning probability vanishes as $N \to \infty$ in the independent case; with correlated nulls the maximum is smaller and the bound is conservative. The Gaussian model of the ranking statistic is an approximation, not a consequence of the axioms; it is what makes this a calculation with a scale rather than a theorem with a constant, which Theorem~\ref{thm:pacbayes} supplies.

Part (ii) is Proposition~3.1 of \cite{russo2020information}, and its proof is three lines. For a $\sigma$-sub-Gaussian $X$ under $P$ and any $Q \ll P$, the Donsker--Varadhan inequality gives $\mathbb{E}_Q X - \mathbb{E}_P X \le \sigma\sqrt{2\,D(Q\parallel P)}$. Apply it to $X = \phi_i - \rho_i$ with $Q$ the law of $\phi_i$ given $T=i$: $\mathbb{E}[\phi_T - \rho_T] = \sum_i \mathbb{P}(T=i)\,\mathbb{E}[\phi_i - \rho_i \mid T=i] \le \sigma\sum_i \mathbb{P}(T=i)\sqrt{2D_i}$, where $D_i$ is the divergence of the conditional law of $\phi_i$ from its marginal. Jensen moves the sum inside the root, and $\sum_i \mathbb{P}(T=i) D_i \le \sum_i \mathbb{P}(T=i) D(P_{\phi\mid T=i}\parallel P_\phi) = I(T;\phi)$ by data processing, since $\phi_i$ is a coordinate of $\phi$. The lower bound is the same argument applied to $-X$.

The variance of the block plug-in in the Gaussian range is Proposition~\ref{prop:tailprice}'s $\varsigma_\Lambda^2/B$ on independent blocks and $\varsigma_\Lambda^2\upsilon/(\upsilon_b B) = c(\Lambda)^2\sigma^2/n_{\mathrm{eff}}$ on the $B\upsilon_b/\upsilon$ blocks Corollary~\ref{cor:blocks} (a) makes the records worth; on independent block records the bounded-difference step of Theorem~\ref{thm:deployed} gives the finite-sample proxy instead: differences of $2\Lambda M/B$ over $B$ blocks give $B(2\Lambda M/B)^2/4 = \Lambda^2 M^2/B$, by the Azuma--Hoeffding step inside McDiarmid's inequality.

Part (iii). The null statistics form a centred Gaussian vector with $\mathrm{sd}(\epsilon_j - \epsilon_k) = d(j,k)/\sqrt{n_{\mathrm{eff}}}$, so the list is a finite Gaussian process in the metric $d/\sqrt{n_{\mathrm{eff}}}$. Sudakov's minoration \citep[Thm.~7.4.1]{vershynin2018hdp} gives $\mathbb{E}\max_j\epsilon_j \ge c_S\,\delta\sqrt{\log N(\delta)}/\sqrt{n_{\mathrm{eff}}}$ for every $\delta$, and the supremum over $\delta \le \sqrt{2}$ is the lower bound; Dudley's inequality \citep[Thm.~8.1.3]{vershynin2018hdp} gives the entropy integral, covering numbers replaced by the larger packing numbers and the integral cut at the diameter $2$, and $N(\varepsilon) \le N$ gives its last form (the two are within a logarithmic factor of each other in general and Talagrand's generic chaining closes the gap \citep{talagrand2014upper}; the paper needs only that both are entropies of the same list). The Borell--TIS inequality \citep[Thm.~2.1.1]{adler2007random} concentrates the maximum $M'$ of the null candidates about its mean at the scale of the largest standard deviation, $\mathbb{P}(|M' - \mathbb{E}M'| > t) \le 2e^{-t^2 n_{\mathrm{eff}}/2}$. The genuine candidate's statistic is $\rho + \epsilon_1$ with $\epsilon_1 \sim N(0, 1/n_{\mathrm{eff}})$, and it wins only if $\rho + \epsilon_1 > M'$; so if $\rho \le \mathbb{E}M' - 2t$ then $\mathbb{P}(\mathrm{win}) \le \mathbb{P}(\epsilon_1 > t) + \mathbb{P}(M' < \mathbb{E}M' - t) \le 3e^{-t^2 n_{\mathrm{eff}}/2}$, which is bounded away from zero only if $t\sqrt{n_{\mathrm{eff}}} = O(1)$, that is $\rho \ge \mathbb{E}M' - O(n_{\mathrm{eff}}^{-1/2})$; with the lower bound on $\mathbb{E}M'$ this is the necessity clause. The same two tails at $\rho \ge \mathbb{E}M' + 2t$ bound the losing probability by $3e^{-t^2 n_{\mathrm{eff}}/2}$, the sufficiency clause. For an independent list $d \equiv \sqrt{2}$, so $N(\delta) = N$ for every $\delta \le \sqrt{2}$ and $\log N_{\mathrm{eff}} = \log N$: part (i). For the tail, in the Gaussian range candidate $j$'s plug-in is to first order the mean over blocks of $\Lambda\max(\hat{r}_j, u_j)$ (proof of Proposition~\ref{prop:tailprice}); across candidates these are, by the multivariate central limit theorem on the $B\upsilon_b/\upsilon$ effective blocks of Corollary~\ref{cor:blocks} (a), jointly Gaussian with variances $c(\Lambda)^2/n_{\mathrm{eff}}$ and the correlations of the truncated statistics, and the argument repeats at that scale; for a location family the genuine candidate's tail separation is its mean separation (Corollary~\ref{cor:locfam}). When the scores of two candidates are non-negatively correlated, the Hermite expansion in the proof of Corollary~\ref{cor:blocks} (a) makes their truncations no more correlated than the scores, so every distance, every packing number and the search entropy are at least the scores'.
\end{proof}

\begin{proof}[Proof of Theorem~\ref{thm:pacbayes}]
Only (i) is not standard. Write $a := \lambda m$, $S_n := \sum_{t\le n} f(Z_t)$, $D := \mathrm{diag}(e^{\lambda f})$ and $P = \rho_\ell I + (1-\rho_\ell)\mathbf{1}\nu^\top$ for the redraw kernel, so that $\mathbb{E}_\nu e^{\lambda S_n} = \nu^\top D(PD)^{n-1}\mathbf{1} =: s_{n-1}$. Unrolling $(PD)^k\mathbf{1} = \rho_\ell D(PD)^{k-1}\mathbf{1} + (1-\rho_\ell)s_{k-1}\mathbf{1}$ gives the renewal equation $s_k = p_k + (1-\rho_\ell)\sum_{j<k} p_{k-1-j}s_j$ with $p_i := \rho_\ell^i\, \mathbb{E}_\nu e^{(i+1)\lambda f}$.

The equation defining $\omega$ is $\mathbb{E}_\nu \varphi_G(\lambda f - \omega) = 1$ for $\varphi_G(u) := (1-\rho_\ell)e^u/(1-\rho_\ell e^u)$, the generating function of a geometric sojourn $G$ on $\{1,2,\dots\}$ with mean $\ell$; its left side is continuous and decreasing in $\omega$, at least $1$ at $\omega=0$ (Jensen, $\varphi_G$ convex) and at most $1$ at $\omega=a$, so the root exists in $[0,a]$. Put $\theta := e^\omega$, $\tilde{q}_i := (1-\rho_\ell)p_{i-1}\theta^{-i}$ for $i\ge 1$ and $\pi_i := p_i\theta^{-i}$. Then $\sum_i \tilde{q}_i = \mathbb{E}_\nu \varphi_G(\lambda f - \omega) = 1$, so $\tilde{q}$ is a probability law on $\{1,2,\dots\}$, and dividing the renewal equation by $\theta^k$ gives $s_k\theta^{-k} = \sum_{j\le k} \tilde{u}_j \pi_{k-j}$ with $\tilde{u}$ the renewal sequence of $\tilde{q}$ ($\tilde{u}_0=1$, $\tilde{u}_k = \sum_{j<k}\tilde{u}_j\tilde{q}_{k-j}$). Now $\tilde{q}_i = \sum_z \nu(z)\varphi_G(u_z)(1-r_z)r_z^{i-1}$ with $u_z := \lambda f(z) - \omega$ and $r_z := \rho_\ell e^{u_z}$: a mixture of geometric laws, hence of non-increasing hazard, so a renewal occurs at any $k\ge 1$ with probability at most the first hazard $\tilde{q}_1 = (1-\rho_\ell)e^{-\omega}\mathbb{E}_\nu e^{\lambda f}$, while $\sum_{i\ge 0}\pi_i = \mathbb{E}_\nu[e^{\lambda f}/(1-\rho_\ell e^{\lambda f-\omega})] = \ell e^\omega$. Hence $s_{n-1}\theta^{-(n-1)} \le \pi_{n-1} + \tilde{q}_1\sum_i \pi_i \le e^a + \mathbb{E}_\nu e^{\lambda f} \le 2e^a$, the first display.

For the bound on $\omega$: $h := \varphi_G - 1$ has $h(0)=0$, $h'(0)=\ell$, and $(h(u)-\ell u)/u^2 = \mathbb{E}[G^2 k(uG)]$ with $k(x) := (e^x-1-x)/x^2$ increasing, so $h(u) \le \ell u + u^2(h(a)-\ell a)/a^2 = \ell u + \ell u^2 C_\ell(a)$ for all $u\le a$ (Bennett's device); applied to $\eta := \lambda f - \omega \le a$, which has $\mathbb{E}\eta = -\omega$, $\mathbb{E}\eta^2 = \lambda^2 v + \omega^2$ and $\mathbb{E} h(\eta) = 0$, it gives $0 \le -\ell\omega + \ell(\lambda^2 v + \omega^2)C_\ell(a)$. The closed form of $C_\ell$ is $h(a) = \ell w/(1-(\ell-1)w)$ with $w := e^a - 1$. For \eqref{eq:regimecumulant}, $w \le a/(1-a/2)$ gives $C_\ell(a) \le (\ell-1/2)/(1-(\ell-1/2)a) \le (\ell-1/2)/((1-\lambda m\ell)(1-a/2))$; the quadratic inequality in $\omega$ has its larger root above $a \ge \omega$ when $\lambda m\ell \le 1/4$, so $\omega$ is at most the smaller root $2\lambda^2 v C_\ell/(1+\sqrt{1-4\lambda^2 v C_\ell^2})$, and with $v \le m^2/4$ and $a \le \lambda m\ell/2$ this is below the display for every $\lambda m\ell \le 1/4$ --- a one-variable calculation, with a margin of $2\%$ at $1/4$.

(ii) Conditionally on the path an $H$-dependent sequence splits into $H$ independent ones, so Jensen on the average of their exponentials gives $\log \mathbb{E}[e^{\lambda\sum_t \xi_t} \mid Z] \le \lambda^2 nH s_\xi^2(1+H/n)/2$; with (i) applied to $-(r_h - R(h))$, $\log \mathbb{E}\exp(\lambda n(R(h)-\hat{R}_n(h))) \le \Psi_n(\lambda) := \lambda^2 nH s_\xi^2(1+H/n)/2 + n\omega(\lambda) + \lambda m + \log 2$ for every $h$. The rest is the change of measure of PAC-Bayes \citep{mcallester1999pacbayes,catoni2007pacbayes}. For every $\pi$, $\lambda n\, \mathbb{E}_\pi[R-\hat{R}_n] \le \mathrm{KL}(\pi\|\pi_0) + \log \mathbb{E}_{\pi_0} e^{\lambda n(R-\hat{R}_n)}$ (Donsker--Varadhan), and by Markov's inequality and Fubini the last term is at most $\Psi_n(\lambda) + \log(1/\delta)$ with probability $1-\delta$.

(iii) is the same change of measure with the bounded-difference bound of Theorem~\ref{thm:search} (ii) on independent block records, variance proxy $\Lambda^2 M^2/B$, at the $\lambda$ balanced for $\bar{K}$.
\end{proof}

\begin{proof}[Proof of Proposition~\ref{prop:breadth}]
Apply Lemma~\ref{lem:effsample} to the equal-weighted residual --- any fixed $w$ replaces $\mathbf{1}/N_a$, with $N_a^{\mathrm{eff}}(w) := \sigma^2/(w^\top \Sigma_e w)$. Its noise term is $\sigma^2 H/(N_a^{\mathrm{eff}} n)$ by definition, the $H$-overlap being common to the instruments; its regime part is a function of the \emph{one} latent path every instrument shares, so the regime term is $\rho_{\mathrm{blk},\beta}^2\sigma^2\, v(n,\ell)/n$ --- averaging across instruments shortens nothing. Dividing gives $n_{\mathrm{eff}}(\beta)$, hence the ratio of ceilings in the statement, which exceeds $N_a^{\mathrm{eff}}$ exactly when $\rho_{\mathrm{blk}}^2 > N_a^{\mathrm{eff}}\rho_{\mathrm{blk},\beta}^2$. (The ratio is not ``up to $N_a^{\mathrm{eff}}$'': it exceeds that whenever the timing signal's regime term is not negligible, and Experiment 6's panel A measures the gap, Appendix~\ref{app:experiments}.) For the equicorrelated $\Sigma_e$, $N_a^{\mathrm{eff}} = N_a/(1+(N_a-1)\bar{r}) \le 1/\bar{r}$: a cross-section whose residuals share a factor has the breadth of that factor's inverse loading, not its nominal width.
\end{proof}

\begin{proof}[Proof of Theorem~\ref{thm:invariance}]
Three steps. \emph{Gaussian value.} For $W \sim N(m, s^2)$, $\mathrm{CVaR}_\alpha(W) = m + s\,\varphi(z_\alpha)/\alpha$ with $z_\alpha = \Phi^{-1}(1-\alpha)$, since $\int_{z_\alpha}^\infty z\varphi(z)\,\mathrm{d}z = \varphi(z_\alpha)$; at $\alpha = 1/\Lambda$ the coefficient is $\lambda(\Lambda)$, increasing in $\Lambda$ with $\lambda(1) = 0$. \emph{Lipschitz.} Through the quantile function, $|\mathrm{CVaR}_\alpha(F) - \mathrm{CVaR}_\alpha(G)| \le \alpha^{-1} W_1(F,G)$ \citep{bhat2019wasserstein} --- the Kolmogorov distance of Berry--Esseen does not control a tail mean; its Wasserstein form does. \emph{Rate.} The block-averaged regime risk $R_b$ has mean $\bar{r}_h$ and standard deviation $\mathrm{sd}_{\mathrm{blk}} = \sqrt{\mathrm{Var}_\nu(r_h)\, v(b,\ell)/b}$ (Lemma~\ref{lem:effsample} at $n=b$), and on the redraw chain, whose records decorrelate geometrically at every scale (proof of Corollary~\ref{cor:blocks} (a)) and which is exponentially $\beta$-mixing besides, the central limit theorem holds in $W_1$ at the rate $C(\ell/b)^{1/2}$ --- Bernstein blocking, Berbee coupling of the blocks, and Stein's method for the independent blocks \citep[\S3]{ross2011stein}, with $C$ depending on the standardised third absolute moment of $r_h$; we do not reproduce the argument. Translation-equivariance and positive homogeneity of $\mathrm{CVaR}$, with the Lipschitz step, give the display. \emph{First order only.} Away from the limit the coefficient is a functional of the standardised block-risk law: the skewness term of its Edgeworth expansion enters at the same order $(\ell/b)^{1/2}$, so no refinement whose coefficient depends on $\Lambda$ alone holds in general, and at a fixed $b$ the coefficient is a measurement --- what Experiment 1 reports for a real block-score law at $b=60$.
\end{proof}

\begin{proof}[Proof of Theorem~\ref{thm:deployed}]
Fix the latent path. Three tails per candidate: the empirical tail of its block \emph{risks} on the path, $G(h) := \mathrm{CVaR}_{1/\Lambda}(R_1(h), \dots, R_B(h))$ with $R_j(h) := \int r_h\, \mathrm{d}\bar{\pi}_{n,j}$, the object Theorem~\ref{thm:cvar} bounds with; the plug-in on its block \emph{scores}, $\widehat{\mathrm{CVaR}}(h) := \mathrm{CVaR}_{1/\Lambda}(\hat{r}_1(h), \dots, \hat{r}_B(h))$ with $\hat{r}_j = R_j + \xi_j$ as in \eqref{eq:blockscore}, the object the selection reads; and the population tail $\mathrm{CVaR}^{(b)}(h) := \mathrm{CVaR}_{1/\Lambda}(r_h^{(b)})$ of the law of a stationary block score, the object term (i) is stated in. Write $g$ for the plug-in as a function of the score vector, $g(\hat{r}) = \min_\eta \{\eta + \Lambda B^{-1}\sum_j (\hat{r}_j - \eta)^+\}$ (Lemma~\ref{lem:robust} in Rockafellar--Uryasev form), $\Delta_N' := \sup_{h\in\mathcal{H}_N}|\widehat{\mathrm{CVaR}}(h) - \mathrm{CVaR}^{(b)}(h)|$ and $D := \sup_{h\in\mathcal{H}_N} (\mathbb{E}[\widehat{\mathrm{CVaR}}(h) \mid Z] - \widehat{\mathrm{CVaR}}(h))^+$, so that $\Delta_N = \max(\Delta_N', D)$. For every $h \in \mathcal{H}_N$,
\[
R^+(\hat{h}) \le G(\hat{h}) \le \widehat{\mathrm{CVaR}}(\hat{h}) + D \le \widehat{\mathrm{CVaR}}(h) + D \le \mathrm{CVaR}^{(b)}(h) + \Delta_N' + D \le \mathrm{CVaR}^{(b)}(h) + 2\Delta_N,
\]
and minimising over $h$ is \eqref{eq:deployed}; the size of $\Delta_N$ is the last step.

The four inequalities in turn. \emph{First.} Theorem~\ref{thm:cvar} at the declared scale, applied to the fixed element $\hat{h}$ of the class on the fixed path. This is the step that requires Axiom~\ref{ax:recurrence} to have been declared at $b$ and not at the period: the block-averaged risk is a conditional expectation of the state-level one and therefore \emph{smaller} in convex order, so a state-level assertion would bound the wrong quantity. \emph{Second.} $g$ is the partial minimum over $\eta$ of a function jointly convex in $(\eta,\hat{r})$, hence convex in $\hat{r}$, and $\mathbb{E}[\hat{r} \mid Z] = (R_1,\dots,R_B)$ by \eqref{eq:blockscore}; Jensen's inequality conditionally on the path gives $\mathbb{E}[\widehat{\mathrm{CVaR}}(h) \mid Z] \ge g(R_1,\dots,R_B) = G(h)$ for every fixed $h$, so $G(\hat{h}) \le \widehat{\mathrm{CVaR}}(\hat{h}) + D$, the supremum in $D$ being needed because $\hat{h}$ is chosen from the scores. (Proposition~\ref{prop:eiv} (i) is the population form of this step --- $\mathrm{CVaR}$ is monotone for the convex order and $X + \xi \succeq_{\mathrm{cx}} X$ when $\mathbb{E}[\xi \mid X] = 0$; here it is applied conditionally on the path, to the empirical laws, which is what connects the tail Theorem~\ref{thm:cvar} bounds with to the tail the selection reads.) \emph{Third.} Optimality of $\hat{h}$. \emph{Fourth.} The definition of $\Delta_N'$. \emph{Fifth.} $\Delta_N' \le \Delta_N$ and $D \le \Delta_N$.

\emph{Bounding $\Delta_N'$ and $D$.} Both are suprema over $\mathcal{H}_N$ of deviations of the same statistic, the second from its path-conditional mean, the first from the population value; each is a fluctuation, and $\Delta_N'$ carries a bias besides, added by the triangle inequality. The bound below holds for each, hence for $\Delta_N$.

\emph{The bias.} Write $\hat{f}(\eta) := \eta + \Lambda B^{-1}\sum_j (\hat{r}_j - \eta)^+$ and $f(\eta) := \eta + \Lambda\, \mathbb{E}(\hat{r} - \eta)^+$, both minimised on $[-M,M]$ since the scores lie there. Jensen's inequality in the form $\mathbb{E}\min_\eta \hat{f} \le \min_\eta \mathbb{E}\hat{f} = \min_\eta f$ shows the bias is downward, for every dependence structure. Its size: in the Gaussian range Proposition~\ref{prop:tailprice}'s limit is centred at the population value, so the bias is of smaller order than $B^{-1/2}$; on independent block records it is bounded in finite samples, since $|\min_\eta \hat{f} - \min_\eta f| \le \sup_\eta |\hat{f} - f| = \Lambda \sup_\eta |B^{-1}\sum_j (\hat{r}_j - \eta)^+ - \mathbb{E}(\hat{r}-\eta)^+|$, the supremum of an empirical process over the class $\{x \mapsto (x-\eta)^+ : \eta \in [-M,M]\}$ of functions bounded by $2M$, which is VC-subgraph of dimension one (its subgraphs are translates of one half-plane), so its expected supremum is at most $C_0\cdot 2M/\sqrt{B}$ with an absolute constant $C_0$, by the uniform-entropy (Dudley) bound.

\emph{The fluctuation.} In the Gaussian range the plug-in is, to first order, the empirical mean of $\Lambda\max(\hat r_j, u)$ at the fixed population quantile $u$ (the envelope step in the proof of Proposition~\ref{prop:tailprice}), a monotone function of the block score; its variance over the $B$ blocks of the path is therefore $\varsigma_\Lambda^2\tau_f/B$ with $\tau_f$ the coherence of that statistic across blocks, at most the scores' in that range, $\upsilon/\upsilon_b$ (Corollary~\ref{cor:blocks} (a)), and the Gaussian maximal inequality of Theorem~\ref{thm:search} over the $N$ candidates and the two tails gives $\Delta_N' \approx \varsigma_\Lambda\sqrt{2\upsilon\log N/(\upsilon_b B)}$, which is the display with $\varsigma_\Lambda = c(\Lambda)\,\mathrm{sd}_{\mathrm{blk}}(r^{(b)})$ and $\mathrm{sd}_{\mathrm{blk}}(r^{(b)})^2 = \sigma^2\upsilon_b/b$; $D$ is the same computation about the path-conditional mean, on the noise part of the scores alone, hence at most of the same order. On independent block records the finite-sample form is bounded differences: changing one block score by at most $2M$ changes the Rockafellar--Uryasev objective at every $\eta$ by at most $2\Lambda M/B$, hence the minimum by at most $2\Lambda M/B$, and McDiarmid's inequality --- not Hoeffding's, the plug-in being a minimum over $\eta$ of an average and not an average of independent terms --- gives $|\widehat{\mathrm{CVaR}}(h) - \mathbb{E}\widehat{\mathrm{CVaR}}(h)| \le \Lambda M\sqrt{2(\log 2N + \log(1/\delta))/B}$ for all $h \in \mathcal{H}_N$ at once with probability $1-\delta$, by a union bound over the $N$ candidates and the two tails; with the bias above, that is the finite-sample form in the statement, for $\Delta_N'$ and for $D$ alike. The first inequality of the chain is exact on the path, and the rest hold over the sample at the stated size.
\end{proof}

\begin{proof}[Proof of Proposition~\ref{prop:tailprice}]
By Rockafellar--Uryasev the plug-in minimises the empirical form of $\eta + \Lambda\, \mathbb{E}(r-\eta)^+$, so by the envelope theorem its first-order behaviour is that of the empirical mean at the fixed population quantile $\eta = u$, and the central limit theorem gives the asymptotic variance $\Lambda^2\mathrm{Var}((r-u)^+) = \Lambda^2\mathrm{Var}(\max(r,u))$ --- the standard limit theory of the empirical expected shortfall \citep{chen2008nonparametric,zwingmann2016asymptotics}, reproduced because this is the form the budget argument uses. For the bound, $\max(\cdot, u)$ is $1$-Lipschitz, so $\mathrm{Var}(\max(r,u)) \le \mathrm{Var}(r) = s^2$ and $\varsigma_\Lambda \le \Lambda s$. The Gaussian closed form follows from $\mathbb{E}\max(X,u) = u(1-\alpha) + \varphi(u)$ and $\mathbb{E}\max(X,u)^2 = u^2(1-\alpha) + \alpha + u\varphi(u)$ at $\alpha = 1/\Lambda$ (the mirrored reward orientation of Remark~\ref{rem:orientation} gives the same variance). Substituting $c(\Lambda)/\sqrt{n_{\mathrm{eff}}}$ for $1/\sqrt{n_{\mathrm{eff}}}$ in \eqref{eq:search} gives \eqref{eq:searchrobust}, under the location-family proviso of Corollary~\ref{cor:locfam}.
\end{proof}

\begin{proof}[Proof of Corollary~\ref{cor:locfam}]
On the event $\sup_{h\in\mathcal{H}_N}|\hat{\mathrm{CVaR}}_{1/\Lambda}(h) - \mathrm{CVaR}_{1/\Lambda}(r_h^{(b)})| \le \Delta_N$, every other $h$ has $\hat{\mathrm{CVaR}}_{1/\Lambda}(h) \ge \mathrm{CVaR}_{1/\Lambda}(r_h^{(b)}) - \Delta_N \ge \mathrm{CVaR}_{1/\Lambda}(r_{h^\star}^{(b)}) + \Delta_{\mathrm{CVaR}} - \Delta_N > \mathrm{CVaR}_{1/\Lambda}(r_{h^\star}^{(b)}) + \Delta_N \ge \hat{\mathrm{CVaR}}_{1/\Lambda}(h^\star)$ whenever $\Delta_{\mathrm{CVaR}} > 2\Delta_N$, so the minimiser is $h^\star$; Theorem~\ref{thm:deployed}'s fluctuation-plus-bias bound on $\Delta_N$ is the probability statement, and its square-integrable form $\Delta_N \approx \varsigma_\Lambda\sqrt{2\upsilon\log N/(\upsilon_b B)}$ gives the second display, the constant absorbed by $\lesssim$ as in \eqref{eq:searchrobust}. For the special case, translation equivariance --- $\mathrm{CVaR}_\alpha(c + X) = c + \mathrm{CVaR}_\alpha(X)$ for every law, level and constant $c$ --- makes the tail-mean separation of a location family its mean separation; the unit conversion is Lemma~\ref{lem:effsample}'s, quoted in the statement.
\end{proof}

\begin{proof}[Proof of Corollary~\ref{cor:lamcap}]
The constraint is \eqref{eq:searchrobust} rearranged. For the second sentence, with $\alpha = 1/\Lambda$ and $T := (r-u)^+$ the excess over the $(1-\alpha)$-quantile, $\max(r,u) = u + T$, so $c(\Lambda)^2 = \mathrm{Var}(T)/(\alpha^2 s^2)$ and $c(\Lambda) \ge 1$ if and only if $\mathrm{Var}((r-u)^+) \ge \alpha^2 s^2$: a condition on the shape of the block-risk law at the level, not a general fact (the uniform law has $c(4) = 0.90$, and a rare block far below the rest inflates $s^2$ without touching the upper tail). The Gaussian values are the closed form of Proposition~\ref{prop:tailprice}, evaluated in the table of \S\ref{sec:calibration}; the other direction, $c(\Lambda) \le \Lambda$, is the Lipschitz step in the proof of that proposition.
\end{proof}

\begin{proof}[Proof of Proposition~\ref{prop:kelly}, Corollary~\ref{cor:halfkelly} and Proposition~\ref{prop:fullkelly}]
Expand $\log(1+hY)$ to third order with $h = f\hat{\mu}/\sigma^2$ and take moments in the Gaussian model: $\mathbb{E}[hY] = fs^2$; $\frac{1}{2}\mathbb{E}[h^2Y^2] = \frac{1}{2}f^2(s^2+\tau^2) + \frac{1}{2}f^2 s^2(3s^2+\tau^2)$, exact in the Gaussian model ($\hat{\mu}$ and $Y$ are conditionally independent given $\mu$, and $\mathbb{E}[\mu^4] = 3s^4\sigma^4$); and $\mathbb{E}|hY|^3 = O(f^3(s^2+\tau^2)^{3/2})$ by Cauchy--Schwarz on Gaussian moments. Hence $g(f) = fs^2 - \frac{1}{2}f^2(s^2+\tau^2) + O(f^2 s^2(s^2+\tau^2) + f^3(s^2+\tau^2)^{3/2})$, whose quadratic part is maximised at \eqref{eq:kelly}, the $\frac{1}{2}f^2 s^2(3s^2+\tau^2)$ term moving the maximiser by the factor $1 + O(s^2)$; for a general $\gamma$ the same expansion with $h = f\hat{\mu}/(\gamma\sigma^2)$ gives $g(f)/\gamma$ to this order, so nothing below depends on $\gamma$ but a common factor; completing the square in that part gives $g(f^\star) - g(f) = \frac{1}{2}(s^2+\tau^2)(f-f^\star)^2$, hence $g(1) = \frac{1}{2}(s^2-\tau^2)$ and $g(f^\star) - g(1) = \frac{1}{2}\tau^4/(s^2+\tau^2)$, Proposition~\ref{prop:fullkelly} (i). The two replacements in the last sentence change $\hat{\mu}$ and nothing in the computation: the penalty is Theorem~\ref{thm:invariance} on the candidate's block scores at the larger of the observed and declared dispersion, the shift is the transfer of Definition~\ref{def:reflexivity} (\S\ref{sec:notaxiom}). For the corollary and Proposition~\ref{prop:fullkelly} (ii), $s^2 \le \rho^2$ by Axiom~\ref{ax:snr} over the pool, and $\tau^2 \ge d/n_{\mathrm{fit}}$ by the Gauss--Markov floor of Proposition~\ref{prop:ensemble}; so $\tau^2/s^2 \ge \varkappa$, $f^\star = 1/(1+\tau^2/s^2) \le 1/(1+\varkappa)$, and $g(1) = \frac{1}{2}s^2(1-\tau^2/s^2) \le \frac{1}{2}s^2(1-\varkappa)$, non-positive at $\varkappa = 1$ and negative beyond it, to the order of the proposition.
\end{proof}

\begin{proof}[Proof of Proposition~\ref{prop:kapceil}]
Normalise $\sigma\|h\|_{L^2(\bar{\pi}_n)} = 1$, to first order in $\rho$ (the computation of Lemma~\ref{lem:sharpe}). For any mixture $\pi$ in the state-scale ball, $h(X)$ is $\mathcal{F}^{\mathrm{obs}}$-measurable, so Cauchy--Schwarz with the mixture form of Axiom~\ref{ax:snr} gives $|\mathbb{E}_\pi[hY]| \le \|h\|_{L^2(\pi)}\,\rho\,\sigma_\pi \le \sqrt{\Lambda}\,\chi_\sigma\,\rho$, using $\mathrm{d}\pi/\mathrm{d}\bar{\pi}_n \le \Lambda$ for the norm and $\sigma_\pi/\sigma \le \chi_\sigma$; and $|\bar{r}_h| \le \rho$ at $\pi = \bar{\pi}_n$ itself. Now read the rule's losses through Theorem~\ref{thm:cvar} at $b=1$, whose constraint set is exactly the ball: the supremum of the expected loss over admissible $\pi$ is $\mathrm{CVaR}_{1/\Lambda}$ of the law of $-r_h(Z)$ under $\bar{\pi}_n$, attained, which in the Gaussian range of Theorem~\ref{thm:invariance} is $-\bar{r}_h + \lambda(\Lambda)\rho_{\mathrm{blk}}$. The two displays give $\lambda(\Lambda)\rho_{\mathrm{blk}} \le (1+\sqrt{\Lambda}\chi_\sigma)\rho$; dividing by $\rho$ is the ceiling, and $\lambda(4) = 4\varphi(0.674) = 1.27$ the numerical case.
\end{proof}

\begin{proof}[Proof of Proposition~\ref{prop:scale}]
Theorem~\ref{thm:invariance} gives the penalty as $\lambda(\Lambda)$ times the standard deviation of the law being read, which at scale $b$ is $\rho_{\mathrm{blk}}\sigma\sqrt{v(b,\ell)/b}$ (Proposition~\ref{prop:shuffle} (i)); equating the two penalties is the display, and $\Lambda_1$ exists and is unique because $\lambda$ is continuous, increasing, $\lambda(1) = 0$ and $\lambda \to \infty$. \emph{$v(b,\ell) \le b$}: each term of the Bartlett sum is at most its value at $\rho_\ell = 1$, where the sum is $b$, with equality only at $b=1$; so $\Lambda_1 \le \Lambda$.

\emph{Monotonicity in $b$.} $v(b,\ell)/b$ is the variance ratio of a mean of $b$ terms with non-negative non-increasing autocorrelations $\gamma_j$, and the claim $S_{b+1}/(b+1)^2 \le S_b/b^2$ for $S_b := \sum_{s,t\le b} \gamma_{s-t}$ reduces to $\sum_{j=0}^b c_j\gamma_j \ge 0$ with $c_0 = b(b+1)$ and $c_j = 2[b(b+1) - j(2b+1)]$. Abel summation writes the sum as $\sum_{j=0}^{b-1} C_j(\gamma_j - \gamma_{j+1}) + C_b\gamma_b$ with partial sums $C_j = (2j+1)\,b(b+1) - (2b+1)\,j(j+1) \ge 0$ for $0 \le j \le b$, because $x \mapsto x(x+1)/(2x+1)$ is increasing; every term is then non-negative. (The weights $c_j$ themselves are not monotone --- $c_1 > c_0$ from $b=4$ --- so the partial sums have to be computed, not inferred.) Hence $\Lambda_1$ is non-increasing in $b$.

\emph{(i).} Monotonicity gives $v(b,\ell)/b \le v(\ell,\ell)/\ell$ for $b \ge \ell^* \ge \ell$, and the closed form of Lemma~\ref{lem:effsample} gives $v(\ell,\ell)/\ell = 2 - 1/\ell - 2(1-1/\ell)(1-(1-1/\ell)^\ell)$: it is $1$, $0.750$, $0.728$, $0.725$, $0.724$ at $\ell = 1,\dots,5$ and rises towards $2/e = 0.736$, staying below $2/e$ for every $\ell \ge 3$ (an elementary expansion via $(1-x)^{1/x} \le e^{-1-x/2}$; $\ell = 3,4$ are the computed values). So $\sqrt{v(\ell,\ell)/\ell} \le \sqrt{2/e} = 0.858$ for every $\ell \ge 3$, the two degenerate cases $\ell = 1,2$ excepted. \emph{(ii)} is (i) read against Corollary~\ref{cor:nonempty}.
\end{proof}

\begin{proof}[Proof of Corollary~\ref{cor:invbudget}]
By \eqref{eq:blockscore} the two parts of a block score are uncorrelated, so $\mathrm{sd}_{\mathrm{blk}}^2 \ge \rho_{\mathrm{blk}}^2 v(b,\ell)/b$ in unit-PnL-variance units, and \eqref{eq:nonempty} reads $E_N > \lambda(\Lambda)\rho_{\mathrm{blk}}\sqrt{v(b,\ell)/b} = \vartheta\rho_{\mathrm{blk}}$. The denominator of $E_N/\rho_{\mathrm{blk}}$ needs a bound \emph{below}, which Axiom~\ref{ax:snr} does not supply --- a ceiling on the dispersion is the wrong direction. The declaration supplies it: stage S1 declares $\rho_{\mathrm{blk}} = \kappa\rho$ and Proposition~\ref{prop:kelly} charges the penalty at that floor, so $E_N > \vartheta\kappa\rho$, path by path on the event that S5 opens a position. The second display adds $\hat{\mu} \le \rho$, which Lemma~\ref{lem:sharpe} gives for the mean and not for the realised value just after a search, so it is stated as a hypothesis on the pipeline; granting it, $\vartheta \le \delta/\kappa$, and $\lambda$ increasing gives $\Lambda_1 \le \lambda^{-1}(\delta/\kappa)$, with $\lambda(2.624) = 1$ the numerical case. A measured $\kappa_h$ may replace the declared $\kappa$ whenever the measurement lower-bounds the dispersion the deployment window will show.
\end{proof}

\begin{proof}[Proof of Proposition~\ref{prop:ensemble}]
The perturbations are drawn independently of the data and the members share hyperparameters and candidate set, so $\bar{h}$ is one fixed function of the sample and auxiliary randomness: the search term is unchanged. The capacity argument is not that the hull's dimension is no larger --- for trees it is strictly richer --- but that the covariance trace $d_{\mathrm{eff}} = \sigma^{-2}\sum_i \mathrm{Cov}(\hat{y}_i, y_i)$ of \S\ref{sec:calibration} is linear in the fitted values, so the ensemble's value equals any one member's exactly. The exchangeable decomposition of a member's error with common correlation $\varrho$ gives $\tau_{\mathrm{ens}}^2 = \varrho\tau^2 + (1-\varrho)\tau^2/m$, the stated form with $\tau_\infty^2 = \varrho\tau^2$; the floor $\tau_{\mathrm{ens}}^2 \ge d/n_{\mathrm{fit}}$ is the Gauss--Markov bound applied to $\bar{h}$, an estimator in the class --- Theorem~\ref{thm:capacity} is an upper bound and cannot supply it.
\end{proof}

\begin{proof}[Proof of Proposition~\ref{prop:minimax}]
(i) Assouad's lemma in its two-point form \citep{tsybakov2009introduction}, on the hypercube $\theta_\omega := \delta(2\omega - \mathbf{1})$, $\omega \in \{0,1\}^d$: for any estimator, randomised included, $\max_\omega \mathbb{E}_\omega\|\hat{\theta}-\theta_\omega\|^2 \ge d\delta^2(1 - \max_{\omega,j}\mathrm{TV}(P_\omega, P_{\omega^{(j)}}))$. Neighbouring laws differ by $2\delta X_{ij}$ in the mean of observation $i$, so the Kullback--Leibler divergence has expectation $2N\delta^2/\sigma^2$ over the design, and Pinsker with Jensen gives $\mathrm{TV} \le \delta\sqrt{N}/\sigma$. Axiom~\ref{ax:snr} contains the ball $\{\|\theta\| \le \rho\sigma\}$, in which we work. If $d \le 4\rho^2 N$, take $\delta = \sigma/(2\sqrt{N})$: the cube lies in the ball, $\mathrm{TV} \le 1/2$, and the risk is at least $d\delta^2/2 = \sigma^2 d/(8N)$; otherwise $\delta = r/\sqrt{d}$ gives $\mathrm{TV} < 1/2$ and risk at least $r^2/2$.

\emph{(ii)} The design is fixed, so neighbouring laws are $N(X\theta_\omega, \Sigma)$ and $N(X\theta_{\omega^{(j)}}, \Sigma)$ with $\Sigma = (\sigma^2/H)LL^\top$, where $L$ is the $n\times(n+H-1)$ forward-sum matrix $(Lv)_t = \sum_{s=t}^{t+H-1} v_s$, of full row rank; their Kullback--Leibler divergence is $1/2\,(2\delta)^2\,x^{(j)\top}\Sigma^{-1}x^{(j)}$ with $x^{(j)}$ the $j$-th column of $X$.

The quadratic form has a variational bound: for any $x = Lc$, $x^\top(LL^\top)^{-1}x = \min\{\|v\|^2 : Lv = x\} \le \|c\|^2$, the minimiser being $L^\top(LL^\top)^{-1}x$. Column $j$ is $x^{(j)} = Lc^{(j)}/(H\sqrt{\kappa_j})$ with $c^{(j)}$ the $\pm 1$ block sequence --- a trailing rather than a forward sum is a shift of $c$, which does not change its norm --- so $\|c^{(j)}\|^2 = n$, and $\kappa_j := \|Lc^{(j)}\|^2/(nH^2) \in [1/3, 1]$: over a block whose sign agrees with the next, $(Lc)_t = \pm H$ throughout, and over a block whose sign flips, $(Lc)_t$ runs linearly from $\pm H$ to $\mp H$ with mean square $H^2(1/3 + 2/(3H^2))$, up to boundary terms of relative order $H/n$.

Hence $x^{(j)\top}\Sigma^{-1}x^{(j)} \le (H/\sigma^2)\cdot n/(H^2\kappa_j) \le 3n_{\mathrm{fit}}/\sigma^2$, the divergence is at most $6\delta^2 n_{\mathrm{fit}}/\sigma^2$, and Pinsker gives $\mathrm{TV} \le \delta\sqrt{3n_{\mathrm{fit}}}/\sigma$. The rest is (i) with $N$ replaced by $3n_{\mathrm{fit}}$: if $d \le 12\rho^2 n_{\mathrm{fit}}$ take $\delta = \sigma/(2\sqrt{3n_{\mathrm{fit}}})$, so that the cube fits ($d\delta^2 = \sigma^2 d/(12n_{\mathrm{fit}}) \le r^2$), $\mathrm{TV} \le 1/2$ and the risk is at least $d\delta^2/2 = \sigma^2 d/(24n_{\mathrm{fit}})$; otherwise $\delta = r/\sqrt{d}$ gives $\mathrm{TV} \le \rho\sqrt{3n_{\mathrm{fit}}/d} < 1/2$ and risk at least $r^2/2$.

Nothing was sub-sampled: the estimator is an arbitrary function of all $n$ labels, and the bound says the overlap leaves them worth a constant multiple of $n_{\mathrm{fit}}$ independent ones. The persistence of the design is what makes the variational bound bite. For a white design the form $x^\top \Sigma^{-1} x$ is governed instead by the small eigenvalues of $\Sigma$, and the spectrum of the equal-weight overlap, $(\sigma^2/H)|1 + e^{i\omega} + \dots + e^{i(H-1)\omega}|^2$, vanishes at $\omega = 2\pi k/H$; generalised least squares exploits those frequencies and beats $\sigma^2 d/n_{\mathrm{fit}}$, which is Theorem~\ref{thm:capacity}'s restriction to features persistent at the label scale.
\end{proof}

\begin{proof}[Proof of Proposition~\ref{prop:shuffle}]
(i) $\mathrm{Cov}(r_s,r_t) = \rho^{|s-t|}\mathrm{Var}_\nu(r)$ as in the proof of Proposition~\ref{prop:occfloor}, so a block's variance is $\mathrm{Var}_\nu(r)\,v(b,\ell)/b$ with $v$ the Bartlett sum of Lemma~\ref{lem:effsample}, equal to $1$ at $b=1$ and increasing to $2\ell-1$. A fold is a simple random sample of size $n/k$ from the path's $n$ values, so the finite-population formula gives its conditional variance $s_n^2(k-1)/(n-1)$, and $\mathbb{E} s_n^2 = \mathrm{Var}_\nu(r)(1 - v(n,\ell)/n)$; with $k = B$ the ratio of the two is $v\,(n-1)/(n-b)$.

(ii) Hoeffding's inequality for sampling without replacement \citep{hoeffding1963probability}, and for the conditionally independent noise, with a union bound over $k$ folds, two tails and $|\mathcal{H}|$ candidates, puts every fold score of every candidate within $2t$ of $\bar{R}_n(h)$ outside probability $4|\mathcal{H}|k\exp(-nt^2/(2kM^2))$, which vanishes when $k\log k/n \to 0$; at $t$ below a quarter of the runner-up gap both procedures select the unique minimiser of $\bar{R}_n$, and the laws of the deployed positions agree to that probability.

(iii) By (ii) the shuffled statistics of $A$ and $B$ converge to $a$ and to $a - u(1-\hat{p}_n) + v\hat{p}_n$, $\hat{p}_n \to p$ the realised occupation of the rare state, and $u(1-p) > vp$ makes the latter smaller, so $B$ is selected eventually. The worst admissible future of Theorem~\ref{thm:cvar} puts mass $\Lambda p \le 1$ on the rare state, giving $R^+(B) = a - u(1-\Lambda p) + v\Lambda p$ and $R^+(A) = a$; the two differences follow by subtraction.

For the block statistic, write $O_b$ for a block's occupation of the rare state, so that $B$'s block loss is $a - u + (u+v)O_b + \bar{\xi}_b$ and $A$'s is $a + \bar{\xi}_b$. By Proposition~\ref{prop:eiv} (i) the conditionally mean-zero noise can only raise the upper-tail $\mathrm{CVaR}$ of $B$'s score, and by the robust representation with the admissible density $\Lambda\mathbf{1}\{O_b=1\}$ --- admissible since $\Lambda\mathbb{P}(O_b=1) \le \Lambda p \le 1$ --- $\mathrm{CVaR}_{1/\Lambda}(O_b) \ge \Lambda\mathbb{P}(O_b=1) \ge \Lambda p\rho^{b-1}$, the block being wholly in the rare state whenever it starts there and sees no redraw. For $A$, Cauchy--Schwarz gives $\mathrm{CVaR}_{1/\Lambda}(\bar{\xi}_b) = \Lambda\mathbb{E}[\bar{\xi}_b\mathbf{1}_{\mathrm{tail}}] \le \Lambda\,\mathrm{sd}(\bar{\xi}_b)\sqrt{1/\Lambda} = \sigma_\xi\sqrt{\Lambda/b}$. The block scores form a stationary ergodic sequence, so as $B \to \infty$ the empirical $\mathrm{CVaR}$s converge to these population values ($\mathrm{CVaR}_\alpha$ is continuous for the Wasserstein-1 distance), and the block statistic prefers $A$ whenever $a - u + (u+v)\Lambda p\rho^{b-1} > a + \sigma_\xi\sqrt{\Lambda/b}$, which is \eqref{eq:blockwins}. At the stated numbers the left side of \eqref{eq:blockwins} is $5.6 \times 0.4 \times 0.975^{39} - 0.6 = 0.235$ and the right side is $0.316\sigma_\xi$; as $b/\ell \to \infty$, $\rho^{b-1} \to 0$ and the left side tends to $-u$.
\end{proof}

\begin{proof}[Proof of Theorem~\ref{thm:canonical}]
S1 is Theorem~\ref{thm:impossible} and S4 is Theorem~\ref{thm:search}; S2, S3 and S5 are Propositions \ref{prop:minimax}--\ref{prop:fullkelly}, proved above.
\end{proof}

\begin{proof}[Proof of Proposition~\ref{prop:tailfloor}]
\emph{The family.} $S_t$ i.i.d.\ with $\mathbb{P}(S=s) = 1/\Lambda$ (Axiom~\ref{ax:persistence} at $\ell=1$, the records independent; two fixed kernels, so Axiom~\ref{ax:invariance} with defect zero). Features $X_t^{(j)} = R_t e_{j,t}$, $j \le N$, with $R_t$ a fair sign and $e_{j,t}$ i.i.d.\ signs of mean $\gamma$, $\gamma^2 = 1/2$, all independent of $S$; candidate $j$ holds one unit of $h_j(X) = X^{(j)}$. Write $U_j := X^{(j)} - \gamma R$, so that $\mathbb{E}[X^{(k)}U_j] = (1-\gamma^2)\mathbf{1}_{k=j}$. Under $P_i$, $Y = \mu + \sigma\eta$ with $\eta$ standard normal and
\[
\mu = \begin{cases} -aR + \tilde{a}U_i & S = s \\ (aR - \tilde{a}U_i)/(\Lambda-1) & S = g \end{cases} \qquad \tilde{a} := \frac{a\gamma}{1-\gamma^2}, \quad a := \delta/\gamma.
\]
Candidate $j$'s conditional loss $\mathbb{E}[-X^{(j)}Y \mid S]$ is $\delta\mathbf{1}_{j\ne i}$ in the stress state and $-\delta\mathbf{1}_{j\ne i}/(\Lambda-1)$ in the calm one, by the orthogonality of $U_i$: stationary means all vanish, which is (i); the trapped candidates' risk law puts mass exactly $1/\Lambda$ on the atom $\delta$, so their $\mathrm{CVaR}_{1/\Lambda}$ is $\delta$ against candidate $i$'s zero, and Theorem~\ref{thm:cvar} at $b=1$, with the history's occupation at its stationary values, makes these the worst-case deployed risks, which is (ii). Axiom~\ref{ax:snr} state-wise: $\mathbb{E}[\mu^2 \mid s] = a^2 + \tilde{a}^2(1-\gamma^2) = a^2/(1-\gamma^2) = 4\delta^2 \le \rho^2\sigma^2 \le \rho^2\mathrm{Var}(Y \mid s)$ exactly when $\delta \le \rho\sigma/2$, the calm state at that over $(\Lambda-1)^2$; and $\mathrm{Var}(Y \mid z) \in [\sigma^2, \sigma^2(1+\rho^2)]$ is the $\chi_\sigma$ clause.

\emph{(iii).} Let $P_0$ be the law with no exempt candidate ($\tilde{a} = 0$ in both states). Given $(S,X)$ the conditional laws of $Y$ under $P_i$ and $P_0$ are Gaussians whose means differ by $\tilde{a}U_i$ in stress and $-\tilde{a}U_i/(\Lambda-1)$ in calm, and the laws of $(S,X)$ are common, so
\[
\mathrm{KL}(P_i^{(n)} \| P_0^{(n)}) = n\,\frac{\tilde{a}^2(1-\gamma^2)}{2\sigma^2}\left[\frac{1}{\Lambda} + \frac{1-1/\Lambda}{(\Lambda-1)^2}\right] = \frac{n\delta^2}{2(1-\gamma^2)\sigma^2(\Lambda-1)} = \frac{n\delta^2}{\sigma^2(\Lambda-1)}
\]
at $\gamma^2 = 1/2$, using $\tilde{a}^2(1-\gamma^2) = \delta^2/(1-\gamma^2)$ and $1/\Lambda + 1/(\Lambda(\Lambda-1)) = 1/(\Lambda-1)$. With the index uniform on $\{1,\dots,N\}$, the variational form of mutual information gives $I \le N^{-1}\sum_i \mathrm{KL}(P_i^{(n)} \| P_0^{(n)})$, and Fano's inequality $\mathbb{P}(\hat{T}=i) \le (I+\log 2)/\log N$; success at $1/2$ therefore needs $n\delta^2/(\sigma^2(\Lambda-1)) \ge \frac{1}{2}\log(N/4)$, which is \eqref{eq:tailfloor}. Randomisation and adaptivity sit inside ``any measurable $\hat{T}$''; a rule observing $Z$ is covered because the divergence is computed on the joint law of $(S,X,Y)$, and an observable-only rule sees a marginal and, by data processing, no more.

\emph{(iv).} Take $\mu = \tilde{a}U_i$ in both states: candidate $i$'s mean edge is $\tilde{a}(1-\gamma^2) = \delta$, the per-period divergence is $\tilde{a}^2(1-\gamma^2)/(2\sigma^2) = \delta^2/\sigma^2$, and the same Fano step gives the floor without $(\Lambda-1)$; the ranking by sample means resolves it once $\delta$ clears the Gaussian null maximum $\sigma\sqrt{2\log N/n}$ with margin (Theorem~\ref{thm:search} (i)), i.e.\ at $n \asymp \sigma^2\log N/\delta^2$.

\emph{(v).} With the path observed, the stress sample holds $n/\Lambda\,(1+o(1))$ periods with high probability (Chernoff); each candidate's stress-mean estimate is sub-Gaussian at scale $\sigma\sqrt{\Lambda/n}$ up to $(1+\rho^2)^{1/2}$, and the smallest of the $N$ estimates identifies $i$ once $\delta \ge C\sigma\sqrt{\Lambda\log N/n}$, by the Gaussian maximal inequality.
\end{proof}

\begin{proof}[Proof of Proposition~\ref{prop:eiv}]
(i) Write $\mathrm{CVaR}_\alpha$ in its Rockafellar--Uryasev form, $\mathrm{CVaR}_\alpha(W) = \min_\eta\{\eta + \alpha^{-1}\mathbb{E}(W-\eta)^+\}$. For every $\eta$, conditioning on $r$ and applying Jensen's inequality to the convex $x \mapsto (x-\eta)^+$, $\mathbb{E}[(r+\xi-\eta)^+ \mid r] \ge (r+\mathbb{E}[\xi\mid r]-\eta)^+ = (r-\eta)^+$, so the objective for $r+\xi$ dominates the objective for $r$ at every $\eta$, and the minimum does too. Nothing about the law of $(r,\xi)$ beyond $\mathbb{E}[\xi\mid r] = 0$ and integrability was used; in particular $\sigma_\xi^2(r)$ may depend on $r$. (This is the convex-order monotonicity of $\mathrm{CVaR}$ \citep{follmer2016stochastic}, with $r+\xi \succeq_{\mathrm{cx}} r$ from conditional Jensen, written out.)

(ii) For $W \sim N(m,v)$ and $z_\alpha := \Phi^{-1}(1-\alpha)$, $\mathrm{CVaR}_\alpha(W) = m + \sqrt{v}\,\mathbb{E}[Z \mid Z>z_\alpha] = m + \sqrt{v}\,\varphi(z_\alpha)/\alpha$, since $\int_{z_\alpha}^\infty z\varphi(z)\,\mathrm{d}z = \varphi(z_\alpha)$; at $\alpha = 1/\Lambda$ the coefficient is $\lambda(\Lambda)$ of Theorem~\ref{thm:invariance}. Independence makes $r+\xi \sim N(\bar{r}, s^2+\sigma_\xi^2)$, and subtracting the two values gives the display. For the two regimes put $x := \sigma_\xi^2/s^2$: $\sqrt{s^2+\sigma_\xi^2} - s = s(\sqrt{1+x}-1)$, which is $sx/2 + O(sx^2) = \sigma_\xi^2/(2s)$ as $x\to 0$ and $\sigma_\xi - s + O(s^2/\sigma_\xi)$ as $x\to\infty$. The substitutions $\sigma_\xi^2 = \sigma^2 H/b$ and $s^2 = \rho_{\mathrm{blk}}^2\sigma^2 v(b,\ell)/b$ are Lemma~\ref{lem:effsample} at $n=b$, the block's own noise and regime variances, and their ratio is the display; $v(b,\ell) \to 2\ell-1$ as $b/\ell \to \infty$ by the same lemma. The last equality is the dictionary $\rho_{\mathrm{blk}} = \kappa\rho$, $\rho^2 = \mathrm{SR}_{\mathrm{ann}}^2/P$ with $P$ periods per year (Lemma~\ref{lem:sharpe}) and $\ell = \tau_{\mathrm{yr}}P$ (Remark~\ref{rem:upsbox}), with $2\ell-1 \approx 2\ell$.

(iii) The first sentence is (i): as $B \to \infty$ the plug-in converges to $\mathrm{CVaR}_{1/\Lambda}$ of the score law, which by (i) sits above that of the risk law by a gap $B$ does not touch --- under (ii), by the display, a positive gap whenever $\sigma_\xi^2 > 0$. For the second, apply (ii) to each candidate, the two sharing $(\bar{r}, s^2)$, and subtract: the noisier candidate's tail mean is larger by $\lambda(\Lambda)(\sqrt{s^2+\sigma_2^2} - \sqrt{s^2+\sigma_1^2}) > 0$, a penalty on it since losses are ranked upward. When $\sigma_\xi^2(r)$ rises with $r$ the upper tail is where the noise is largest, so the gap of (i), a functional of the noise in the tail, is largest for the candidate whose bad blocks are its noisy ones --- the reading Experiment 2 tests.
\end{proof}

\begin{proof}[Proof of Proposition~\ref{prop:audit}]
Index the candidates the search can return by $i$, and write $\psi_i$ for candidate $i$'s realised performance on the held-out window $W$ and $\rho_i := \mathbb{E}[\psi_i \mid \mathcal{F}_{\mathrm{pre}}]$ for its true risk there, conditional on the data before the boundary (and on the path); $\epsilon_i := \psi_i - \rho_i$ is the window's own noise, mean zero given $\mathcal{F}_{\mathrm{pre}}$ by construction; Axiom~\ref{ax:adapted} and the purge of Proposition~\ref{prop:embargo} at the boundary are what make $\rho_i$ the rule's true risk on the window rather than a quantity the pre-boundary observations still reach. The truncated search returns $T_{\mathrm{trunc}}$, a function of the pre-boundary data alone, so $\mathbb{E}[A_{\mathrm{trunc}}] = \mathbb{E}[\psi_{T_{\mathrm{trunc}}}] = \mathbb{E}[\rho_{T_{\mathrm{trunc}}}]$: the held-out score of a rule chosen without the window is unbiased for its true risk. The full search returns $T_{\mathrm{full}} \in \operatorname{arg\,max}_i \phi_i$ for a statistic $\phi_i = a_i + b\,\psi_i$ in which the window enters with a positive weight $b$ and $a_i$ is $\mathcal{F}_{\mathrm{pre}}$-measurable --- the full-sample mean is one, $b$ the window's share of the sample. What the argument uses is that the weight on $\psi_i$ is positive and the \emph{same} for every candidate: with candidate-specific weights $b_i$ it delivers $\mathbb{E}[b_T\epsilon_T] \ge 0$ and no more, and for a statistic that is not affine in $\psi_i$, a Sharpe ratio, the sign of the selection term needs its own argument. Then
\[
\mathbb{E}[A_{\mathrm{full}}] - \mathbb{E}[A_{\mathrm{trunc}}] = \underbrace{\mathbb{E}[\epsilon_{T_{\mathrm{full}}}]}_{\text{selection bias, } \ge 0} + \underbrace{\mathbb{E}[\rho_{T_{\mathrm{full}}}] - \mathbb{E}[\rho_{T_{\mathrm{trunc}}}]}_{\text{gain in true risk from the extra data}}.
\]
The first term is non-negative for every law of the data: $G(\epsilon) := \max_i(a_i + b\rho_i + b\epsilon_i)$ is convex in the noise vector, so conditionally on $\mathcal{F}_{\mathrm{pre}}$ Jensen gives $\mathbb{E}[G(\epsilon)] \ge G(0) = \max_i(a_i+b\rho_i) \ge \mathbb{E}[a_{T_{\mathrm{full}}} + b\rho_{T_{\mathrm{full}}}]$, while $G(\epsilon) = a_{T_{\mathrm{full}}} + b\rho_{T_{\mathrm{full}}} + b\epsilon_{T_{\mathrm{full}}}$; subtracting, $b\,\mathbb{E}[\epsilon_{T_{\mathrm{full}}}] \ge 0$. This is the selection term of Theorem~\ref{thm:search}, whose part (ii) bounds its size by $\sigma\sqrt{2I(T;\phi)}$ and whose sign Jensen fixes. The second term is what the extra data bought: it is non-negative when the search, given more data, chooses no worse in true risk, which is not a theorem but the purpose of the data; under it the difference is an upper bound on the search component, and equalising the data available to the two arms --- both searching on the pre-boundary window and differing only in whether the post-boundary window was visible --- sets it to zero and turns the bound into an estimate, at the cost of no longer auditing the search that was actually run. For the last claim, $k$ audits each returning a scalar the researcher may act on constitute a search of size $k$ over the held-out window, costing $2\log k$ nats of that window's budget by Theorem~\ref{thm:search}.
\end{proof}

\begin{proof}[Proof of Proposition~\ref{prop:occfloor}]
(i) Two times lie in the same sojourn with probability $\rho^{|s-t|}$, so $\mathrm{Cov}(I_s,I_t) = p(1-p)\rho^{|s-t|}$; summing over the window gives the display, with $(1+\rho)/(1-\rho) = 2\ell-1$ and $(1-\rho)^{-2} = \ell^2$. (ii) Conditionally on the redraw pattern the sojourn states are i.i.d.\ $\nu$, so $\bar{\pi}^+(k)$ is a weighted sum of independent Bernoulli$(p)$ variables with weights $L_j/m$ summing to one; Hoeffding's inequality \citep{hoeffding1963probability} gives the display, and $\mathbb{E} W$ is the covariance sum of (i) again, $\sum_j L_j^2$ counting the same-sojourn pairs. (iii) The chain is uniformly ergodic, so the Markov-chain central limit theorem \citep{meyn2009markov} applies to $I_t$ with asymptotic variance $p(1-p)(2\ell-1)$; setting $\mathbb{P}(\bar{\pi}^+(k) \ge \Lambda p) = \epsilon$ in the limit and solving gives $\Lambda_\epsilon$.
\end{proof}

\clearpage
\section{Experiments}
\label{app:experiments}
The theorems have constants in them, and constants are where theory usually parts company with practice. Six experiments keep them honest, and every one of them reads real market series --- an axiom, being a premise about markets, is tested on markets or not at all (\S\ref{sec:empirical}); each is one committed script, reproduces from a fixed seed, and has its output in the committed \texttt{RESULTS.txt} and its design in full in its module docstring. The experiments run 1 to 6: one equity index (Experiment 1), the same on twenty series (2), the effective sample measured (3), the five stages walked once with the constants declared first (4), the first trial of a prediction of \S\ref{sec:empirical} (5), and the cross-sectional library the others did without (6). The table below is the index; the sections after it give each experiment its design, one table and its verdict, and the panel-by-panel narrative stays in the output file.

A paper about budgeted search owes an account of its own. Experiment 1's original panels were fixed before their results were seen. Everything added afterwards was adaptive in the sense of Theorem~\ref{thm:search} (ii): Experiment 2 asks whether Experiment 1's one series was typical, Experiments 3--6 came afterwards, and Experiment 6's panel G is the latest addition, its design fixed before its results were seen. The rule libraries were written once and not tuned; every number is regenerated from the committed scripts at fixed seeds. The selection this implies is over \emph{which experiments to run}, not over which results to report --- every panel run is in the output file --- and the results that were not predicted are stated as findings to be replicated, not as confirmations.

\begin{center}
\footnotesize
\begin{longtable}{>{\raggedright\arraybackslash}p{7em}>{\raggedright\arraybackslash}p{0.25\linewidth}>{\raggedright\arraybackslash}p{0.50\linewidth}}
\caption{The experiments, and what each found. Every number here is copied from \texttt{experiments/RESULTS.txt}; the design of each is in its own section below, and in the module docstring in full.}\\
\toprule
 & \textbf{checks} & \textbf{finding} \\
\midrule
\endfirsthead
\toprule
 & \textbf{checks} & \textbf{finding} \\
\midrule
\endhead
\midrule
\multicolumn{3}{r}{\textit{continued on next page}} \\
\endfoot
\bottomrule
\endlastfoot
1 one real series & the plug-ins of \S\ref{sec:price} and \S\ref{sec:empirical} & median $\hat c(4)=1.95$ against the Gaussian $1.43$, $c^2\approx0.9\Lambda$ at every level; the invariance test contradicts no declared coordinate at its own level (panels H, H4); A4's times of the order declared, the volatility coordinate's rank refuting the $\ell = 60$ Experiment 4 declares of it; the scores' coherence across blocks at $1.10$ and the tail statistic's at $1.05$, inside their null and below the $1.52$ the declared $\rho_{\mathrm{blk}} = \rho$ implies \\
2 twenty assets, seven $b$ & Prop.~\ref{prop:eiv} (ii)--(iii), Rem.~\ref{rem:blocklen}, A2's defect test and P4 of \S\ref{sec:empirical} & raw $\hat c(4)$ above Gaussian on 19 of 20 assets, studentised on none, the excess decaying as $b^{-0.53}$; at the test's own level the defect test's rejections fall to the two matrix-priced bond funds (panel D); the tail criterion's excess degradation lands on the plug-in $\hat c(\Lambda')$ at both levels, at the edge of the joint Gaussian null's power (panel E) \\
3 effective sample & Lem.~\ref{lem:effsample}, Thm.~\ref{thm:capacity} & $\upsilon$ median $0.86$ (the high-yield fund $3.3$, a stale NAV); a fitted coefficient's variance tracks $H$, never $\ell^\ast=60$; panel A2 reads the box's lower end off the series' own lag-one autocorrelation, $-0.082$ on the S\&P 500 \\
4 the walk & S1--S5 of Thm.~\ref{thm:canonical}, Prop.~\ref{prop:audit}, Prop.~\ref{prop:ensemble}, Rem.~\ref{rem:lamlower} & the $72$-way comparison is over budget at the price real block scores pay and within it at the Gaussian, the haircut exceeds the selected rule's edge, and S5 returns no position; one audit bounds the search overdraft at $0.80$ annual Sharpe \\
5 what decay scales with & P2 of \S\ref{sec:empirical}; Lem.~\ref{lem:effsample}, Thm.~\ref{thm:capacity}, Thm.~\ref{thm:search} (i) & fit slope $0.66\pm0.06$ on $\sqrt{dH/n}$ and $1.08$ on the null-signal constant; search slope $0.79$ on the derived expected maximum; $\ell$ adds nothing to either. Not refuted \\
6 the cross-section & Prop.~\ref{prop:breadth}; P1 and P3 of \S\ref{sec:empirical}; Prop.~\ref{prop:eiv} & $N_a^{\mathrm{eff}}$ $21.3$ against $1.6$ and a ceiling ratio of $16$; long-only reproduces $c^2\approx0.9\Lambda$, long-short sits at the Gaussian value; P1 a clean null at $\Lambda=2$ \emph{with} power; P3's fitted ladders put the cross-sections' peaks $8\times$ apart against a predicted $7.8\times$ (panel G) \\
\end{longtable}
\end{center}

\subsection{Experiment 1: the constants on one real daily series}
\label{experiment-1-real-series}
The theory settles the range each constant can be in; only data say where a real market sits inside it. Experiment 1 asks that for the constant whose range is widest, $c(\Lambda)$, on the smallest real dataset that can answer it: the daily total-return series of one equity index from 1997 to 2026 ($n=7{,}454$ days), a library of $N=71$ daily timing rules written once and not tuned --- moving-average, momentum, reversal and volatility-target rules, VIX and yield-curve conditions, long--short versions, pairwise conjunctions --- and blocks of $b=60$ days, $B=124$. Each rule's excess returns are standardised and cut into block scores, and $\hat c(\Lambda)$ is the plug-in \eqref{eq:plugin} on them, rule by rule. Only the distribution across rules is reported; no rule is selected.

\begin{table}[htbp]
\centering\small
\begin{tabular}{rrrrrrrr}
\toprule
$\Lambda$ & $c$ Gauss & p25 & median & p75 & max & bound $\Lambda$ & $c_{\mathrm{med}}^2/\Lambda$ \\
\midrule
2 & 1.168 & 1.207 & 1.336 & 1.443 & 1.565 & 2 & 0.89 \\
4 & 1.426 & 1.734 & 1.950 & 2.357 & 2.846 & 4 & 0.95 \\
8 & 1.785 & 2.284 & 2.711 & 3.723 & 5.079 & 8 & 0.92 \\
20 & 2.466 & 3.473 & 4.692 & 6.447 & 11.386 & 20 & 1.10 \\
\bottomrule
\end{tabular}
\caption{Experiment 1, panel A. Quantiles across the $71$ rules of the plug-in $\hat c(\Lambda)$ on real block scores, $b=60$, $B=124$; ``$c$ Gauss'' is the closed form of \S\ref{sec:price}. The median fits $c \approx \Lambda^{0.54}$. Two controls (\texttt{RESULTS.txt}): the plug-in on Gaussian block risks at $B=124$ averages $1.166$, $1.418$, $1.758$, $2.344$ (unbiased); on a rare-regime law with homoskedastic score noise of sd $1/\sqrt{60}$ added, $\hat c(4)$ falls from $3.28$ to $2.11$ --- noise of that kind pushes the estimate \emph{towards} the Gaussian value. Experiment 2 shows the real noise is not of that kind.}
\end{table}

Real daily block \emph{scores} sit on the dear side of \S\ref{sec:price}'s dichotomy: $c^2/\Lambda$ is $0.9$--$1.1$ at every level against $0.68$, $0.51$, $0.40$, $0.30$ for the Gaussian law. The score law is left-skewed and heavy-tailed (per-rule median skewness $-0.53$, excess kurtosis $1.92$), and that alone carries $c$ from $\Lambda^{1/4}$ to $\Lambda^{1/2}$. Buy-and-hold, which involves no design choice of ours, gives $\hat c(4)=2.41$ with a block-bootstrap $90\%$ interval $[1.80, 2.83]$ excluding the Gaussian $1.43$. Panel A is the result the affordability table of \S\ref{sec:calibration} consumes, because it changes which branch the reader should expect to be in and its controls rule out the obvious artefacts. The plug-in is computed on the same block scores the ranking statistic is computed from, so it is the inflation Proposition~\ref{prop:tailprice} defines at this $b$ --- not a measurement of the block-risk law; Experiment 2's studentised column separates the two.

The other panels measure, on the same data, the quantities the body reads off this series panel by panel (\S\ref{sec:empirical}'s table, Remarks~\ref{rem:blocklen} and \ref{rem:whichell}), their narrative staying in \texttt{RESULTS.txt}: the implied block-signal variance behind Proposition~\ref{prop:eiv} (noise-to-signal $6.6$); the lag-one autocorrelation of consecutive block scores, $-0.01$ against a block-permutation null of $\pm0.08$ where $\rho_{\mathrm{blk}}=\rho$ would give $0.17$ (both halves of the prediction at the declared value, whose lags sum to the $1.52$ below), so the implied $\rho_{\mathrm{blk}}^2$ has median $0.0009$; the block scores' coherence across blocks, the count of Corollary~\ref{cor:blocks} (a) read on the sample --- the scores themselves at a median of $1.10$, the tail statistic $\tau_\psi$ at $1.05$ at $\Lambda=4$ and $1.00$ at $1/12$, against a permutation null of $1.00$ ($95\%$ point $1.21$), inside the null and below the $\upsilon/\upsilon_b = 1.52$ the declared $\rho_{\mathrm{blk}} = \rho$ implies at $\ell = 60$, the autocorrelation row's rejection of $\kappa = 1$ read at the block scale (Remark~\ref{rem:blocklen}); the scores' magnitudes at $1.72$, $2.58$ at the $90$th percentile, outside the null --- the score noise's scale is coherent across blocks, a dependence that is not a Gaussian pair's, so the tail statistic's count is read here rather than derived from the scores' (\S\ref{sec:limits}); the shuffled-block null of Remark~\ref{rem:lamlower}'s $\hat\Lambda_{\mathrm{split}}$; prediction P1 on this library (a power statement: $0.45$ and $0.52$ against the mean's $0.56$, $t=-1.4$ and $-0.9$); the coherence times of candidate coordinates; and the occupation floor on $\Lambda$.

\textbf{Panel H4: the defect test's own level.} The null Remark~\ref{rem:lamlower} reads the split against reassigns a cell's days between the halves in chunks of $L$ consecutive cell-days ($L = 20$ in panel H), so that dependence shorter than the chunk survives the permutation; at $L = 1$ it is the classical permutation test, anti-conservative under dependence \citep{chung2013permutation,kunsch1989jackknife,politis1994stationary}. The chunk has a price. The test's size is set by the ratio $r$ of the sampling covariance of the bin indicators' means to the part a Bartlett window of width $L$ retains, which in the geometric instance $\Gamma(j) = (1-1/\ell)^j\Gamma(0)$ is $r = (2\ell-1)/v(L,\ell)$ --- the kernel of \eqref{eq:vcontig} read as a shortfall, the same for every bin count --- and the two-bin size at nominal level $\alpha$ is $2(1 - \Phi(z_{1-\alpha/2}/\sqrt{r}))$. At a bin-indicator coherence of $2\ell - 1 = 66$ cell-days, chunks of $L = 20$ give $r = 4.0$ --- the nominal $5\%$ two-bin test has size $0.33$, and at the panel's eight bins $0.81$ --- and holding the two-bin size below $0.06$ needs $L \approx 12\ell$. The ratio is estimable from the sample the test is run on, as the full to the $L$-truncated integrated autocorrelation time of the bin indicators, and the prescription is to take $L$ past the plateau of the truncated time and report $\hat{r}$ at the $L$ used. Panel H4 measures both halves of that sentence: on a null world built to the geometric instance the whole test at nominal $5\%$ rejects $0.85$--$0.97$ of the time at an indicator coherence of sixty-five cell-days and $L = 20$, approaching its level only past $L = 120$; on the real cells $\hat{r}$ at $L = 20$ is $1.3$--$1.8$ on the volatility coordinates and up to $3.3$ on the term spread, largest on the slowest coordinate. The honest null takes the largest chunk that leaves ten chunks in every used cell and scales its $95\%$ point by $\sqrt{\hat{r}}$ for the coherence it still misses; at it panel H's verdicts reduce to one row in fifteen (Remark~\ref{rem:lamlower}), and panel D of Experiment 2 re-reads every asset the same way.

\textbf{Panel J: Axiom~\ref{ax:persistence} on the declared coordinates} --- the integrated autocorrelation times of each coordinate's cell indicators and rank, the embargo's predictability bound, and, one line per coordinate, the pair coefficient of its five quantile cells at lag $b$; the verdicts are the two A4 rows of \S\ref{sec:empirical}'s table. The numbers the body consumes: the indicators' times run $14$--$276$ against the $2\ell-1 = 79$ that panel E's AR(1) fit asserts for realised volatility (at its measured $\ell = 40$; $119$ at the $\ell = 60$ Experiment 4 declares of it), $29$--$440$ against $197$ on the log VIX, $164$--$628$ against $1475$ on the term spread, the volatility coordinate's rank at $276 = 2\cdot138-1$ --- the clause the budget rests on off by the factor a declaration is refined by; the pair coefficient at $b = 60$ reads $0.19$, $0.29$ and $0.48$ on the three coordinates, a lower bound on the process's mixing coefficient $\beta(b)$ and not small on any of them --- the reason Axiom~\ref{ax:persistence} is a coherence time and not a mixing coefficient (Remark~\ref{rem:independence}); and the predictability bound reads $0.32$--$0.80$ at the $112$-day embargo the AR(1) chain computes against a tolerance of $0.10$ (contradicted), $0.086$--$0.126$ at Experiment 4's $523$ days (not contradicted, the halves disagreeing by $0.2$--$0.4$), the first gap at which it falls to $0.10$ being $390$--$550$ days. No random numbers.

\subsection{Experiment 2: twenty assets, seven block lengths}
\label{app:exp2}
Twenty daily total-return series from 1996/1999 to 2026 --- SPY, QQQ, the nine SPDR sectors, five Vanguard equity index funds and four bond funds, each from its own endpoint, unspliced --- with the price-only part of Experiment 1's library plus its VIX filters, $47$ rules per asset, the same for every asset. The bond funds' NAVs are matrix-priced, so their daily returns are smoothed; they are kept in panels A--B with that caveat and left out of C.

\textbf{Panel A: every asset at $b=60$.} Raw, the candidate-median $\hat c(4)$ exceeds the Gaussian $1.43$ by more than $0.1$ on nineteen of the twenty assets (medians $1.43$--$2.18$, buy-and-hold $1.48$--$2.74$); studentised by the within-block standard deviation it is within $0.1$ of its reference --- $1.45$, the tail of $t_{59}$, what the studentised plug-in returns on i.i.d.\ Gaussian data --- or below it on every asset, and the pooled excess kurtosis goes from $1.1$--$10.8$ to $-0.57$--$0.30$. The excess is the volatility clustering of the score noise --- the bias of Proposition~\ref{prop:eiv} (iii), a property of the estimator on daily data across sectors, regions, styles and bond funds, not of one index and not of the market's block-risk law, which remains unmeasured.

\begin{table}[htbp]
\centering\small
\begin{tabular}{rrrrrrrrr}
\toprule
$b$ & blocks & $\hat c$ raw & $\hat c$ stud. & $\hat c$ ref. & kurt. raw & kurt. stud. & VR & floor \\
\midrule
5 & 28531 & 2.11 & 1.46 & 1.86 & 7.53 & 2.88 & 0.95 & 0.92 \\
10 & 14263 & 2.08 & 1.39 & 1.59 & 8.10 & 0.35 & 0.93 & 0.92 \\
21 & 6790 & 1.93 & 1.32 & 1.48 & 5.03 & $-$0.12 & 0.89 & 0.92 \\
42 & 3387 & 1.82 & 1.31 & 1.46 & 3.89 & $-$0.22 & 0.88 & 0.92 \\
63 & 2256 & 1.78 & 1.30 & 1.45 & 2.41 & $-$0.32 & 0.91 & 0.92 \\
126 & 1127 & 1.63 & 1.35 & 1.43 & 1.41 & $-$0.32 & 0.87 & 0.92 \\
252 & 561 & 1.50 & 1.38 & 1.43 & 0.52 & $-$0.34 & 0.88 & 0.92 \\
\bottomrule
\end{tabular}
\caption{Experiment 2, panel B: the block-length sweep, all $940$ candidates pooled. ``Raw'' and ``stud.'' are candidate medians of the plug-in $\hat c(4)$ on block means and on studentised block means; ``ref.'' is the studentised plug-in on i.i.d.\ Gaussian data at that $b$ (the tail of $t_{b-1}$); VR is the median of $b\,\hat{\mathrm{Var}}(S_b)$ over the daily variance and ``floor'' the median Newey--West(10) long-run variance over the daily variance, so VR $-$ floor is the block-mean variance beyond ten-day dependence.}
\end{table}

\textbf{Panel B: the sweep.} The raw excess over the Gaussian value decays as $b^{-0.53}$ and is gone at one year; the studentised plug-in sits at or below its reference at every $b$. The variance ratio is below one and flat and the long-run-variance floor matches it at every $b$: on these rules the block-mean variance beyond ten-day dependence is indistinguishable from zero from a week to a year --- the block-signal variance $s^2$ of Proposition~\ref{prop:eiv} is below what thirty years resolve, as Axiom~\ref{ax:snr} predicts for a daily timing rule on one asset. It also corrects Experiment 1's ratio: the within-block Newey--West floor used there is biased down, so $6.6$ is a lower bound.

\textbf{Panel D: Experiment 1's invariance test, one asset at a time.} Remark~\ref{rem:lamlower}'s statistic and null, per asset, on three declarations: terciles of the common term spread, terciles of the asset's own realised volatility, and their $3\times3$ product. At the test's own level (Experiment 1, panel H4) the positives are two equity series in sixteen against the $2.4$ false the forty-eight equity tests predict at $5\%$, plus the two matrix-priced bond funds whose NAV staleness panels A--B already price as measurement rather than mechanism (their $\hat r$, $10$--$12$, is the largest in the panel); at the chunked null the term spread declared alone would be contradicted on fourteen of the sixteen equity series, and the pair removes or lowers that defect at Proposition~\ref{prop:representation}'s price ($\Lambda(\psi)$ $1.3$--$1.5 \to 1.6$--$1.9$).

\textbf{Panel E: P4 on our own library.} The one library never filtered by publication --- the same $47$ rules on the sixteen equity series --- supplies what a library of published signals cannot: an unselected control. Each criterion (mean, $\mathrm{CVaR}_{1/4}$, $\mathrm{CVaR}_{1/12}$) selects on the first half of the blocks, its selected rule's in-minus-out-of-sample value read \emph{in that criterion's own units}, net of the library's rule-average gap --- the $N=1$ control that absorbs Definition~\ref{def:reflexivity}'s decay --- so the ratio of the tail criterion's excess degradation to the mean's is Corollary~\ref{cor:budgetrobust}'s charge made observable, predicted at $\hat c(\Lambda')$ of the same training blocks. Pooled, the ratio lands on the plug-in at both levels --- $1.97$ on $\hat c(4)=1.82$, $2.92$ on $\hat c(12)=2.82$, moving between the levels by $+0.95$ against a predicted $+1.00$ --- where the same machinery on Gaussian blocks at the libraries' own covariance returns $0.90$ and $1.41$; against the joint null that keeps the one regime path all sixteen share, the measured ratios sit at its $4.5\%$ and $8.0\%$ points, and the cross-library ordering is beyond its power altogether (sixteen libraries on one regime path are not sixteen draws, Proposition~\ref{prop:breadth}).

\subsection{Experiment 3: the effective sample, measured}
\label{app:exp3}
Experiment 2's twenty series and forty-seven rules, nothing fitted: every number is a variance ratio.

\textbf{The mean term.} The Bartlett long-run-variance ratio $\hat\upsilon(W)$ of each rule's PnL --- the empirical counterpart of \eqref{eq:vcontig} --- has a median over assets of $0.91$ at a month, $0.86$ at a quarter and at a year, and $0.76$ at three years (at a year, medians across rules $0.72$--$1.00$ on fifteen of the sixteen equity series, the non-synchronously priced EM fund at $1.38$; $3.3$ on the high-yield fund whose matrix-priced NAV is autocorrelated by construction). The block-level charge $n/\ell^\ast$ with $\ell^\ast=60$ would need $\hat\upsilon=60$.

\textbf{The capacity term.} OLS of the $H$-day forward excess return on four persistent features: the Newey--West over HC0 variance ratio of the coefficients, the effective-sample ratio $n/n_{\mathrm{fit}}$, has medians $0.92,3.7,13.9,35.7$ at $H=1,5,21,63$ --- it tracks $H$ at $0.6$--$0.7$ of it and never approaches $60$.

\textbf{The three terms side by side.} With $\rho^2=0.01$ the median asset has $n=6830$, $n_{\mathrm{eff}}\approx8000$ and a tail-term sample $n_{\mathrm{eff}}/\hat c^2\approx2400$ against $n/60=115$ under the block-level charge: budgets of $68$ nats for a fit, $80$ for a search on the mean and $24$ for a search on the quarter-tail, where the block count $n/60$ would give $1.1$ for all three.

\textbf{The box.} Both ends of $\upsilon$ are Axiom~\ref{ax:snr}'s ceiling times a coherence time (Remark~\ref{rem:upsbox}): at the declared $\rho^2=0.01$ and $\tau_\mu = 1$ the upper end is $2.40$ for every rule, and only the high-yield fund's $3.3$ breaks it; below $H=1$ the one admissible road is the predictability term $\pi_\mu$, and a majority of the forty-seven rules take it on fifteen of the twenty series.

\textbf{Panel A2, what the shortfall is made of.} The first-order autocorrelation of the daily excess return is the same fact with no estimator in between: $-0.082$ on the S\&P 500 at $t=-7.2$, negative on sixteen of the twenty series and at $|t|>2$ on thirteen, median magnitude $0.060$ --- six tenths of the declared ceiling read as a per-period Sharpe ratio by Lemma~\ref{lem:sharpe}; the four positive ones are the three stale-priced bond funds and a non-synchronously traded emerging-market fund. The median rule's $0.86$ asks for $\rho\tau_\mu=0.072$, a predictable component of $0.72$ periods at the declared $\rho$, inside the box a one-period component supplies, $0.80$. Under Axiom~\ref{ax:invariance} as stated (\S\ref{sec:axioms}) this is an edge Axiom~\ref{ax:snr} budgets for rather than a measurement against an axiom. A ledger that means to stay conservative should still charge $\max(H,\hat\upsilon)$; Experiment 4 declared $\upsilon=2.19$ rather than measuring it, and its verdict survives the substitution.

\subsection{Experiment 4: the five stages walked once}
\label{app:exp4}
No theorem is under test: this experiment checks that S1--S5 of Theorem~\ref{thm:canonical} can be walked in order on committed data with the constants declared first, and that the budget of \S\ref{sec:budget} comes out as a ledger with a line per stage. What it does test are the declarations S1 makes --- \S\ref{sec:empirical}'s table reads three of them off this ledger. Experiment 1's series and its $71$ rules plus one fitted candidate, in three eras fixed before anything is computed: training (1997 to end-2016, $n=5034$ days), an embargo in which nothing is read, and a held-out era ($1897$ days to August 2026) read once, by the audit of S4. It is not a test of P1.

\textbf{S1.} Declared: $\Lambda=4$, the criterion run at $1/\Lambda$, $\ell=60$ (declared of the volatility coordinate, above Experiment 1 panel E's measured $40$), $\rho^2=0.01$, $\rho_{\mathrm{blk}}=\rho$, $H=1$, $\tau_\mu=1$. Then $\upsilon=2.19$ (the measured Bartlett ratio $0.80$ is printed and not used, since using it would be fitting a declared constant; it sits exactly at the lower end $H - 2\rho\tau_\mu = 0.80$ of Remark~\ref{rem:upsbox}'s box), $n_{\mathrm{eff}}=2299$, a fit budget $\rho^2 n/H = 50.3$ nats and a search budget $\rho^2 n_{\mathrm{eff}} = 23.0$. The embargo is declared at $g = 523$ trading days: under Proposition~\ref{prop:embargo} (ii)'s general form the gap is a declared predictability of the regime risk, $\kappa_{\mathrm{emb}}\rho/\rho_{\mathrm{blk}} = 0.10$ over the block's lags, and Experiment 1's panel J bounds that coefficient from below on this coordinate --- not contradicted at $523$, contradicted at the $112$ the AR(1) chain would compute, the first gap at which the bound falls to the tolerance being $390$--$550$ days. Part (i) asks a purge of $H$ for the noise. The block scores' coherence across blocks is Lemma~\ref{lem:effsample}'s identity, $\upsilon/\upsilon_b = 1.52$ at the declared $\ell$ and $\rho_{\mathrm{blk}}$ (Corollary~\ref{cor:blocks} (a)), already inside $n_{\mathrm{eff}}$; panel B reads it at $1.10$ and the tail statistic's at $1.05$, inside their null and below that $1.52$. Panel J's coherence time of the same coordinate's rank is $276 = 2\cdot138 - 1$ against the declared $\ell=60$, and the ledger at $\ell=138$ is a row of the sensitivity table below.

\textbf{S2.} Ridge on $P=200$ features (six base features, the rest random Fourier features of them), the penalty solved so that $d_{\mathrm{eff}}=25.2$ at $\varkappa_{\mathrm{fit}}=1/2$; deployed as the average of $m=20$ refits on random halves of the $83$ training blocks (Proposition~\ref{prop:ensemble}), its block scores out-of-fold. It enters the library as candidate $72$ and ranks $10$th; the fit spends $25.2$ of $50.3$ nats, and by label perturbation the twenty-member average has $d_{\mathrm{eff}} = 15.2\pm0.4$ against one member's $15.4\pm0.2$, Proposition~\ref{prop:ensemble}'s equality to one per cent.

\textbf{S3.} $b=60$, $B=83$, the first day of each block purged, the tail mean over the worst twenty blocks; the pick is a momentum-and-reversal conjunction (training Sharpe $0.63$), and the mean picks the same rule.

\textbf{S4.} $N=72$ charged against the $23.0$ nats, on the tail rows as on the mean row (Lemma~\ref{lem:effsample} at the full length is the whole account of the records' dependence, Corollary~\ref{cor:blocks} (a)):

\begin{center}
\small
\begin{tabular}{lrrrr}
\toprule
price & $c$ & $2c^2\log N$ & left & affordable $N$ \\
\midrule
the mean & $1.00$ & $8.6$ & $+14.4$ & $98{,}000$ \\
Gaussian $c(4)$ & $1.43$ & $17.4$ & $+5.6$ & $284$ \\
plug-in $\hat c(4)$ & $1.87$ & $30.0$ & $-7.0$ & $27$ \\
\bottomrule
\end{tabular}
\end{center}

The comparison is affordable on the mean and at the Gaussian price, and not at the price real block scores pay: the library affordable at the plug-in's $\hat c(4)$ is twenty-seven candidates. One audit of Proposition~\ref{prop:audit} ($2\log 1 = 0$ nats of the held-out window's $8.7$): the search re-run on all $124$ blocks picks a different rule whose held-out Sharpe is $0.76$, the truncated search's pick realises $-0.04$ on the same window (library median $0.57$, buy-and-hold $0.65$, the ridge ensemble $0.60$); $A_{\mathrm{full}} - A_{\mathrm{trunc}} = +0.80$ annual Sharpe, reported and not fed back. \textbf{S5.} The position is $f^\star(\delta\hat\mu - \lambda(\Lambda)\,\mathrm{sd}_{\mathrm{blk}} - c')/\sigma^2$, $\hat\mu$ the training mean of the unit-sd PnL deflated by the max-of-$N$ haircut $\sqrt{2\log N/n_{\mathrm{eff}}}$, $c'$ the turnover at two basis points, $f^\star = 1/(1+\varkappa)$; the dispersion penalty under its three readings (Proposition~\ref{prop:kelly}): (1) the total block-score sd, the committed reading; (2) net of the within-block Newey--West noise floor; (3) the exact block $\mathrm{CVaR}_{1/4}$ in place of $\hat\mu - \lambda\,\mathrm{sd}_{\mathrm{blk}}$. Robust edge in daily-sd units; a position exists only where it is positive.

\begin{center}
\small
\begin{tabular}{lrrrrrr}
\toprule
candidate & $\hat\mu$ & haircut & $\mathrm{sd}_{\mathrm{blk}}$ total / net & (1) total & (2) net & (3) $\mathrm{CVaR}$ \\
\midrule
the pick, $N=72$ & $0.0395$ & $0.0610$ & $0.0962$ / $0$ & $-0.147$ & $-0.025$ & $-0.138$ \\
the pick declared alone & $0.0395$ & $0$ & $0.0962$ / $0$ & $-0.086$ & $+0.036$ & $-0.077$ \\
buy-and-hold, alone & $0.0166$ & $0$ & $0.1168$ / $0.0600$ & $-0.132$ & $-0.060$ & $-0.140$ \\
ridge ensemble, alone & $0.0218$ & $0$ & $0.1009$ / $0$ & $-0.111$ & $+0.017$ & $-0.112$ \\
\bottomrule
\end{tabular}
\end{center}

For the pick ($c'=0.0035$ at a turnover of $0.117$ per day) the haircut alone is the reason under every reading; declared alone it trades only under reading (2), and under (1) and (3) its edge survives to $\Lambda=1.3$ and $1.2$. Buy-and-hold has no position under any reading and an edge only to $\Lambda\approx1.1$; the ridge ensemble is the same case as the pick. Figure~\ref{fig:lambdacurve} is the curve Remark~\ref{rem:lamlower} asks for: every edge crosses zero between $1.1$ and $1.3$ --- below $\hat\Lambda_{\mathrm{split}} = 1.36$ (inside its null's $1.50$), well below Proposition~\ref{prop:occfloor}'s floor $\Lambda_{0.05} = 2.24$ for the held-out window at $p=0.1$, and \emph{at} the floor for the level the curve is drawn at, $p=1/\Lambda$, whose fixed point is $1.17$. Under the committed reading no candidate in the library has a position at $\Lambda=4$.

\begin{figure}[htbp]
\centering
\includegraphics[width=0.68\linewidth]{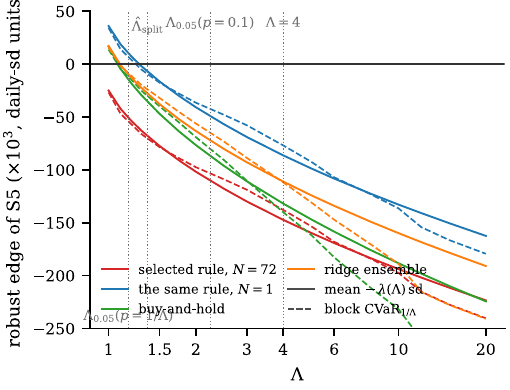}
\caption{Remark~\ref{rem:lamlower}'s curve on Experiment 4's four candidates: the robust edge of S5 in daily-sd units against $\Lambda$, solid under reading (1), dashed under reading (3). Marked: $\hat\Lambda_{\mathrm{split}} = 1.36$, Proposition~\ref{prop:occfloor}'s floor for the held-out window at $p=0.1$ ($2.24$) and at $p=1/\Lambda$ (the fixed point $1.17$), and the declared $\Lambda=4$. Every candidate's edge is gone at the self-consistent floor.}
\label{fig:lambdacurve}
\end{figure}

What the walk shows: the fit is cheap and the search is not; a seventy-two-way comparison on the quarter-tail is over budget at the price real block scores pay, and the haircut the budget imposes is larger than the edge of the rule it selects.

At the dependence the library measures (implied $\hat\rho_{\mathrm{blk}}^2$ median $0.0009$, an invariance ratio $\kappa_h \approx 0.3$ against the declared $1$, $\upsilon=1.11$) the search budget is $45.5$ nats and both tail charges fit, the plug-in's with $15.5$ to spare; at the Bartlett ratio the budget is $63$. At the refined $\ell = 138$ the budget is $13.4$ and neither tail price fits, the Gaussian's over by $4.0$. The verdict on the search moves with the declarations, and the position does not move.

Corollary~\ref{cor:nonempty}'s two bounds contain the whole negative result: the haircut alone allows $N<4.4$ proposals at this edge and effective sample, against the $72$ compared, so no block length produces a position --- the same ceiling is $60$ at the measured $\hat\upsilon$, $58.7$ at the lower end of the box the declared $\tau_\mu$ fixes, and $2.4$ at the $\ell=138$ panel J reads off the declared coordinate's rank, so the conservative declaration is worth a factor of fourteen in it and none of the verdict; declared alone the rule would need $b>693$ days and $n>\Lambda b$, about eleven years.

\subsection{Experiment 5: what out-of-sample degradation scales with}
\label{app:exp5}
The first trial of a prediction of \S\ref{sec:empirical}. P2 says the in-sample minus out-of-sample Sharpe gap scales as $\sqrt{2\log N\,\upsilon/n}$ for a search and $\sqrt{d_{\mathrm{eff}} H/n}$ for a fit, with slopes of order one, and that $\ell$ adds no explanatory power beyond $H$ and $\upsilon$ --- the content of Lemma~\ref{lem:effsample}, against the block-count charge $\sqrt{d\,\ell^\ast/n}$ a mixing-rate account would make. Its regressors already vary in the committed data ($N$ and $n$ within an asset, $H$ by a factor of sixty-three, $\upsilon$ across asset classes), so it can be tried with nothing new fitted: Experiment 2's twenty series and rules, Experiment 3's four features, standard errors clustered by asset. Each term is read against the closed form the prediction names --- for the search, Theorem~\ref{thm:search} (i)'s \emph{bound} --- and against the constant the theorems actually predict: the expected maximum of $N$ null Gaussians with the library's own training-window correlation --- the object Theorem~\ref{thm:search} (iii) bounds on both sides by the list's search entropy --- and the null-signal optimism of the same fit on iid Gaussian returns. ``Of order one'' is a statement about the second column; independent null candidates already realise only $0.49$ to $0.80$ of $\sqrt{2\log N}$ at $N=2$ to $47$.

\textbf{Panel A, the search term.} For each asset and each year from the ninth, the best rule by mean on the expanding window is selected from $N\in\{2,5,12,47\}$ and its training Sharpe minus next-year Sharpe recorded in daily units ($6{,}463$ points, plus an $N=1$ control: every rule's own gap, mean $-0.006$). The mean gap rises with $N$: $-0.000$, $0.006$, $0.010$, $0.015$, against the bound's $0.019$, $0.029$, $0.036$, $0.045$ --- a third of the bound at $N=47$ --- and $0.92$, $0.94$, $0.85$, $0.83$ of the expected maximum under the library's correlation net of the control. On the sixteen equity series the slope is $0.54\pm0.08$ on the bound and $0.79\pm0.12$ on the prediction; the competing regressor $\sqrt{2\log N\,\ell_a/n}$, with $\ell_a$ the coherence time of the asset's own realised volatility, adds $\Delta R^2 = 0.001$ ($t=0.8$). With the four bond funds the slope on the prediction falls to $0.23\pm0.24$: their $\hat\upsilon$ of $0.8$--$3.3$ is matrix-pricing smoothness, as present out of sample as in, and inflates no gap. Since the equities all sit near $\hat\upsilon=0.86$, the panel tests the $\sqrt{2\log N/n}$ scaling at one $\upsilon$ and not the $\upsilon$ factor --- P2's minimal-detectable-effect clause, met in this data.

\textbf{Panel B, the fit term.} For each equity asset, $H\in\{1,5,21,63\}$ and $d\in\{1,2,4,8,16,32\}$ (Experiment 3's features, then random cosine features of them), OLS of the $H$-day forward return fitted on the first half of the series, the Sharpe gap of its $H$-day PnL between halves in the label's own units: $384$ points. The gap grows with both $d$ and $H$ ($0.031$, $0.051$, $0.121$, $0.179$ at $d=4$ over the four $H$, against $\sqrt{4H/n}$ of $0.034$, $0.075$, $0.153$, $0.266$) and in daily units is flat in $H$. The regression on $\sqrt{dH/n}$ has slope $0.66\pm0.06$, $R^2=0.42$; on the null-signal constant (the same fit on iid Gaussian returns, forty draws per cell, $0.50$ to $1.58$ times $\sqrt{dH/n}$) the slope is $1.08\pm0.12$; $\ell$ adds $\Delta R^2=0.008$ ($t=1.2$), and the block-count charge $\sqrt{d\max(\ell^\ast,H)/n}$ with $\ell^\ast=60$, flat in $H$ below sixty days, explains $0.20$ against $0.42$. Here $H$ moves the predictor by a factor of eight within an asset and $\ell_a$ not at all, so the two are separable.

What the experiment shows. Both terms scale with what P2 says they scale with and neither with $\ell$, so P2 is not refuted where it can be read and the block-count charge is, on the same table; the $\upsilon$ factor is untested, and Appendix~\ref{app:predictions} says it cannot be tested here. It is one library on twenty correlated series; the cross-sectional test of \S\ref{sec:limits} remains the one that would carry weight.

\subsection{Experiment 6: the cross-sectional library}
\label{app:exp6}
The substrate the previous five did without. Section~\ref{sec:limits} turns on a cross-section it does not have, because by Proposition~\ref{prop:breadth} a library whose candidates share the macro path holds one asset's worth of regime draws however many candidates it has. Two libraries are read over the same window --- monthly, 1965--2024, annual blocks, forty-eight out-of-sample years --- chosen to differ in effective breadth and in as little else as the record allows: the $212$ published predictors of \cite{chen2022open} as monthly long-short portfolio returns, of which $133$ are complete over the window, and $82$ long-only portfolios from \cite{frenchdata} (49 industries, 25 size--book-to-market, 10 momentum deciles) in excess of the bill. A long-short spread holds the market on both sides; a long-only portfolio holds it once.

\textbf{Panel A, breadth.} The average pairwise correlation is $0.052$ across the long-short library and $0.660$ across the long-only one, so $N_a^{\mathrm{eff}}$ is $21.3$ against $1.6$, stable to within $10\%$ over three windows. The ceiling ratio Proposition~\ref{prop:breadth} predicts from it --- $H=1$, $2\ell-1=7$ from Remark~\ref{rem:whichell}'s monthly list, both regime dispersions measured from the block variance ratio, every term in the units of one instrument's variance --- is $16.4$ for the long-short library's equal-weighted average against the $29$ its breadth alone would buy ($11.0$ by the model-free ratio of block-mean variances): averaging cut the regime term tenfold ($0.37$ to $0.037$) where it cut the noise term twenty-onefold, so at the average the two are of one size and the regime term halves what breadth delivers. Inverse-variance weights raise the ceiling ratio to $37$, and the common regime variance --- what no weighting summing to one removes --- is at most $0.003$ in the same units, a twentieth of a single signal's $0.053$. Read as shares, the published cross-section's co-movement in stress shows as a regime share of the average's block-mean variance ($44\%$) above a single signal's ($27\%$), not as a regime term that breadth failed to reduce.

\textbf{Panel B, the plug-in.} At $b=12$ months the long-only library gives median $\hat c(4)=1.88$ against the Gaussian $1.43$, with $c^2/\Lambda$ of $0.94,0.88,0.80,0.53$ at $\Lambda=2,4,8,20$: Experiments 1 and 2's $c^2\approx0.9\Lambda$, reproduced on a different substrate, a different frequency and a block length twelve times longer. The long-short library gives $1.46$ at $\Lambda=4$ --- the Gaussian value --- and $c^2/\Lambda$ of $0.63,0.53,0.42,0.26$. Studentised, both collapse to the Gaussian ($1.39$ and $1.40$ at $\Lambda=4$), as in Experiment 2. So \S\ref{sec:price}'s dear branch belongs to block scores that carry the market, not to daily data, and the two libraries stand in a $\hat c^2$ ratio of $2.1$ at $\Lambda=20$, which is the spread P4 says it needs.

\textbf{Panel C, P1.} Experiment 1 panel D's protocol --- expanding window from twelve years, re-selected yearly; criteria mean, $\mathrm{CVaR}_{1/2}$, $\mathrm{CVaR}_{1/4}$, $\mathrm{CVaR}_{1/8}$, $\mathrm{CVaR}_{1/12}$ and mean $-\lambda(4)$ sd --- with each candidate standardised by its training window's sd, reported on two layers: the selected candidate's out-of-sample Sharpe, and the pooled lower quarter-tail of its yearly block scores. On Sharpe, against mean selection: $-0.09\pm0.15$ at $1/2$, $-0.46$ at $1/4$ ($t=-2.4$), $-0.73$ at $1/8$ ($t=-3.3$) on the long-short library; on the long-only library $-0.01$ at $1/2$ and $-0.19$ to $-0.22$ at the narrower levels, none past $|t|=1.2$ (its industry sub-library, $N=47$, is the one cut the tail favours on Sharpe: $+0.15\pm0.10$ at $1/2$, within $0.08$ elsewhere). On the deployed layer the long-only library's selected block tail moves by less than $0.10$ either way ($+0.13,+0.05,+0.03,+0.07$ at the four levels on the industry cut) and the long-short library's moves by less than $0.08$ either way, its worst quarter of held-out years no worse than $-0.01$ under any criterion. With membership held fixed and the first out-of-sample year moved, the $1/8$ deficit runs $-0.73,-0.55,-0.90,-0.73$ across four starts, and moves to $-0.56$ at $b=6$ months and $-0.19$ at $b=24$: the signature of the block count, read in Appendix~\ref{app:predictions}. Across terciles of $\mathrm{sd}_{\mathrm{blk}}$ the advantage does not increase on any library.

\textbf{Panel G, P3's ladder} (the design was fixed before its results were seen, and it draws from its own RNG stream). Three ladders share one feature family --- eight trailing own-return means, then random cosine features to $d = 64$, OLS --- trained on 1965--1994 and read once on 1995--2024: the OSAP long-short cross-section, the French long-only cross-section, and pure timing of the French pooled average. The OSAP arm's out-of-sample Sharpe rises to a plateau it holds to the end of the grid ($0.69$--$0.71$ from $d = 12$, argmax at $48$); the French cross-section peaks at $d = 6$ ($0.39$) and is declining by $d = 32$; the two peaks sit $8\times$ apart against the $7.8\times$ panel A's terms predict --- Proposition~\ref{prop:breadth}'s ratio of the two denominators, noise term plus regime term, each library in the units of its own instruments; $13.2\times$ the breadth ratio alone would say, and panel A is the finding that the OSAP average's regime term is not droppable --- P3's decade of displacement, the curves being flat within $0.02$ over neighbouring grid points, so the peaks are read to a factor of about two. The timing arm never clears $0.19$ against a standard error of $0.18$, so its ceiling is read only as an upper bound --- the small-budget end of the same prediction.

What the experiment shows. Breadth in the cross-section is real and large, and buys much less than its size, because the published cross-section's own average carries a regime term the proposition needed to be zero. The tail is dear where block scores carry the market and Gaussian where they do not. P1 gets its first test with power, and at the one level whose tail can be read it buys nothing; where the criteria do separate, they separate by an amount that moves with the block count and not with the window, which is Proposition~\ref{prop:eiv}'s bias and not a verdict on Theorem~\ref{thm:cvar}. The complexity displacement P3 predicts is there, at the order predicted, where there is breadth to see it (panel G).

\subsection{The five predictions: power, and what happened}
\label{app:predictions}
Section~\ref{sec:empirical} states each prediction with what would refute it and a one-line verdict, and the sections above carry the evidence panel by panel. What is left, and belongs nowhere else, is the \emph{power}: the effect each test could have resolved, without which a null is not a finding.

\textbf{P1: CVaR selection against mean selection.} The advantage lives in the stress blocks the held-out era contains, so power is set by the number of independent regime draws, not of candidates (Proposition~\ref{prop:breadth}): on a single-asset timing library the standard error of a Sharpe difference is $0.1$ (Experiment 1) against a predicted margin of $\lambda(\Lambda)\,\mathrm{sd}_{\mathrm{blk}}$, so that null is a power statement and nothing more. The cross-section supplies the power: at $\Lambda=2$, the one level whose tail is never thin, the standard error is $0.15$ in annual Sharpe against a predicted $\lambda(2)\,\mathrm{sd}_{\mathrm{blk}}=0.93$, and the two criteria are indistinguishable on both libraries --- the first null on this prediction with power behind it rather than in front of it, bounding the advantage at a small fraction of what P1 asks for (Experiment 6, panel C). At the narrower levels the criterion \emph{loses}, and the deficit moves with the block length in either direction and not with the window: the signature of how many blocks the tail averages --- at $b=12$ months the early windows read one block, a minimum and not a tail mean --- which is a narrow level colliding with Proposition~\ref{prop:occfloor}'s floor (\S\ref{sec:eiv}). (Moving the \emph{data} start instead, $93$ candidates from $1953$ or $183$ from $1985$, flips the sign of the narrow-level comparison on its own, a fact about the ragged availability of published signals.) P1's second clause fares worse: across terciles of the candidates' block dispersion the advantage does not increase on any library.

\textbf{P2: what out-of-sample degradation scales with.} The two slopes are identified only if $H$ and $\upsilon$ vary across the strategies by more than their measurement error; the square-root form is the weak-signal limit, stated for the regime this paper says markets are in, $\rho^2 n_{\mathrm{eff}}/d$ of order one, and read against the \emph{derived} constant --- the expected maximum of $N$ null candidates under the library's own correlation --- not the bound. Both terms clear that bar in Experiment 5 and are not refuted; $\ell$, the alternative, is. The $\upsilon$ factor does not clear it and cannot: Remark~\ref{rem:upsbox} confines $\upsilon$ to a factor of $3.0$ on daily labels at the declared ceiling ($1.8$ at a ceiling of one annual Sharpe unit), against the ten- to sixtyfold spread a slope would need, so no admissible series identifies one and the falsifiable content is the box --- inside which the equities sit at $0.86$ and the bond funds at $0.8$--$3.3$, the high-yield fund alone above its ceiling, on stale-NAV smoothing that inflates no gap.

\textbf{P3: the cross-sectional complexity ceiling.} The prediction is a ratio, so $N_a^{\mathrm{eff}}$ must itself be measured --- from the spectrum of the cross-sectional residual correlation --- and the peaks resolved on a log axis to better than that ratio; a cross-section whose residuals load on a few factors has $N_a^{\mathrm{eff}}$ far below its nominal breadth and a correspondingly smaller, possibly invisible, displacement. Tried (Experiment 6, panels A and G): the input is measured rather than assumed and is smaller than the nominal breadth twice over, and the fitted ladders return panel G's verdict --- $8\times$ against the predicted $7.8\times$, the timing half read only as a bound. Not refuted where it can be read.

\textbf{P4: the tail's price against the degradation ratio.} At $B\approx100$ blocks the plug-in carries the score-noise bias of Proposition~\ref{prop:eiv} (iii) and a sampling error of its own, the held-out era must hold enough blocks for a tail mean, and $\hat c^2$ must differ by a factor of two or more across the libraries compared. A library of published signals cannot supply the test: Definition~\ref{def:reflexivity}'s decay operates on a set selected for publication and has already spent the difference the test reads. Read where no publication filter operates --- our own library, each criterion's degradation net of a no-selection control --- the pooled ratio lands on the plug-in at both levels, at the edge of the joint null's power (Experiment 2, panel E): the level tracks, and the cross-library ordering needs libraries with more than one regime path.

\textbf{P5: budget saturation and the revealed Kelly fraction.} $\varkappa$ is read per pipeline from the worksheet of \S\ref{sec:calibration} --- $d_{\mathrm{eff}}$, $n$, $H$ and the forecast's shrinkage factor --- and without the last the prediction is not testable at all; with it, a revealed fraction is informative only to the extent it is pinned down more tightly than the wide band an interior budget utilisation already implies (Remark~\ref{rem:kappa}). Untried, and the prediction \S\ref{sec:calibration} declines to treat as already confirmed.

\clearpage
\section{Notation, and the symbols that are used twice}
\label{app:notation}

Every symbol is defined where it first appears; this table collects them, and its last part flags the letters the paper reuses.

{\footnotesize
\begin{longtable}{>{\raggedright\arraybackslash}p{0.11\linewidth}>{\raggedright\arraybackslash}p{0.53\linewidth}>{\raggedright\arraybackslash}p{0.24\linewidth}}
\toprule
\textbf{symbol} & \textbf{meaning} & \textbf{where} \\
\midrule
\endhead
\multicolumn{3}{l}{\emph{The constants}} \\
$\Lambda \ge 1$ & recurrence bound, $\|\mathrm{d}\bar\pi^{+}/\mathrm{d}\bar\pi_n\|_\infty \le \Lambda$; a premise, not estimable from the past & \ref{ax:recurrence}, Cor.~\ref{cor:notestimable} \\
$\ell \ge 1$, $\ell^*$ & coherence time of the slowest declared coordinate (A4); $\ell^* := \max(\ell, H, L_x)$ the contiguity scale, the shortest block that is a unit & \ref{ax:persistence} \\
$\upsilon_b$, $\upsilon/\upsilon_b$ & $\upsilon$ of \eqref{eq:neff} computed at $n = b$, and the block scores' \emph{record coherence}: their integrated autocorrelation time across blocks, so that $B$ blocks are worth $B\upsilon_b/\upsilon$ independent ones --- the block count of Thm.~\ref{thm:deployed}'s selection error, an identity of Lem.~\ref{lem:effsample} and not a declaration & Cor.~\ref{cor:blocks} (a), Rem.~\ref{rem:blocklen} \\
$L_x$ & memory of the declared feature map: how far back a feature reads & \ref{ax:persistence} \\
$\rho \in (0, 1)$ & signal ceiling, $\mathbb{E}[\mu(X)^2] \le \rho^2 \mathrm{Var}(Y)$; equals the maximal per-period Sharpe ratio & \ref{ax:snr}, Lem.~\ref{lem:sharpe} \\
$\rho_{\mathrm{blk}}$ & regime dispersion of a candidate's edge in Sharpe units, $\mathrm{sd}_\nu(r_h)/\sigma$; asserted with the representation & Lem.~\ref{lem:effsample}, Rem.~\ref{rem:whichell} \\
$\chi_\sigma = \bar\sigma/\underline{\sigma}$ & volatility ratio across regimes; estimable from within-block sd & \ref{ax:snr} \\
$\pi_\mu$, $\tau_\mu$ & predictability term of the effective sample and the coherence, in periods, of the component that produces it; $|\pi_\mu| \le \rho \tau_\mu$; every worksheet here charges $\pi_\mu$ at $0$ and declares $\tau_\mu = 1$ & Lem.~\ref{lem:effsample}, Rem.~\ref{rem:upsbox} \\
$\Delta$, $\delta$, $c(C, h)$ & latent leverage (Thm.~\ref{thm:impossible}); edge decay and cost & \ref{def:reflexivity} \\
$b$, $\Lambda(b)$ & evaluation scale, and the recurrence bound declared at it & \ref{ax:recurrence}, Prop.~\ref{prop:scale} \\
$R_j(h)$ & regime risk averaged over historical block $j$ & Thm.~\ref{thm:cvar} \\
$\kappa$ & invariance ratio: declared $\rho_{\mathrm{blk}} = \kappa \rho$ & S1 of Thm.~\ref{thm:canonical}, Prop.~\ref{prop:kelly}, Cor.~\ref{cor:invbudget} \\
$\vartheta$, $\Lambda_1$ & scale-free protection, and the state-scale reading of an assertion & Prop.~\ref{prop:scale} \\
\multicolumn{3}{l}{\emph{Objects}} \\
$X_t, Y_t, Z_t$ & feature, label (forward return over $H$ periods), latent state & \S\ref{sec:setting} \\
$Q(\cdot \mid z)$ & the invariant mechanism & \ref{ax:invariance} \\
$\bar\pi_n$, $\bar\pi^{+}$, $\bar\pi_{n, j}$ & occupation measures of the history ($n$ periods), the deployment window ($m$ periods) and block $j$; random, functions of the path & \eqref{eq:occupation}, Rem.~\ref{rem:occupation} \\
$\nu$ & stationary law of the latent chain ($\mathrm{Var}_\nu$) & Lem.~\ref{lem:effsample}, Prop.~\ref{prop:occfloor} \\
$h$, $L(a, y)$, $M$ & strategy, loss, loss bound & Def.~\ref{def:strategy} \\
$r_h(z)$, $\bar{R}_n(h)$, $R^{+}(h)$ & regime-conditional risk; its history and deployment averages & \eqref{eq:regimerisk} \\
$\psi$, $\varepsilon(\psi)$, $\Lambda(\psi)$, $R_\psi^{+}(h)$ & coarsening of the latent space; its invariance defect, conditioned on the coarsened path; the recurrence constant of the coarsened path, $\le \Lambda$; the deployed risk given the coarsened path, $\mathbb{E}[R^{+}(h) \mid \psi(Z)]$ & Prop.~\ref{prop:representation} \\
$\varepsilon_0$ & the defect declared of the representation itself; charged at $2 M \varepsilon_0$ & \ref{ax:invariance}, Prop.~\ref{prop:representation} \\
$\Psi$ & menu of representations declared at S1; selection within it is charged $\log|\Psi|$ beside $\log N$ & \S\ref{search} \\
\multicolumn{3}{l}{\emph{Evaluation}} \\
$b$, $B = \lfloor n/b \rfloor$ & block length and count & Cor.~\ref{cor:blocks} \\
$\hat{r}_j(h)$, $r_h^{(b)}$, $\xi_j$ & block score, its law (block-averaged regime risk plus noise), the noise & \eqref{eq:blockscore}, Thm.~\ref{thm:deployed} \\
$v(b, \ell)$ & regime variance a contiguous block keeps, $(2\ell - 1) - 2\rho(1-\rho^b)\ell^2/b$ with $\rho = 1 - 1/\ell$ the chain's self-transition probability & \eqref{eq:vcontig}, Prop.~\ref{prop:shuffle} \\
$\upsilon$, $n_{\mathrm{eff}} = n/\upsilon$ & variance inflation $H + \rho_{\mathrm{blk}}^2(2\ell - 1) + 2\pi_\mu$ of a backtest mean; the effective sample of the search term & \eqref{eq:neff} \\
$n_{\mathrm{fit}} = n/H$ & effective sample of the capacity term & Thm.~\ref{thm:capacity} \\
$g$, $\kappa_{\mathrm{emb}}$ & purge-plus-embargo gap: $H$ periods purged for the noise (feature memory is priced by A5, not purged), and the declared predictability's lag for the loss; the share of signal the embargo bias may consume (the subscript keeps the tolerance apart from the invariance ratio $\kappa$) & Prop.~\ref{prop:embargo} \\
$\mathrm{pred}_h(k)$ & predictability of the regime risk at lag $k$: the largest correlation the observable history has with $r_h(Z_k)$; at most $\rho_\ell^k$ on the redraw chain, bounded below by an observed coordinate's autocorrelation & Prop.~\ref{prop:embargo} (ii) \\
$\alpha = 1/\Lambda$, $u$ & tail level; the $(1 - \alpha)$-quantile of the block law (losses; $q$ in the reward orientation) & Prop.~\ref{prop:tailprice}, Rem.~\ref{rem:orientation} \\
$\lambda(\Lambda)$ & invariance penalty $\Lambda \varphi(\Phi^{-1}(1 - 1/\Lambda))$ & Thm.~\ref{thm:invariance} \\
$\varsigma_\Lambda^2$, $c(\Lambda) = \varsigma_\Lambda/s$ & asymptotic variance of the tail plug-in; the tail's price, $c(1) = 1$ and $c(\Lambda) \le \Lambda$ ($c \ge 1$ is a property of the block-risk law's shape at the level, not a general fact) & Prop.~\ref{prop:tailprice} \\
$\tau_\psi$ & the coherence across blocks of the tail statistic $\max(\hat r_j, u)$, the one statistic Thm.~\ref{thm:deployed}'s fluctuation actually pays: at most the scores' $\upsilon/\upsilon_b$ in the Gaussian range, and read on the scores otherwise (Exp.~1, panel B) & Rem.~\ref{rem:blocklen}, Cor.~\ref{cor:blocks} (a) \\
$\Delta_N$, $\Delta_{\mathrm{CVaR}}$ & uniform selection error, the larger of the plug-in's distance from the population tail and its shortfall below its path-conditional mean; tail-mean separation of the genuine candidate & Thm.~\ref{thm:deployed}, Cor.~\ref{cor:locfam} \\
$\sigma_\xi^2(r)$ & score-noise variance given the regime risk & Prop.~\ref{prop:eiv} \\
$\mathrm{sd}_{\mathrm{within}}$, $E_N$ & within-block noise sd of the score, the floor of $\mathrm{sd}_{\mathrm{blk}}$ at $\mathrm{sd}_{\mathrm{within}}/\sqrt{b}$; the edge net of cost and haircut & Cor.~\ref{cor:nonempty} \\
$\epsilon_n(k)$ & replace-$k$ sensitivity of an estimator & \eqref{eq:replacek} \\
\multicolumn{3}{l}{\emph{Budget and sizing}} \\
$d$, $d_{\mathrm{eff}}$, $N$ & nominal and effective dimension; candidates evaluated over the system's life & Thms.~\ref{thm:capacity}, \ref{thm:search} \\
$T$, $I(T; \phi)$ & the selection rule, and the mutual information (nats) between the choice and the statistics it was chosen from; $\le \log N$ for a fixed list & Thm.~\ref{thm:search} (ii) \\
$d(j,k)$, $N(\delta)$, $N_{\mathrm{eff}}$ & the correlation metric on a candidate list, $\sqrt{2(1 - r_{jk})}$ with $r_{jk}$ the correlation of the null statistics; its packing number at scale $\delta$; and the list's \emph{search entropy} $\log N_{\mathrm{eff}} = \sup_{\delta \le \sqrt{2}} \frac{1}{2}\delta^2\log N(\delta) \le \log N$, equal to $\log N$ for an independent list --- the search charge in both directions & Thm.~\ref{thm:search} (iii) \\
$\pi$, $\pi_0$, $\mathrm{KL}(\pi \| \pi_0)$ & the researcher's final choice as a law on the class, the prior fixed before the data, and their divergence --- the search term as a theorem & Thm.~\ref{thm:pacbayes} \\
$\omega(\lambda)$, $C_\ell(a)$ & the regime cumulant per period on the redraw chain (root of the renewal equation) and its Bennett coefficient, $C_\ell(0+) = \ell - 1/2$ & Thm.~\ref{thm:pacbayes} (i) \\
$\varkappa = d/(\rho^2 n_{\mathrm{fit}})$ & capacity-budget utilisation; $1$ is saturation, not the optimum & Cor.~\ref{cor:halfkelly}, Rem.~\ref{rem:kappa} \\
$\tau^2$, $s^2$, $f^\star$ & estimation variance and prior variance of the edge, in Sharpe units; the Kelly fraction $s^2/(s^2 + \tau^2) \le 1/(1 + \varkappa)$ & Prop.~\ref{prop:kelly} \\
$N_a$, $N_a^{\mathrm{eff}}$ & instruments, and effectively independent instruments & Prop.~\ref{prop:breadth} \\
$\hat{\Lambda}_{\mathrm{split}}$ & split-sample density ratio: a constraint on $\Lambda$, not an estimate & Rem.~\ref{rem:lamlower} \\
\bottomrule
\end{longtable}
}

\textbf{Symbols used twice.} These letters carry more than one meaning, always in different statements; the context decides, and the reader editing one formula from memory of another should check.

{\footnotesize
\begin{tabularx}{\linewidth}{>{\raggedright\arraybackslash}p{0.08\linewidth}>{\raggedright\arraybackslash}X>{\raggedright\arraybackslash}X}
\toprule
\textbf{symbol} & \textbf{one meaning} & \textbf{the other} \\
\midrule
$\Psi$ & the menu of representations (\S\ref{search}) & a local symbol of Appendix A: $\Psi_n(\lambda)$ in Thm.~\ref{thm:pacbayes}'s proof \\
\midrule
$m$ & deployment window length (\S\ref{sec:setting}, Prop.~\ref{prop:occfloor}) & ensemble size (Prop.~\ref{prop:ensemble}); and the range bound on $r_h$, $m = 2\rho\bar\sigma$ for a bounded rule's PnL, inside Theorem~\ref{thm:pacbayes} and its condition $\lambda m \ell \le 1/4$ \\
\midrule
$K$ / $K_z$ & occupied cells of the declared partition (\ref{ax:recurrence}'s consequences) & coordinates of the declared state, $Z_t = (Z_t^1,\dots,Z_t^{K_z})$ (\ref{ax:persistence}, Lem.~\ref{lem:effsample} (iii)) --- the subscript distinguishes \\
$n_{\mathrm{eff}}$ / $N_{\mathrm{eff}}$ / $N_a^{\mathrm{eff}}$ & the effective sample, in periods (\eqref{eq:neff}) & the search entropy of a candidate list (Thm.~\ref{thm:search} (iii)); the effective breadth of a cross-section, in instruments (Prop.~\ref{prop:breadth}) --- three ``effective'' counts, of periods, candidates and instruments \\
$c$ & reflexivity cost $c(C, h)$, and $c'$ in stage S5 (\ref{def:reflexivity}) & the tail price $c(\Lambda)$ (Prop.~\ref{prop:tailprice}); both appear in S5 of Theorem~\ref{thm:canonical} \\
$s$ & sd of the block law (Prop.~\ref{prop:tailprice}) & prior sd of the edge (Prop.~\ref{prop:kelly}) \\
$\Delta$ & latent leverage (Thm.~\ref{thm:impossible}) & $\Delta_N$, $\Delta_{\mathrm{CVaR}}$ (Thm.~\ref{thm:deployed}, Cor.~\ref{cor:locfam}) \\
$\rho$ & signal ceiling (\ref{ax:snr}) & self-transition probability $1 - 1/\ell$ inside $v(b, \ell)$ (Props.~\ref{prop:shuffle}, \ref{prop:occfloor}); $\varrho$, a different letter, is the common-noise share \\
$\nu$ & stationary law of the chain & ridge eigenvalues $\nu_i$ (\S\ref{sec:budget}) \\
$\kappa$ / $\kappa_{\mathrm{emb}}$ / $\varkappa$ & invariance ratio, $\rho_{\mathrm{blk}} = \kappa\rho$ (S1 of Thm.~\ref{thm:canonical}, Prop.~\ref{prop:kelly}) & embargo bias tolerance (Prop.~\ref{prop:embargo}); budget utilisation (Cor.~\ref{cor:halfkelly}) --- three symbols, one word when spoken; Experiment 4's worksheet uses all three at once \\
\midrule
$\lambda$ & the invariance penalty $\lambda(\Lambda)$ (Thm.~\ref{thm:invariance}) & the ridge penalty inside $d_{\mathrm{eff}}(\lambda)$ (Thm.~\ref{thm:capacity}); the temperature of the PAC-Bayes bound (Thm.~\ref{thm:pacbayes}); the eigenvalues $\lambda_i$ in the proof of Lem.~\ref{lem:effsample} \\
$\varepsilon$ / $\epsilon$ & the invariance defect, $\varepsilon_0$ and $\varepsilon(\psi)$ (\ref{ax:invariance}, Prop.~\ref{prop:representation}) & the level of Prop.~\ref{prop:occfloor}'s floor $\Lambda_\epsilon$; the replace-$k$ sensitivity $\epsilon_n(k)$ (Thm.~\ref{thm:impossible}); the null statistics $\epsilon_j$ (Thm.~\ref{thm:search}) --- two glyphs, the defect always the first \\
$\delta$ & edge decay (\ref{def:reflexivity}, stage S5) & the confidence level of Thm.~\ref{thm:pacbayes} and of Thm.~\ref{thm:deployed}'s finite-sample form; the gap Prop.~\ref{prop:tailfloor} resolves; the cube's half-width in the proof of Prop.~\ref{prop:minimax} \\
$\phi$ / $\varphi$ & the feature map $\phi(\mathcal{F}_t^{\mathrm{obs}})$ (S1) and the in-sample statistics $\phi_i$ (Thm.~\ref{thm:search} (ii)) & $\varphi$ the standard normal density, inside $\lambda(\Lambda)$ and $c(\Lambda)$; $\varphi_G$ a generating function in the proof of Thm.~\ref{thm:pacbayes} \\
\bottomrule
\end{tabularx}
}

Orientation: the prose works in losses and prices the upper tail; the experiments work in rewards and the lower tail. Remark~\ref{rem:orientation} shows the two give the same $c(\Lambda)$.

\end{document}